\documentclass[12pt]{article}
\usepackage{fullpage}

\usepackage[T1]{fontenc}
\usepackage[utf8]{inputenc}
\usepackage{lmodern}
\usepackage{microtype}
\usepackage{amsmath,amssymb}
\usepackage{graphicx}
\usepackage{booktabs}
\usepackage{hyperref}
\usepackage{url}
\usepackage{color}
\usepackage{xcolor}
\usepackage{ulem}
\usepackage{enumitem}
\usepackage{marvosym}
\usepackage{amsfonts}

\usepackage{xcolor}
\usepackage[full]{textcomp}
\usepackage[osf]{newtxtext}

\usepackage{microtype}

\usepackage{amssymb}          
\usepackage{mathtools}

\usepackage{tikz}
\usepackage{amsthm}           
\usepackage{mhequ}
\usepackage{hyperref}

\newcommand{\eqdef}{\stackrel{\mbox{\tiny def}}{=}}

\numberwithin{equation}{section}

\DeclareMathOperator{\id}{id}
\def\Wick#1{\mathopen{:}#1\mathclose{:}}

\newcommand{\EE}{\mathbb{E}}     
\newcommand{\RR}{\mathbb{R}}   
\newcommand{\TT}{\mathbb{T}}

\newcommand{\cC}{\mathcal{C}}
\newcommand{\dD}{\mathcal{D}}

\let\f\frac

\def\d{\partial}

\newcommand{\eps}{\varepsilon}

\colorlet{symbols}{blue!90!black}
\colorlet{testcolor}{green!60!black}

\tikzset{
	smalldot/.style={circle,fill=symbols,draw=symbols, solid,inner sep=0pt,minimum size=0.5mm},
}

\def\scal#1{\langle#1\rangle}
\def\cent#1{\mathopen{{\langle\kern-0.3em\rangle}}#1\mathclose{{\langle\kern-0.3em\rangle}}}
\def\bscal#1{\big\langle#1\big\rangle}
\def\Bscal#1{\Big\langle#1\Big\rangle}

\makeatletter
\def\DeclareSymbol#1#2#3{\expandafter\gdef\csname MH@symb@#1\endcsname{\tikz[baseline=#2,scale=0.15,draw=symbols]{#3}}\expandafter\gdef\csname MH@symb@#1s\endcsname{\scalebox{0.65}{\tikz[baseline=#2,scale=0.15,draw=symbols]{#3}}}}
\def\<#1>{\csname MH@symb@#1\endcsname}
\makeatother

\AtBeginDocument{%
\crefalias{filejlemma}{lemma}%
\crefalias{filejproposition}{proposition}%
\crefalias{filecthm}{theorem}%
\crefalias{fileccor}{corollary}%
\crefalias{fileclem}{lemma}%
\crefalias{filecprop}{proposition}%
\crefalias{filecconj}{conjecture}%
\crefalias{filecdefinition}{definition}%
\crefalias{filecremark}{remark}%
\crefalias{filecexample}{example}%
\crefalias{filecnotation}{remark}%
\crefalias{fileqfourtheorem}{theorem}%
\crefalias{fileqfourdefinition}{definition}%
\crefalias{fileqfourconjecture}{conjecture}%
\crefalias{fileqfourlemma}{lemma}%
\crefalias{fileqfourcorollary}{corollary}%
\crefalias{fileqfourproposition}{proposition}%
\crefalias{fileqfourobservation}{remark}%
\crefalias{fileqfourremark}{remark}%
\crefalias{fileqfourproblem}{problem}%
\crefalias{fileqfourassumption}{assumption}%
\crefalias{fileqfourclaim}{claim}%
\crefalias{fileqfourfact}{fact}%
\crefalias{fileqfourexample}{example}%
\crefalias{fileftheorem}{theorem}%
\crefalias{fileflemma}{lemma}%
\crefalias{filefcorollary}{corollary}%
\crefalias{filefcor}{corollary}%
\crefalias{filefproposition}{proposition}%
\crefalias{filefconjecture}{conjecture}%
\crefalias{filefquestion}{question}%
\crefalias{filefremark}{remark}%
\crefalias{filefexample}{example}%
\crefalias{filefnotation}{remark}%
\crefalias{filefobservation}{remark}%
\crefalias{filefrant}{remark}%
\crefalias{filefwarning}{warning}%
\crefalias{filefassumption}{assumption}%
\crefalias{filefaxiom}{axiom}%
\crefalias{filefdefinition}{definition}%
\crefalias{filehtheorem}{theorem}%
\crefalias{filehcorollary}{corollary}%
\crefalias{filehlemma}{lemma}%
\crefalias{filehobservation}{remark}%
\crefalias{filehproposition}{proposition}%
\crefalias{filehdefinition}{definition}%
\crefalias{filehclaim}{claim}%
\crefalias{filehfact}{fact}%
\crefalias{filehassumption}{assumption}%
\crefalias{filehwarning}{warning}%
\crefalias{filehconjecture}{conjecture}%
\crefalias{fileadefinition}{definition}%
\crefalias{filealemma}{lemma}%
\crefalias{fileaproposition}{proposition}%
\crefalias{fileacorollary}{corollary}%
\crefalias{filearemark}{remark}%
\crefalias{fileelemma}{lemma}%
\crefalias{fileglemma}{lemma}%
}

\DeclareSymbol{0}{0}{\draw[white] (-.4,0) -- (.4,0); \draw (0,0)  -- (0,1.8) node[smalldot] {};}
\DeclareSymbol{1}{-1}{\draw[white] (-.4,0) -- (.4,0); \draw (0,-0.5)  -- (0,1.3) node[smalldot] {};}
\DeclareSymbol{2}{-1}{\draw (-0.7,1.3) node[smalldot] {} -- (0,-0.5) -- (0.7,1.3) node[smalldot] {};}
\DeclareSymbol{3}{1}{\draw (-1,1.5) node[smalldot] {} -- (0,0) -- (1,1.5) node[smalldot] {}; \draw (0,2) node[smalldot] {} -- (0,0);}
\DeclareSymbol{20}{-.2}{\draw (-0.7,1.5) node[smalldot] {} -- (0,0.5); \draw (0.7,1.5) node[smalldot] {} -- (0,0.5) -- (0,-0.5);}
\DeclareSymbol{30}{0}{\draw (-1,1.5) node[smalldot] {} -- (0,0.5) -- (1,1.5) node[smalldot] {}; \draw (0,2) node[smalldot] {} -- (0,0.5) -- (0,-.5);}
\DeclareSymbol{31}{0}{\draw (-1,1.5) node[smalldot] {} -- (0,0.5) -- (1,1.5) node[smalldot] {}; \draw (0,2) node[smalldot] {} -- (0,0.5) -- (0,-.5) -- (1,0.5) node[smalldot] {} ;}

\DeclareMathOperator{\vecop}{vec}
\DeclareMathOperator{\diag}{diag}
\DeclareMathAlphabet{\catsymbfont}{U}{rsfs}{m}{n}

\newcommand{\bR}{\mathbb{R}}
\newcommand{\co}{\colon}
\newcommand{\scrS}{\mathscr{S}}
\newcommand{\aO}{{\catsymbfont{O}}}

\usepackage{float}
\usepackage{alltt}
\usepackage{comment}
\usepackage{mathrsfs}
\usepackage{etoolbox}
\IfFileExists{bbm.sty}{\usepackage{bbm}}{\providecommand{\mathbbm}[1]{\mathbb{#1}}}
\IfFileExists{stmaryrd.sty}{\usepackage{stmaryrd}}{}
\providecommand{\widebar}[1]{\overline{#1}}
\IfFileExists{upgreek.sty}{\usepackage{upgreek}}{}
\IfFileExists{cases.sty}{\usepackage{cases}}{}
\IfFileExists{subcaption.sty}{\usepackage{subcaption}}{}
\IfFileExists{easytable.sty}{\usepackage[thinlines]{easytable}}{}
\IfFileExists{adjustbox.sty}{\usepackage[export]{adjustbox}}{}
\IfFileExists{appendix.sty}{\usepackage{appendix}}{}
\IfFileExists{thmtools.sty}{\usepackage{thmtools}}{}
\IfFileExists{stackengine.sty}{\usepackage{stackengine}}{}
\IfFileExists{scalerel.sty}{\usepackage{scalerel}}{}
\IfFileExists{xy.sty}{\usepackage[all]{xy}}{}
\IfFileExists{tikz-cd.sty}{\usepackage{tikz-cd}}{}
\IfFileExists{xypic.sty}{\usepackage{xypic}}{}
\IfFileExists{arcs.sty}{\usepackage{arcs}}{}
\IfFileExists{pgfplotstable.sty}{\usepackage{pgfplotstable}}{}
\IfFileExists{colortbl.sty}{\usepackage{colortbl}}{}
\makeatletter
\@ifundefined{pgfplotsset}{}{\pgfplotsset{compat=1.16}}
\makeatother
\IfFileExists{cleveref.sty}{\usepackage[capitalize,noabbrev,nameinlink]{cleveref}}{}
\usetikzlibrary{positioning,arrows,decorations.pathreplacing,calc}
\providecommand{\web}[1]{\href{#1}{\nolinkurl{#1}}}
\providecommand{\LaTeXunderbrace}{\underbrace}
\providecommand{\cref}[1]{\ref{#1}}
\providecommand{\Cref}[1]{\ref{#1}}
\providecommand{\crefalias}[2]{}
\makeatletter
\newcommand{\PatchBibSection}{\let\SavedMergedSection\section\def\section{\@ifstar\subsubsection\subsubsection}}
\newcommand{\EndPatchBibSection}{\let\section\SavedMergedSection}
\makeatother

\title{First Proof}
\author{ Mohammed Abouzaid\footnote{\Letter\ Corresponding author, Email: abouzaid@stanford.edu}
\\
  \textit{Stanford University} 
  \and
Andrew J. Blumberg\\
\textit{Columbia University} 
\and
    Martin Hairer \\
\textit{EPFL and Imperial}
\and
Joe Kileel\\
\textit{University of Texas at Austin}
\and
Tamara G. Kolda \\
\textit{MathSci.ai}
\and
Paul D. Nelson\\
\textit{Aarhus University}
\and
Daniel Spielman \\
\textit{Yale University}
\and 
Nikhil Srivastava\footnote{\Letter\ Corresponding author, Email: nikhil@math.berkeley.edu}
\\
\textit{University of California, Berkeley}
\and 
Rachel Ward\footnote{\Letter\ Corresponding author, Email: rward@math.utexas.edu} \\
\textit{University of Texas at Austin}
\and
Shmuel Weinberger\\
\textit{University of Chicago}
\and
Lauren Williams\footnote{\Letter\ Corresponding author, Email: williams@math.harvard.edu} \\
\textit{Harvard University}
}
\date{\today}
\begin{document}
\maketitle

\emph{In this Arxiv preprint v2, we include an additional appendix with author-written solutions and comments for each of the ten questions, as posted to \web{https://1stproof.org} on February 13, 2026. Information about the next round of First Proof can be found on \web{https://1stproof.org}}

\begin{abstract}
To assess the ability of current AI systems to correctly answer research-level mathematics questions, we share a set of  ten math questions which have arisen naturally in the research process of the authors.  The questions had not been shared publicly  until now; the answers are known to the authors of the questions but will remain encrypted for a short time. 
\end{abstract}

\newpage 

\section{Introduction}

In baking, the {\it first proof}, or bulk fermentation process, is a crucial step in which one lets the entire batch of dough ferment as one mass, before dividing and shaping it into loaves.  
  
This manuscript represents our preliminary efforts to come up with an objective and realistic methodology for assessing the capabilities of AI systems to autonomously solve research-level math questions. After letting these ideas ferment in the community, we hope to be able to produce a more structured benchmark in a few months.

One of our primary goals is to develop a sophisticated understanding of the role that AI tools could play in the workflow of professional mathematicians.  While commercial AI systems are undoubtedly already at a level where they are useful tools for mathematicians\footnote{For instance, mathematicians are using AI tools to do literature searches, check manuscripts for  errors, write computer code, and bounce ideas.}, {\bf it is not yet clear where AI systems stand at solving research-level math questions on their own, without an expert in the loop}.  At the moment, most math benchmarks assess the performance of AI systems on math contest questions, an artificial domain that does not reflect the practice of creative mathematics by researchers.  

Evaluation of research capabilities is a challenging task.  As frontier AI systems are now highly capable of searching the literature and translating mathematical questions from one format to another, it is challenging to disentangle problem-solving capabilities from search capabilities when conducting such an assessment.   Our core observation is that an ideal test should involve {\bf research math questions which arose naturally in the process of a mathematician's own research, were subsequently solved by the mathematician, but have not yet been posted to the internet.}  

Towards this end, we present a diverse set of 10 research-level math questions, drawn from the mathematical fields of  algebraic combinatorics, spectral graph theory, algebraic topology, stochastic analysis, symplectic geometry, representation theory, lattices in Lie groups, tensor analysis, and numerical linear algebra, each of which came about naturally in the research process for one of the authors (sometimes together with collaborators).  Each question has been solved by the author(s) of the question with a proof that is roughly five pages or less,  but the answers are not yet posted to the internet.  The page restriction is due to the technical limitations of current publicly available AI systems, and this means that many of the questions on our list are not of sufficient importance to qualify as publishable research on their own, but are smaller components in future publications.

Most of the questions that we have collected are extracted from lemmas arising in larger works whose main results go beyond what current systems are capable of tackling. Significant effort is required to identify such lemmas as crucial steps in these works.   

Before explaining the nature of our evaluation, we will try to be clear about {\bf what math research is}.  Contrary to the popular conception that research is only about finding solutions to well-specified, age-old problems (e.g., Fermat's Last Theorem), most of the important parts of modern research involve  
figuring out what the question actually is  
and  
developing frameworks within which it can be answered.   
Perelman's proof of the Poincaré conjecture was a stunning achievement.  But in order for it to be possible, Thurston had to develop a new way of thinking about geometric objects and Hamilton had to invent a new kind of dynamics explaining how such objects change.

Our `first proof' experiment is focused on the final and most well-specified
stage of math research, in which the question and frameworks are already
understood. We do not address the selection of questions to study, the formulation of new definitions, and the development of novel theories. 
We wish to be clear that our choice of emphasis on proving well-formed statements is
driven by the judgment that this is a first step; evaluation of the performance
of frontier systems on the higher-level research tasks above is also essential.

The answers to our set of ten research level math questions have been encrypted and posted to 
\url{https://1stproof.org}.  
The authors will release the answers on February 13, 2026. 
{\bf We invite the community to experiment with our ten questions before the answers are released, and to share their results and observations online.} Ideally, participants should share a complete transcript of their interaction with an AI system. In this process, we hope to gain insight into questions such as: What is an appropriate prompting strategy? What format should an answer take and how should it be graded? Are there data contamination issues we have missed? 
We hope to use this understanding to design a more formal benchmark. A few months later, we plan to finalize a second set of questions; 
we are open to devising agreements to test 
AI models
on these questions prior to making them public.

Unlike other proposed math research benchmarks (see Section~\ref{sec:related}), our question list should not be considered a benchmark in its current form.  For one, our questions are not numerous enough to be considered a benchmark.   By construction, producing research-level math questions with answers which have not yet been published, and whose answers are a certain length, requires substantial human effort. A typical mathematician might create and address 
a few such questions a year. Additionally,
we have not specified a formal grading scheme for answers. While we have found correct answers to each of the questions, correct answers are not always  unique ---  there may be  multiple proofs or, alternatively, multiple counterexamples. 
This makes assessment more challenging, as it must at present be done by a human expert.

Compared to previous assessments of AI systems in completing tasks related to mathematical research (discussed in Section~\ref{sec:related} below), to the best of our knowledge, ours is the first  to simultaneously have all of the following features:
\begin{itemize}[noitemsep]
    \item The questions are sampled from the true distribution of questions that mathematicians are currently working on. Their answers are proofs, which at present must be graded by humans.
    \item The answers have never appeared on the internet, in talks, or in any public forum. This eliminates a substantial data contamination problem.  
    \item The questions are being made public in this document. This means they cannot be reused in the future, but they can be examined by everyone. 
    \item We allow models unfettered access to outside resources such as Internet searches, bringing them closer to representing real-world assessments.
\end{itemize}

We ran 
preliminary tests on many of our ten questions using GPT 5.2 Pro and Gemini 3.0 Deepthink; we briefly discuss our mitigation strategy for data contamination
in Section~\ref{sec:implementation}.  Our tests indicate that ---  when the system is given one shot to produce the answer --- the best publicly available AI systems struggle to answer many of our questions.  
In the interest of following a clear protocol, 
we chose not to iteratively interact with the systems, or even re-run the queries.
However, we expect that through such interactions we would 
be able 
to coax the systems to produce better answers. \\

\noindent {\bf Conflicts of interest.} 
No funding was received for the design or implementation of this
project. None of the authors of this report was employed by or consulted with
AI companies during the project, nor will they do so while contributing to it.\\

\noindent {\bf Acknowledgment.} We thank the Simons Institute for the Theory of Computing for hosting the organizational meeting of this project in early December 2025, with support from the Director's Opportunity Fund.  PN is supported by a research grant (VIL54509) from VILLUM FONDEN.  This statement reflects author support and does not imply sponsor involvement in the benchmark.

\section{The questions}\label{sec:problems}
\begin{enumerate}
\item Let $\mathbb{T}^3$ be the three dimensional unit size torus and let $\mu$ be the $\Phi^4_3$ measure on the space of distributions $\mathcal{D}'(\mathbb{T}^3)$. Let $\psi : \mathbb{T}^3 \to \mathbb{R}$ be a smooth function that is not identically zero and let $T_\psi : \mathcal{D}'(\mathbb{T}^3) \to \mathcal{D}'(\mathbb{T}^3)$ be the shift map given by $T_\psi(u) = u + \psi$ (with the usual identification of smooth functions as distributions). Are the measures $\mu$ and $T_\psi^* \mu$ equivalent? Here, equivalence of measures is in the sense of having the same null sets and $T_\psi^*$ denotes the pushforward under $T_\psi$.

\item Let \(F\) be a non-archimedean local field with ring of integers \(\mathfrak o\).  Let $N_r$ denote the subgroup of $\mathrm{GL}_{r}(F)$ consisting of upper-triangular unipotent elements.  Let \(\psi:F\to \mathbb C^\times\) be a nontrivial additive character of conductor \(\mathfrak o\), identified in the standard way with a generic character of $N_r$.
Let \(\Pi\) be a generic irreducible admissible representation of \(\mathrm{GL}_{n + 1}(F)\), realized in its \(\psi^{-1}\)-Whittaker model \(\mathcal W(\Pi,\psi^{-1})\).  Must there exist \(W\in \mathcal W(\Pi,\psi^{-1})\) with the following property?

Let $\pi$ be a generic irreducible admissible representation of \(\mathrm{GL}_{n}(F)\), realized in its $\psi$-Whittaker model \(\mathcal W(\pi,\psi)\).  Let $\mathfrak{q}$ denote the conductor ideal of $\pi$, let \(Q\in F^\times\) be a generator of \(\mathfrak q^{-1}\), and set
\[
  u_Q := I_{n+1} + Q\,E_{n,n+1} \in \mathrm{GL}_{n + 1}(F),
\]
where \(E_{i, j}\) is the matrix with a \(1\) in the \((i, j)\)-entry and \(0\) elsewhere.  For some \(V\in \mathcal W(\pi,\psi)\), the local Rankin--Selberg integral
\[
  \int_{N_n\backslash \mathrm{GL}_{n}(F)} W(\operatorname{diag}(g,1) u_Q)\,V(g)\,|\det g|^{s-\frac12}\,dg
\]
is finite and nonzero for all \(s\in\mathbb C\).

\item Let $\lambda=(\lambda_1 > \dots > \lambda_n \geq 0)$ be a partition with distinct parts.  Assume moreover that
$\lambda$ is {\it restricted}, in the sense that it has a unique part of size $0$ and no part of size $1$.
Does there exist a nontrivial Markov chain on $S_n(\lambda)$ whose stationary distribution is given by
$$\frac{F^*_{\mu}(x_1,\dots,x_n; q=1,t)}{P^*_{\lambda}(x_1,\dots,x_n;
q=1,t)} \text{ for }\mu\in S_n(\lambda)$$
where $F^*_{\mu}(x_1,\dots,x_n; q,t)$ and
$P^*_{\lambda}(x_1,\dots,x_n;q,t)$ are the
interpolation ASEP polynomial and interpolation Macdonald polynomial,
respectively?  If so, prove that the Markov chain you construct has the
desired stationary distribution.  By ``nontrivial'' we mean that the
transition probabilities of the Markov chain should not be described
using the polynomials $F_{\mu}^*(x_1,\dots,x_n; q,t)$. 

\item Let $p(x)$ and $q(x)$ be two monic polynomials of degree $n$:
\[
p(x) = \sum_{k=0}^n a_k x^{n-k} \quad \text{and} \quad q(x) = \sum_{k=0}^n
b_k x^{n-k}
\]
where $a_0 = b_0 = 1$. Define $p \boxplus_n q(x)$ to be the polynomial
\[
(p \boxplus_n q)(x) = \sum_{k=0}^n c_k x^{n-k}
\]
where the coefficients $c_k$ are given by the formula:
\[
c_k = \sum_{i+j=k} \frac{(n-i)! (n-j)!}{n! (n-k)!} a_i b_j
\]
for $k = 0, 1, \dots, n$.
For a monic polynomial $p(x)=\prod_{i\le n}(x- \lambda_i)$, define 
$$\Phi_n(p):=\sum_{i\le n}(\sum_{j\neq i} \frac1{\lambda_i-\lambda_j})^2$$ and $\Phi_n(p):=\infty$ if $p$ has a multiple root.
Is it true that if $p(x)$ and $q(x)$ are monic real-rooted polynomials of
degree $n$, then
$$\frac1{\Phi_n(p\boxplus_n q)} \ge \frac1{\Phi_n(p)}+\frac1{\Phi_n(q)}?$$

\item Fix a finite group $G$.  Let $\aO$ denote an incomplete transfer
system associated to an $N_\infty$ operad.  Define the slice
filtration on the $G$-equivariant stable category adapted to $\aO$ and
state and prove a characterization of the $\aO$-slice connectivity of
a connective $G$-spectrum in terms of the geometric fixed points.

\item For a graph $G = (V, E)$, let $G_S = (V, E(S,S))$ denote the graph with the same vertex set, 
but only the edges between vertices in $S$. Let $L$ be the Laplacian matrix of $G$ and let $L_S$ be the Laplacian of $G_S$. 
I say that a set of vertices $S$ is $\epsilon$-light if the matrix $\epsilon L - L_S$ is positive semidefinite. 
Does there exist a constant $c > 0$ so that for every graph $G$ and every $\epsilon$ between $0$ and $1$, $V$ contains an $\epsilon$-light subset $S$ of size at least $c \epsilon |V|$? 

\item Suppose that $\Gamma$ is a uniform lattice in a real semi-simple group, and that $\Gamma$ contains some 2-torsion. Is it possible for $\Gamma$ to be the fundamental group of a compact manifold without boundary whose universal cover is acyclic over the rational numbers $\mathbb{Q}$?

\item   A polyhedral Lagrangian surface $K$ in $\bR^4$ is a finite polyhedral complex all of whose faces are Lagrangians, and which is a topological submanifold of $\bR^4$. A Lagrangian smoothing of $K$ is a Hamiltonian isotopy $K_t$ of smooth Lagrangian submanifolds, parameterised by $(0,1]$, extending to a topological isotopy, parametrised by $[0,1]$, with endpoint $K_0 = K$.

Let $K$ be a polyhedral Lagrangian surface with the property that exactly $4$ faces meet at every vertex. Does $K$ necessarily have a Lagrangian smoothing? 

\item Let $n \geq 5$.  
Let $A^{(1)}, \ldots, A^{(n)} \in \mathbb{R}^{3 \times 4}$ be Zariski-generic.   
For $\alpha, \beta, \gamma, \delta \in [n]$, construct $Q^{(\alpha \beta \gamma \delta)} \in \mathbb{R}^{3 \times 3 \times 3 \times 3}$ so that its $(i, j, k, \ell)$ entry for $1 \leq i, j, k, \ell \leq 3$ is given by $Q^{(\alpha \beta \gamma \delta)}_{i j k \ell} = \det [A^{(\alpha)}(i, :); A^{(\beta)}(j, :); A^{(\gamma)}(k, :); A^{(\delta)}(\ell, :)]$.
Here $A(i, :)$ denotes the $i$th row of a matrix $A$, and semicolon denotes vertical concatenation. 
We are interested in algebraic relations on the set of tensors $\{Q^{(\alpha \beta \gamma \delta)} : \alpha, \beta, \gamma, \delta \in [n] \}$.

More precisely, does there exist a polynomial map $\mathbf{F}: \mathbb{R}^{81n^4} \rightarrow \mathbb{R}^N$ that satisfies the following three properties?
\smallskip
\begin{itemize}\setlength\itemsep{0.5em}
\item The map $\mathbf{F}$ does not depend on $A^{(1)}, \ldots A^{(n)}$. 
\item The degrees of the coordinate functions of $\mathbf{F}$ do not depend on $n$.
\item Let $\lambda \in \mathbb{R}^{n \times n \times n \times n}$ satisfy 
$\lambda_{\alpha \beta \gamma \delta} \neq 0$ for precisely $\alpha, \beta, \gamma, \delta \in [n]$ that are not identical.  Then $\mathbf{F}(\lambda_{\alpha \beta \gamma \delta} Q^{(\alpha \beta \gamma \delta)} : \alpha, \beta, \gamma, \delta \in [n]) = 0$ holds if and only if there exist $u, v, w, x \in (\mathbb{R}^*)^n$ such that $\lambda_{\alpha \beta \gamma \delta} = u_{\alpha} v_{\beta} w_{\gamma} x_{\delta}$ for all $\alpha, \beta, \gamma, \delta \in [n]$ that are not identical. 
\end{itemize}

\item Given a $d$-way tensor $\mathcal{T} \in \mathbb{R}^{n_1 \times n_2 \times \cdots \times n_d}$ 
such that the data is unaligned (meaning the tensor $\mathcal{T}$ has missing entries),
we consider the problem of computing a CP decomposition of rank $r$ where some modes are infinite-dimensional and constrained to be in a Reproducing Kernel Hilbert Space (RKHS). 
We want to solve this using an alternating optimization approach, and our question is focused on the mode-$k$ subproblem for an infinite-dimensional mode. 
For the subproblem, then CP factor matrices 
$A_1, \dots, A_{k-1}, A_{k+1}, \dots, A_d$ are fixed, and we are solving for $A_k$.

Our notation is as follows.
Let $N = \prod_i n_i$ denote the product of all sizes.
Let $n \equiv n_k$ be the size of mode $k$, let
$M = \prod_{i\neq k} n_i$ be the product of all dimensions except $k$, and assume $n \ll M$.
Since the data are unaligned, this means only a subset of $\mathcal{T}$'s entries are observed, and we let $q \ll N$ denote the number of observed entries.
We let $T \in \mathbb{R}^{n \times M}$ denote the mode-$k$ unfolding of the tensor $\mathcal{T}$ with all missing entries set to zero.
The $\vecop$ operations creates a vector from a matrix by stacking its columns,
and we let $S \in \mathbb{R}^{N \times q}$ denote the selection matrix (a subset of the $N \times N$ identity matrix) such that $S^T \vecop(T)$ selects the $q$ known entries of the tensor $\mathcal{T}$ from the vectorization of its mode-$k$ unfolding.
We let $Z = A_d \odot \cdots \odot A_{k+1} \odot A_{k-1} \odot \cdots \odot A_1 \in \mathbb{R}^{M \times r}$ be the Khatri-Rao product of the factor matrices corresponding to all modes except mode $k$.
We let $B = TZ$ denote the MTTKRP of the tensor $\mathcal{T}$ and Khatri-Rao product $Z$.

We assume $A_k = KW$ where
$K \in \mathbb{R}^{n \times n}$ denotes the psd RKHS kernel matrix for mode $k$.
The matrix $W$ of size $n \times r$ is the unknown for which we must solve. 
The system to be solved is
\begin{equation} 
	\left[
    (Z \otimes K)^T S
    S^T (Z \otimes K)
    + \lambda (I_r \otimes K) 
  \right] \vecop(W)
	= (I_r \otimes K) 
	\vecop( B ). \nonumber 
\end{equation}
Here, $I_r$ denotes the $r \times r$ identity matrix.
This is a system of size $nr \times nr$
Using a standard linear solver costs $O(n^3 r^3)$, 
and explicitly forming the matrix is an additional expense.

Explain how an iterative preconditioned conjugate gradient linear solver can be used to solve this problem more efficiently. Explain the method and choice of preconditioner. Explain in detail how the matrix-vector products are computed and why this works. Provide complexity analysis. 
We assume $n,r < q \ll N$. Avoid any computation of order $N$.

\end{enumerate}

\section{Related work}\label{sec:related}

As mentioned earlier, there have been several proposed math research benchmarks.  We discuss a few of them here.

 FrontierMath \cite{frontier} is a benchmark of ``several hundred unpublished, expert-level mathematics problems that take specialists hours to days to solve.''  
It was funded by OpenAI.  Presently, the FrontierMath problems are private
(apart from 12 examples that are publicly available). OpenAI has access to a subset of FrontierMath problems  and solutions, and EpochAI has access to the full set of solutions. The FrontierMath problems are structured so that each final answer is an integer or symbolic expression, which makes them automatically gradable, as well as amenable to post-training via reinforcement learning. 

IMProofBench \cite{schmitt2025improofbench} is a broader mathematical proof benchmark, designed to evaluate the ability of AI systems to create research-level mathematical proofs.  The problems are designed to allow for automatic grading of subquestions, but still require human experts to fully verify correctness. 
 The IMProofBench questions are private. 

The RealMath benchmark for research-level math questions \cite{zhang2025realmath} scrapes (i.e. collects papers automatically from)
math and computer science categories in \href{https://arxiv.org}{arXiv.org}, skewing toward fields with ``constructive'' theorems like probability and statistics.  It only scrapes questions posted after the ``training data cutoff'' of the AI models being tested, where training data cutoff refers to the final date from which web data was collected and used for training data.  Like FrontierMath, the RealMath questions are designed to facilitate automatic grading, with a final short symbolic or numeric answer. Unlike FrontierMath and IMProofBench, the RealMath questions are public and intended to be refreshed every so often to avoid data contamination.

\section{Implementation details}\label{sec:implementation}

Over the span of a few weeks, we tested roughly 20 research-level math questions using Gemini 3 Pro, GPT-5.1 Pro, and then GPT-5.2 Pro when GPT 5.2 Pro became available. The final selection of questions used the following criteria:
  \begin{enumerate}[noitemsep]
  \item Use of the AI system did not reveal the existence of a previous answer to the question that was unknown to the authors.
  \item A one page statement was sufficient for the systems to ``understand'' the formulation of the question, i.e. it was able to reformulate the question in its own language before starting to answer it. 
  \item Agreement was reached with the authors of the question to release a human generated proof within the required parameters (length and timeframe).
 \item No member of the team contributed more than one problem.
   \end{enumerate}
The reason for testing more than 10 questions was to probe the ``boundary'' between the types of questions the models can solve and the types of questions beyond their reach.  To minimize data contamination, we turned off the option to share data for training and improving models, but we are aware that data is still retained for 3 days by Google, and 30 days by OpenAI\footnote{According to our reading of the OpenAI \href{https://help.openai.com/en/articles/8983778-chat-and-file-retention-policies-in-chatgpt}{Terms of Service}, a chat can be retained longer than 30 days if the chat has been de-identified and disassociated from the author.  According to our reading of the Gemini \href{https://support.google.com/gemini/answer/13594961}{Terms of Service}, chats reviewed by human reviewers may be retained for up to 3 years.}.  Throughout the process, we have endeavored  to keep the answers to our questions private. We have uploaded encrypted answers to the private repository,  
\url{https://1stproof.org}.
We will make the answers publicly available about a week 
after we release the questions.

\section{Discussion}
We have presented a set of ten research-level mathematics questions.  
As mentioned earlier, mathematical research consists of multiple components, including:
\begin{itemize} [noitemsep]
\item creating and selecting the questions to study, which will guide and shape the field;
\item developing novel theories for approaching these questions, including formalizing new definitions and frameworks;
\item finding answers to the selected questions, and rigorously proving that these answers are correct.
\end{itemize}
Our `first proof' experiment is focused on the final, most well-specified, and most measurable stage of mathematical research, that is, finding answers to the selected questions. We do not address the  question of evaluating whether AI systems can reasonably create  questions to study, or develop novel theories. 

We plan to create a second set of questions of the same nature as the ones in Section~\ref{sec:problems} in the coming months, and we are open to devising agreements to test frontier AI systems on the second set of questions before we release them.  We hope that this second set of questions can serve as a form of benchmark for testing the capabilities of AI. 

Beyond the next release, depending on technological developments, we plan to release additional sets of questions by removing some of the artificial constraints we imposed on our chosen questions, such as length, as well as to explore ways of measuring performance along other aspects of the work of research mathematics.

\appendix

\section{First Proof Solutions and comments}

On February $4$ and $5$, 2026, we tested the questions on Gemini 3.0 Deep Think and ChatGPT 5.2 Pro
with the following prompts.
\begin{quote}
    {\bf Prompt 1.}
    The following is a research-level math question.  The question has an answer, but it might not appear on the internet. Please make a best effort to provide a rigorous and complete answer to the question. Write the output as a compilable LaTeX document using the standards of rigor and scholarship that prevail in the mathematical literature. 

    {\bf Prompt 2 (Internet Discouraged).}
    The following is a research-level math question.  The question has an answer. Please make a best effort to provide a rigorous and complete answer to the question. Write the output as a compilable LaTeX document using the standards of rigor and scholarship that prevail in the mathematical literature.  Do not use web search, but instead try to reason through the answer.
    
\end{quote}
In the following subsections, we comment briefly on the best LLM solutions that we obtained in these internal tests.

\subsection{Question 1: Martin Hairer}

In this case, a note with a very short sketch of proof (far short of the level of detail one would expect
for a published article) was posted on the author's homepage some time ago. The answer given by
GPT-Pro simply quotes that note, claiming that it contains a detailed proof of the result. 
This is incorrect and it is despite the LLM being 
specifically instructed to comply with ``mathematics publication'' levels of scholarship. (Taking for granted
a result that is merely stated in an unpublished note with a very rough sketch of proof is not 
considered acceptable in the mathematics literature.)

Another behaviour we observed was that the LLM would take as a premise
the (wrong!) statement that the $\Phi^4_3$ measure is equivalent to the free field measure, from which 
it then correctly deduces the (incorrect) claim that the $\Phi^4_3$ measure is quasi-invariant under
smooth shifts.

\subsection{Question 2: Paul Nelson}

In some attempts, the LLM constructed $W$ depending on $\pi$, but the problem asks for a single $W$ that works for all $\pi$.  This is a critical condition; without it, the problem is much easier and the solution is well-known.  In some (but not all) cases, the LLM noted that it had solved a weaker problem.

In the best attempt in our trial runs, 
ChatGPT 5.2 Pro identified a suitable choice of $W$ and reduced (as in our solution) to exhibiting $V$ for which the integral
$\int_{\mathrm{GL}_n(\mathfrak{o})} V(g) \psi(- Q g_{n n}) \, d g$ does not vanish.  This nonvanishing is the key point.

ChatGPT then attempted to choose $V$ so that the integrand is constant on its support, which, if possible, would make the nonvanishing clear.  This strategy is unviable.  For instance, when $n = 1$, $V$ must be (a nonzero multiple of) a character of $F^\times$ and the integral is a normalized Gauss sum; in particular, the integrand is typically non-constant.  For larger $n$, the unviability follows similarly by considering the action of the center.

To identify the specific error in the attempted solution, we look for the first place asserting stronger support properties of $V$ than are generally true.  The culprit is the support condition claimed in the ``standard Howe-vector existence result,''
which never holds: it contradicts the fact that $V$ has a central character.

\subsection{Question 3: Lauren Williams}
The best solution that LLM's produced for Question 3 in our internal experiments was to use the Metropolis-Hastings algorithm to produce a Markov chain whose stationary distribution had the desired formula.  However, by design, the Metropolis-Hastings algorithm uses the desired formula to define its transition rates.  This algorithm can be used to cook up a Markov chain with \emph{any} desired distribution.  Hence this is considered a ``trivial'' solution to the problem (which specifically asked that the transition probabilities not be described in terms of the interpolation polynomials).  Sometimes the LLM's would give a slight variant of the above trivial solution where they would replace the interpolation polynomials by an equivalent formula for them (the signed multiline queue formula of Ben Dali--Williams).   

Another common response given by LLM's was to change the problem to a related but different, and already-solved problem, namely, 
to replace interpolation ASEP and interpolation Macdonald polynomials by ASEP and Macdonald polynomials.  In this case the solution to this problem is the $t$-Push TASEP and was given in a paper by Ayyer, Martin, and Williams.  

\subsection{Question 4: Nikhil Srivastava}
The only attempt at the general $n\ge 4$ case of this question was made by ChatGPT Pro 5.2 with the no internet prompt. After collecting some standard  facts in the first three pages, its plan was to execute Blachman's approach to the classical Stam inequality (Section 4). In this approach the key step is to identify the score function of a sum of independent random variables $X+Y$ as a conditional expectation of the score function of $X$ conditioned on $X+Y$, in the appropriate joint probability space, after which the inequality reduces to Cauchy-Schwartz. The main difficulty is finding an analogue of this joint probability space in the finite free setting.

The LLM attempted to find a probability space in which a score function could live by considering the random matrix model for the finite free convolution  $r(x)=p\boxplus_n q(x)=\mathbb{E} \mathrm{det}(xI-A-UBU^T)$. It gathered some facts about $r(x)$ for large real $x$ away from the roots, asserted wrongly that $\Phi_n(r)$ can be read off from residues of $(r'(x)/r(x))'$ at the roots of $r(x)$, and then asserted that the proof can be finished via the residue calculus without giving details. This sequence of steps did not make sense to me.

At a conceptual level, this proof strategy cannot succeed because only the score function of $r(x)$ is considered, and the score functions of $p(x),q(x)$ are never mentioned. It also does not exploit the fact that $\boxplus_n$ preserves real roots, which must be used since the inequality is not true for arbitrary polynomials.

\subsection{Question 5: Andrew J. Blumberg}

The best solutions by Gemini and ChatGPT 5.2 Pro contained an essentially correct statement of the definition of the $\mathcal{O}$-slice filtration and the connectivity characterization.  The proofs offered, like the proof from the work with Michael A. Hill and Tyler Lawson which generated this question, closely follow the basic outline of a previous paper by Hill-Yarnall.  However, in each case, some of the details were either sketched or slightly garbled.  For example, the ChatGPT solution claims to be working in the $\mathcal{O}$-stable category, but is breezy about what is required (and subsequent statements it makes are then missing hypotheses).  Section 4 introduces and uses the notion of ``geometric objects'' from Hill-Yarnall without defining them.  The Gemini solution outline an argument for sufficiency of the condition which is more of a sketch than an argument.

A number of LLM runs produced serious hallucinations, citing lemmas that did not exist from Hill-Hopkins-Ravenel or in one case confabulating an entire paper and attributing the result to this putative source.  Some also contained seriously false statements, for example about the spectra to which the tom Dieck splitting applies.

\subsection{Question 6: Daniel Spielman}

Gemini asserted that it presented a proof of the existence of a constant that satisfied Question 6. 
But, after some correct statements, it presented a very vague explanation of how the proof could be finished. To me, it seems unlikely that the approach can be turned into a correct proof.

ChatGPT 5.2 Pro asserted that it could not answer the question.
So, it instead offered a correct upper bound of $1/2$ on the constant, if it exists.

\subsection{Question 7: Shmuel Weinberger}

In the no internet version, Theorem~4 and in the internet version it
is Lemma~5, are false (they are the same statement). The
counterexample is $\bR^1$ and $f$ is a translation. It has no fixed points,
but its Lefschetz number in their sense is $-1$.

The AI proofs only use finite complex and
Poincar\'e duality. However, Fowler's paper shows that if $\Gamma$ is a
lattice in a linear semisimple group $G$, then taking a homomorphism
from $\Gamma$ to a finite group $\Delta$, with kernel $\Gamma_0$ torsion free,
the product $M^3 \times (K\backslash G/\Gamma_0 \times E\Delta)/\Delta$, where $E\Delta$ is a
contractible space with free $\Delta$ action, and $M^3$ is any closed
hyperbolic $3$-manifold, has the rational type of a finite complex, and
satisfies Rational Poincar\'e duality. It has fundamental group
$\pi_1(M^3) \times \Gamma$ which is a lattice in $\mathrm{SO}(3,1) \times G$. This shows that
all such proofs must fail.

 Some proofs try to use ``multiplicativity of Euler
characteristic in finite covers''. This is false for infinite
complexes with finitely generated homology over $\mathbb{Q}$. The simplest
example I know is the following: Consider the universal cover of $\bR P^2$
wedge an infinite number of $S^2$'s. It has an involution, and $\pi_2$ is
$\mathbb{Z}[-1] + \mathbb{Z}[\mathbb{Z}/2]^\infty$. ($\mathbb{Z}[-1]$ is $\mathbb{Z}$ acted on by the involution by
multiplication by $-1$.) This module is, after tensoring with $\mathbb{Z}[1/2]$ a
free $\mathbb{Z}[1/2][\mathbb{Z}/2]$ module, so one can use a free basis to equivariantly
attach $D^3 \times \mathbb{Z}/2$'s to kill the homology ($=$homotopy). The new space
will be rationally acyclic, and both it and its quotient under $\mathbb{Z}/2$
will be, and will have rational Euler characteristic $= 1$.

\subsection{Question 8: Mohammed Abouzaid}
The best two solutions produced during testing both correctly identified the existence of a local smoothing near every vertex; the proof uses essentially the same basic linear algebra argument that appears in the human solution. The proof then proceeds to perform a local-to-global gluing argument. It was a priori clear that there must be a gap in this argument because the LLM solution refers to the existence of a linear symplectic transformation that brings a neighbourhood of each vertex and each edge into a standard position, but fails to discuss the compatibility between these choices. In the case of the solution produced by the model which was not discouraged to use the internet, the error was finally identified, after a careful reading, in Step 3 of the Proof of Theorem 1: the LLM system asserted that one can choose disjoint neighbourhoods of the edges and of the vertices. In the other case, the error is in Step 2: the model performs a local move near vertices, which changes the local geometry near the edges, invalidating the application of the edge move. 

The errors in these solutions can be repaired at the cost of significant computations of changes of coordinates, which would become extremely burdensome in any generalisation. The point of the solution we provide is to obtain a proof which avoids (most of) the hard work, and which experts can readily generalise to other symplectic manifolds (in any dimension).

\subsection{Question 9: Joe Kileel}
The best LLM answer found during testing was NoInternet-040226.
This is an essentially correct answer.  
It constructs the same algebraic relations as in my own answer, namely the various $5 \times 5$ minors of the four $3n \times 27n^3$ flattenings of the block tensor  assembling together the $Q^{(\alpha \beta \gamma \delta)}$.
The proof by the LLM that the algebraic relations satisfy the desired properties differs from my own argument.  
The LLM considers a torus action on an appropriate Grassmannian, argues the stabilizer of a generic point is 1-dimensional, and  uses this to show separability of $\lambda$ in a somewhat fidgety way.  
By contrast, I directly constrain $\lambda$ by considering  certain selected  algebraic relations.  
Some other LLM answers produced during testing were incorrect, and claimed that no algebraic relations exist that satisfy the desired properties.  
Those answers seemed to get confused about the question setup midway through.
My question is closely related to a work I published with  Miao and Lerman in 2024 (\url{https://proceedings.neurips.cc/paper_files/paper/2024/hash/80cddcdd52c84d19b8b4a27a8e8c17d8-Abstract-Conference.html}). Indeed, it is a fourth-order variant of Theorem 2 in that paper which concerns the third-order case.  Therefore, if LLMs locate and understand that paper they would have a warm-start for this question.

\subsection{Question 10: Tammy Kolda}

The best LLM solution was correct and better than the solution I provided in that it lowered the computational complexity. Most importantly, it had an insight that was obvious in hindsight but that I had not seen yet myself. Since LLMs are well known to surface existing solutions, I tried search on ``subsampled kronecker product matvec'' and found that the main idea in the solution exists in \url{https://arxiv.org/pdf/1601.01507}. (I am not sure if this is the only source of the solution, but it is at least one such solution.)

The LLM solution did not meet the standards of including appropriate citations, but it was otherwise a good solution. The solution I had provided included a transformation of the problem that the LLM did not do, but the problem was open-ended and this was not necessary. I am planning to borrow aspects of the LLM solution, although I hope to do a better job at attribution of the ideas.

\section{The human-generated solutions to our problems}
Our solutions appear below. They are identical to the solutions which were encrypted on February 5, 2026, and released on February 13, 2026.

\subsection{Question 1: Martin Hairer.}
\textit{Authors:} Martin~Hairer and Jacopo Peroni\par\medskip
\textit{Title:} (Lack of) quasi-shift invariance of the $\Phi^4_3$ measure\par\medskip

This is a simplified version of a result that will appear as part of \cite{f10:Jacopo}. 
The proof relies strongly on the ideas from \cite{f10:HKN24}.

\begingroup
\providecommand{\eqdef}{\stackrel{\mbox{\tiny def}}{=}}

\numberwithin{equation}{section}

\newtheorem{filejlemma}{Lemma}[section]
\newtheorem{filejproposition}[filejlemma]{Proposition}

\providecommand{\id}{\operatorname{id}}
\def\Wick#1{\mathopen{:}#1\mathclose{:}}

\providecommand{\EE}{\mathbb{E}}     
\providecommand{\RR}{\mathbb{R}}      
\providecommand{\TT}{\mathbb{T}}
\providecommand{\N}{\mathbb{N}}

\providecommand{\cC}{\mathcal{C}}
\providecommand{\dD}{\mathcal{D}}

\let\f\frac

\def\d{\partial}

\providecommand{\eps}{\varepsilon}

\colorlet{symbols}{blue!90!black}
\colorlet{testcolor}{green!60!black}

\tikzset{
	smalldot/.style={circle,fill=symbols,draw=symbols, solid,inner sep=0pt,minimum size=0.5mm},
}

\def\scal#1{\langle#1\rangle}
\def\cent#1{\mathopen{{\langle\kern-0.3em\rangle}}#1\mathclose{{\langle\kern-0.3em\rangle}}}
\def\bscal#1{\big\langle#1\big\rangle}
\def\Bscal#1{\Big\langle#1\Big\rangle}

\makeatletter
\def\DeclareSymbol#1#2#3{\expandafter\gdef\csname MH@symb@#1\endcsname{\tikz[baseline=#2,scale=0.15,draw=symbols]{#3}}\expandafter\gdef\csname MH@symb@#1s\endcsname{\scalebox{0.65}{\tikz[baseline=#2,scale=0.15,draw=symbols]{#3}}}}
\def\<#1>{\csname MH@symb@#1\endcsname}
\makeatother

\DeclareSymbol{0}{0}{\draw[white] (-.4,0) -- (.4,0); \draw (0,0)  -- (0,1.8) node[smalldot] {};}
\DeclareSymbol{1}{-1}{\draw[white] (-.4,0) -- (.4,0); \draw (0,-0.5)  -- (0,1.3) node[smalldot] {};}
\DeclareSymbol{2}{-1}{\draw (-0.7,1.3) node[smalldot] {} -- (0,-0.5) -- (0.7,1.3) node[smalldot] {};}
\DeclareSymbol{3}{1}{\draw (-1,1.5) node[smalldot] {} -- (0,0) -- (1,1.5) node[smalldot] {}; \draw (0,2) node[smalldot] {} -- (0,0);}
\DeclareSymbol{20}{-.2}{\draw (-0.7,1.5) node[smalldot] {} -- (0,0.5); \draw (0.7,1.5) node[smalldot] {} -- (0,0.5) -- (0,-0.5);}
\DeclareSymbol{30}{0}{\draw (-1,1.5) node[smalldot] {} -- (0,0.5) -- (1,1.5) node[smalldot] {}; \draw (0,2) node[smalldot] {} -- (0,0.5) -- (0,-.5);}
\DeclareSymbol{31}{0}{\draw (-1,1.5) node[smalldot] {} -- (0,0.5) -- (1,1.5) node[smalldot] {}; \draw (0,2) node[smalldot] {} -- (0,0.5) -- (0,-.5) -- (1,0.5) node[smalldot] {} ;}
\setcounter{equation}{0}
Let $\TT^3$ be the three dimensional unit size torus and let $\mu$ be the $\Phi^4_3$ measure on the space of distributions $\dD'(\TT^3)$. Let $\psi : \TT^3 \to \RR$ be a smooth function that is not identically zero and let $T_\psi : \dD'(\TT^3) \to \dD'(\TT^3)$ be the shift map given by $T_\psi(u) = u + \psi$ (with the usual identification of smooth functions as distributions). Is the statement ``the measures $\mu$ and $T_\psi^* \mu$ are equivalent'' true? Here, equivalence of measures is in the sense of having the same null sets and $T_\psi^*$ denotes the pushforward under $T_\psi$.

\subparagraph{Some Context}

One of the very few interacting quantum field theories that can be rigorously constructed is the
so-called (bosonic) $\Phi^4$ theory in (space-time) dimensions $2$ and $3$. It has long been known that
in dimension $2$ and finite volume there is a natural identification between the Hilbert space of the
interacting theory and that of the corresponding free theory. On the other hand, 
Glimm \cite{f10:Glimm} observed that this is no longer the case in dimension $3$. At the level of the 
corresponding Euclidean theories (which are represented by probability measures on the
space of Schwartz distributions on the corresponding space-time), this translates into the fact that the 
$\Phi^4$ measure $\mu$ and the corresponding free field measure $\nu$ are equivalent in dimension $2$
but mutually singular in dimension $3$. In fact, there is a sense in which the dimension that delimits
between the two behaviors is $8/3$. It is then natural to ask in which dimensions $\mu$
has the weaker property that $\mu$ and $T_\psi^* \mu$ are equivalent for smooth $\psi$. Here it turns out
that the borderline dimension is $3$, and the question probes on which side it falls.

\paragraph{An incomplete heuristic}

Regarding the proof, a tempting heuristic is to use the fact that one should think of $\mu$ as having the
density with respect to ``Lebesgue measure on $\dD'$'' (which of course doesn't exist) proportional to
\begin{equ}
\exp \Bigl(- \int_{\TT^3} \Bigl(\frac12 |\nabla\Phi(x)|^2 + \frac 14|\Phi(x)|^4 - \frac C2 |\Phi(x)|^2\Bigr)\,dx\Bigr)\;,
\end{equ}
where $C$ is a (diverging) constant of the form $C = 3c_1 - 9 c_2$, where
$c_1$ is the expectation of $|\Phi(x)|^2$ under the free field measure $\nu$ (which is of course infinite)
and $c_2$ is an additional logarithmically divergent constant. The density of $T_\psi^* \mu$
with respect to $\mu$ is then formally given by
\begin{equs}
\exp \Bigl(- \int_{\TT^3} &\Bigl(\frac12 |\nabla\psi(x)|^2 + \frac 14|\psi(x)|^4 + \Phi(x) \Delta \psi(x)
- \Phi(x)\psi^3(x) \\
& - \psi(x) \bigl(\Phi^3(x) - C\Phi(x)\bigr) + \frac{\psi^2(x)}{2} \bigl(3\Phi^2(x) - C\bigr)\Bigr)\,dx\Bigr)\;,
\end{equs}
Since the terms on the first line are well-defined for smooth $\psi$ and 
one expects $\Phi^3 - C\Phi$ and $\Phi^2 - c_1$ to be quite well-behaved, the additional
logarithmically divergent term proportional to $c_2$ causes this ``density'' to diverge, suggesting
(correctly) that $\mu$ and $T_\psi^* \mu$ are mutually singular.

There are at least two problems with such an approach. First, $\Phi^3 - C\Phi$ 
does actually \textit{not} define a random distribution, whether $\Phi$ is distributed according
to $\mu$ or to the free field $\nu$ (which guides the heuristic). This is because if it were,
it would have a covariance behaving like $|x-y|^{-3}$ around the diagonal, which is not integrable 
in dimension $3$. The second problem is that such an argument suggests that, if 
$\mu_n = \exp(- f_n)\,\nu$ for some ``nice'' probability measure $\nu$ and functions $f_n$ that 
fail to converge to a ``nice'' limit, then $\mu_n$ fails to converge to a limit $\mu$. This 
of course is not true: for a suitable (diverging) sequence of constants $c_n$, the sequence
$f_n(x) = c_n + n \cos(nx)$ is such that if $\nu$ is Lebesgue measure on $[0,1]$, then $\mu_n$ 
converges weakly to Lebesgue measure even though the log-densities $f_n$ fail to converge.
Any proof needs to be based on a different approach or to satisfactorily address these problems.

\subparagraph{Notations}

We fix a space-time white noise $\xi$ on $\RR \times \TT^3$.
We define \<1> as the stationary solution to the linear equation
    \[
        (\d_t + 1-\Delta)\<1> = \xi, \qquad \mbox{on } \RR \times \TT^3\;.
    \]
    (We use the convention that symbols represent random space-time distributions rather than elements of a regularity structure.)
Starting from this process, we define \<2> and \<3> as its Wick square and cube respectively, which are given by
\begin{equ}
\<2> = \lim_{N \to \infty} H_2(\<1>_N,c_{N})\;,\qquad
\<3> = \lim_{N \to \infty} H_3(\<1>_N,c_{N})\;,
\end{equ}
where $\<1>_N = P_N \<1>$ and $c_{N} = \EE \<1>_N^2$ (which is constant in space and time).
Here, $P_N$ denotes the projection onto Fourier modes with $|k| \le N$ and $H_n$ denotes the $n$th
Hermite polynomial normalised such that $H_0 \equiv 1$, $H_n' = nH_{n-1}$, and 
$\EE H_n(Z,1) = 0$ for a normal random variable $Z$. 
The first convergence takes place in the space of continuous functions of time with values
in $\cC^{-1-2\kappa}$, while the second convergence takes place in the space-time parabolic 
space $\cC^{-\frac32-3\kappa}$.

With these notations in place, we define \<20> as the stationary solution to 
    \[
        (\d_t + 1-\Delta)\<20> = \<2>\;,
    \]
    and similarly for \<30>.
    For a more comprehensive and pedagogical introduction to the general tree-like notation, we refer the reader to \cite{f10:MWX15}.
    We also write ``$\prec$'' for Bony's paraproduct (in space) as defined for example in \cite[Sec.~2.1]{f10:GIP12}
    and, given a random $N$-dependent process $w$, we will sometimes use the physicists shorthand notation $\Wick{w^k}$
    instead of $H_k(w,c_{N})$.

\subparagraph{Answer and Proof}

The statement is \textbf{false}. In particular, for any smooth function $\psi \not \equiv 0$ and any choice of the
parameters involved in the definition of $\mu$ (mass and coupling constant, provided that the latter is non-zero),
the measures $\mu$ and $T_\psi^* \mu$ are  mutually singular. 

For notational simplicity, we fix the mass and the
coupling constant to $1$, but this has no incidence on the proof.
Our main starting point is the following statement, a proof of which 
can be found for example in \cite{f10:Konstantin} and \cite[Lemma 4.19]{f10:EW24} for \eqref{e:processu}, combined with \cite{f10:CC18}
for \eqref{e:paracontrolled} (see Ansatz~2.11 there). 
Throughout this proof, $\kappa>0$ is chosen small enough 
($\kappa = 1/100$ is certainly sufficient).

\begin{filejproposition} \label{f10:v-regularity}
    There exists a stationary process $v$ that is almost surely continuous
    in time with values in $\cC^{1-2\kappa}(\TT^d)$ and such that the process
    \begin{equ}[e:processu]
        u = \<1> - \<30> + v\;,
    \end{equ}
	is stationary with fixed time distribution equal to $\mu$. Furthermore, the process $v$ is such that
    \begin{equ}[e:paracontrolled]
        v = -3 \bigl((v -\<30>) \prec \<20>\bigr) + v^\sharp\;,
    \end{equ}
    where $v^\sharp$ is continuous with values in $\cC^{1+4\kappa}(\TT^d)$. 
\end{filejproposition}

It was furthermore shown in \cite[Lemmas~3.1 \& 3.4]{f10:HKN24} (but see \cite{f10:Hai14a} for a similar
result using a slightly different regularisation) that the 
processes $\<1>$ and $\<30>$ are almost surely continuous in time with values in 
$\cC^{-\frac12-\kappa}$ and $\cC^{\frac12-3\kappa}$ respectively.

Before we proceed, we remind some notations and preliminary results.
First of all, we define the additional diverging constant
\begin{equation}
    c_{N,2} := \EE \left [ \<2>_N \<20>_N \right ],
\end{equation}
where $\<2>_N:=P_N\<2>$ and $\<20>_N := P_N\<20>$.
The main ingredient of our proof is the event
\begin{equ}
    B^{\gamma}:= \left \{ u \in \dD' : \lim_{N \to \infty}\left \langle (\log N)^{-\gamma} \left ( H_3 \left ( P_N u; c_{N} \right ) + 9c_{N,2} P_Nu\right ),\psi \right \rangle_{\TT^3} = 0 \right \},
\end{equ}
which will be used to distinguish between the shifted and the non-shifted measures.
Here, the limit $N \to \infty$ is restricted to en exponentially growing sequence, for example $N \in 2^{\N}$.

We will also use the following two technical lemmas whose proofs can be found in Section~\ref{f10:proof-lemma-section}. These are very similar to \cite[Lemma 3.11 and Lemma 3.12]{f10:HKN24}.
\begin{filejlemma} \label{f10:Z_to_zero_sigma}
Let  $\gamma > \frac12$. Then, for any fixed $t>0$,
        \begin{equ}
            \lim_{N\to\infty} (\log N) ^{-\gamma} \Wick{\left (  \<1>_N \right )^3}(t)  = 0
        \end{equ}
        almost surely in $\cC^{-\frac{3}{2}}\bigl ( \TT^{3} \bigr )$ and in $L^p \bigl ( \Omega; \cC^{-\frac{3}{2}} \bigl ( \TT^3 \bigr ) \bigr )$ for any $p >0$.
\end{filejlemma}
\begin{filejlemma} \label{f10:cN2_div_frac}
    For $N$ large, one has   $ c_{N,2} \gtrsim 
            \log N $.
\end{filejlemma}

The following results are essentially standard, but we recall their statements for later reference.

\begin{filejlemma}\label{f10:lem:convSimple}
For any polynomial $P$, the expression $\<1>_N P(\<30>_N)$ converges almost surely to some finite limit in 
$\cC^{-\frac12 -\kappa}$.
\end{filejlemma}

\begin{proof}
By paralinearisation and standard commutator estimates (see \cite[Lems~2.4 \& 2.6]{f10:GIP12}) it suffices to 
consider the case $P(x) = x$. This is by now standard, see for example \cite[Sec.~4.4]{f10:CC18}.
\end{proof}

\begin{filejlemma}\label{f10:lem:convParacontrolled}
Let $v$ be a process satisfying the decomposition \eqref{e:paracontrolled}. Then, the expressions
$\Wick{\<1>^2_N}\<30>_N - 3c_{N,2} \<1>_N$ and $\Wick{\<1>_N^2}v_N + 3c_{N,2} (v_N - \<30>_N)$ both converge
almost surely to finite limits in $\cC^{-1 -2\kappa}$ as $N \to \infty$.
\end{filejlemma}

\begin{proof}
Regarding the first expression, its convergence was essentially for example in \cite[Sec.~4.6]{f10:CC18}. (The
approximation used there is slightly different, but the differences are unimportant.)
Regarding the second expression, the claim follows from \cite[Sec.~4.5]{f10:CC18} (modulo again unimportant changes in 
the approximation scheme), combined with the commutator estimate \cite[Lem.~2.4]{f10:GIP12}.
\end{proof}

We now turn to the proof of the main claim. For this, we first claim that if $u$ is as in \eqref{e:processu},
then, for any fixed $t$, one has $u(t) \in B^\gamma$.
Indeed, writing $u_N$ as a shorthand for $P_N u$ and expanding the Wick power, we have
    \begin{align*}
        (\log N)^{-\gamma} &H_3 \left ( u_N; c_{N} \right )
        =(\log N)^{-\gamma} H_3 \left (  \left ( \<1>_N - \<30>_N + v_N \right ); c_{N} \right )\\
        &=(\log N)^{-\gamma} \sum_{i=0}^3 \binom{3}{i} \Wick{\left (  \<1>_N \right )^i} \left (- \<30>_N + v_N \right )^{3-i}\\
        &=(\log N)^{-\gamma} \sum_{i=0}^3 \sum_{j=0}^{3-i} \binom{3}{i}\binom{3-i}{j} \Wick{\left (  \<1>_N \right )^i} \left ( - \<30>_N \right )^{j}\left ( v_N \right )^{3-i-j}\\
        &=(\log N)^{-\gamma}\Wick{ \<1>_N^3}
        - 3(\log N)^{-\gamma}\Wick{\<1>_N^2} \, \<30>_N  + 3(\log N)^{-\gamma}\Wick{\<1>_N^2}  \, v_N \\
        &\quad+ (\log N)^{-\gamma} \sum_{\substack{0 \le i+j \le 3 \\ (i,j) \neq (3,0),(2,1),(2,0)}} \binom{3}{i}\binom{3-i}{j} \Wick{\left (  \<1>_N \right )^i} \left ( - \<30>_N \right )^{j}\left ( v_N \right )^{3-i-j}.
    \end{align*}
    The first term $(\log N)^{-\gamma}\Wick{\<1>_N^3}$ and the terms present in the last sum all converge to $0$ by Lemma~\ref{f10:Z_to_zero_sigma} (given that $\gamma > \frac{1}{2}$), standard product estimates 
(e.g. \cite[Theorem 2.5]{f10:Bon81} or \cite[Proposition 2.3]{f10:MWX15}) and Lemma~\ref{f10:lem:convSimple}.
    
    It therefore remains to show that $-\Wick{\<1>^2_N}\<30>_N + \Wick{\<1>_N^2}v_N + 3c_{N,2} u_N$ also
converges to zero almost surely in the sense of distributions. We rewrite this term as
\begin{equ}
\Wick{\<1>_N^2}v_N + 3c_{N,2} (v_N - \<30>_N) -\bigl(\Wick{\<1>^2_N}\<30>_N - 3c_{N,2} \<1>_N\bigr)\;.
\end{equ}
By Lemma~\ref{f10:lem:convParacontrolled} we know that this expression converges to an element of $\cC^{-1-2\kappa}(\TT^d)$, whence we conclude that
    \[
        \bigl\langle\left ( \log N \right )^{-\gamma} \left ( -\Wick{\<1>^2_N}\<30>_N + \Wick{\<1>_N^2}v_N + 3c_{N,2} u_N \right ),\psi\bigr\rangle \xrightarrow{N \to \infty} 0
    \]
    almost surely, thus proving that $\mu(B^{\gamma}) = 1$.
    
    In order to conclude the proof, it suffices to show that $u + \psi \not \in B^\gamma$. 
For this, we expand similarly to before the expression appearing in this event as
    \begin{align*}
        (\log N)^{-\gamma}&H_3 \left( \left ( u_N + \psi_N \right ); c_{N} \right ) + (\log N)^{-\gamma}9c_{N,2} \left ( u_N + \psi_N \right )\\
        &=(\log N)^{-\gamma} \sum_{i=0}^3 \binom{3}{i} \Wick{\left ( u_N \right )^i} \left ( \psi_N \right )^{3-i} + (\log N)^{-\gamma}9c_{N,2} \left ( u_N + \psi_N \right )\\
        &=(\log N)^{-\gamma}\Wick{\left ( u_N \right )^3} + (\log N)^{-\gamma}9c_{N,2} u_N\\
         &\qquad + 3(\log N)^{-\gamma} \Wick{\left ( u_N \right )^2} \left ( \psi_N \right ) + 3(\log N)^{-\gamma} \left ( u_N \right ) \left ( \psi_N \right )^2\\
         &\qquad+ (\log N)^{-\gamma}\left ( \psi_N \right )^3 + (\log N)^{-\gamma}9c_{N,2} \psi_N.
    \end{align*}
    The sum of the first two terms was just shown to converge to $0$ almost surely in $\cC^{-\frac{3}{2}-3\kappa}(\TT^d)$ for $N\to \infty$.
    
Since $\Wick{u_N^2}$ and $u_N$ both converge to finite distributional limits almost surely by Lemma~\ref{f10:lem:convSimple}, 
the next three terms also converge to $0$ almost surely.
    
    Concerning the last element however, we know from Lemma \ref{f10:cN2_div_frac} that
    \[
        \left ( \log N\right )^{-\gamma} c_{N,2} \gtrsim \left ( \log N\right )^{-\gamma+1}\;.
    \]
Since the contribution of this term to the expression in the event $B^\gamma$ is given by
\begin{equ}
9\left ( \log N\right )^{-\gamma} c_{N,2} \scal{\psi_N,\psi}\;,
\end{equ}
and since $\scal{\psi_N,\psi} \to \|\psi\|^2 > 0$, this diverges, whence we conclude that 
$u + \psi \not \in B^\gamma$ and therefore $\bigl(T_\psi^* \mu\bigr)(B^\gamma) = 0$, so that 
$T_\psi^* \mu$ and $\mu$ are mutually singular.

\paragraph{Proof of the lemmas} \label{f10:proof-lemma-section}

\begin{proof}[Proof of Lemma \ref{f10:Z_to_zero_sigma}]
    We use the embedding $W^{\beta, 2p} \hookrightarrow W^{\beta - \frac{d}{2p}, \infty} \hookrightarrow \cC^{\beta - \frac{d}{2p}}$, with $\beta = -\frac{d}{2}$.
    Using the definition of $W^{\beta,2p}$ norm and the equivalence of moments for Gaussian polynomials, one has
    \begin{align*}
        \EE \left [ \left \lVert (\log N)^{-\gamma} \Wick{\left ( \<1>_N \right )^J} \right \rVert^{2p}_{W^{-\frac{3}{2}, 2p}} \right ] \lesssim \int_{\TT^d} \EE \left [ (\log N)^{-2\gamma} \left | \left \langle \nabla \right \rangle^{-\frac{3}{2}} \Wick{\left ( \<1>_N \right )^3} \right |^{2} \right ]^p dx\;.
    \end{align*}
 Since one has
    \begin{align*}
        \EE \left [ (\log N)^{-2\gamma} \left | \left \langle \nabla \right \rangle^{-\frac{3}{2}} \Wick{\left ( \<1>_N \right )^3} \right |^{2} \right ]
        &\lesssim (\log N)^{-2\gamma} \sum_{|\omega_i| \le N} \left \langle \omega_1 + \dots + \omega_3 \right \rangle^{-3} \prod_{i=1}^3 \left \langle \omega_i \right \rangle^{-2}\\
        &\lesssim (\log N)^{-2\gamma} \sum_{r_1=0}^N \frac{r_1^{2}}{\left ( 1+r_1^2\right )^{\frac{5}{2}}} r_1^{2}
        \lesssim (\log N)^{-2\gamma + 1}\;,
    \end{align*}
the desired result follows from a standard Borel--Cantelli argument.
\end{proof}
Next, we prove Lemma~\ref{f10:cN2_div_frac}, which provides a lower bound on the parameter $\gamma$. This bound ensures that the event $A^\gamma$ (or $B^\gamma$) is distinguishable under the shifted measure when compared to the non-shifted one.
\begin{proof}[Proof of Lemma \ref{f10:cN2_div_frac}]
    Expanding the definition of $c_{N,2} := \EE \left [ \<2>_N \<20>_N \right ]$, we get
    \begin{equs}
	    c_{N,2} &= 2\sum_{\substack{\omega_1 + \omega_2 = \omega_3 \\ |\omega_i| \le N}} \int_{\RR} \hat{P}_{t-u}(\omega_3) \int_{\RR} \hat{P}_{t-u_1}(\omega_1)\hat{P}_{u-u_1}(-\omega_1)du_1\\
        \quad&\times\int_{\RR}\hat{P}_{t-u_2}(\omega_2)\hat{P}_{u-u_2}(-\omega_2)du_2du\\
        &\simeq \sum_{\substack{\omega_1 + \omega_2 = \omega_3 \\ |\omega_i| \le N}} \int_{\RR} e^{-|t-u|\langle \omega_3 \rangle^2} \frac{e^{-|t-u|\langle \omega_1 \rangle^2}}{\langle \omega_1 \rangle^{2}} \frac{e^{-|t-u|\langle \omega_2 \rangle^2}}{\langle \omega_2 \rangle^{2}}du\\
        & \gtrsim \sum_{\substack{|\omega_i| \le N}} \frac{1}{\langle \omega_1 \rangle^{2}}\frac{1}{\langle \omega_2 \rangle^{2}}\frac{1}{\langle \omega_1 \rangle^2+\langle \omega_2 \rangle^2+\langle \omega_1+\omega_2 \rangle^2}\\
        & \gtrsim \sum_{\substack{|\omega_i| \le N}} \frac{1}{1 + |\omega_1|^{2}}\frac{1}{1 + |\omega_2|^{2}}\frac{1}{1+|\omega_1|^{2} \vee |\omega_2|^{2}}\\
        & \gtrsim \sum_{\substack{|\omega_1| \le |\omega_2| \le N}} \frac{1}{1 + |\omega_1|^{2}}\frac{1}{1 + |\omega_2|^{4}}\;.
    \end{equs}
Bounding the sum by an integral, we finally conclude that this expression is bounded
from below by a multiple of
    \begin{equ}
       \int_0^N \frac{r^2}{1+r^2} \int_r^\infty \frac{s^2}{1+s^4}\,ds\,dr
       \gtrsim \int_0^N  \frac{r}{1+r^2}\,dr \simeq \log N\;,
      \end{equ}
     as claimed.
\end{proof}

{\PatchBibSection

\EndPatchBibSection}
\endgroup

\subsection{Question 2: Paul Nelson}
\textit{Author:} Paul D. Nelson\par\medskip
\begingroup
\theoremstyle{plain}
\newtheorem{filedtheorem}{Theorem}
\newtheorem{filedlemma}[filedtheorem]{Lemma}
\newtheorem{filedproposition}[filedtheorem]{Proposition}
\theoremstyle{definition}

\newtheorem{fileddefinition}[filedtheorem]{Definition}
\setcounter{equation}{0}
\subparagraph{Question}
Let \(F\) be a non-archimedean local field with ring of integers \(\mathfrak o\).  Let $N_r$ denote the subgroup of $\mathrm{GL}_{r}(F)$ consisting of upper-triangular unipotent elements.  Let \(\psi:F\to \mathbb C^\times\) be a nontrivial additive character of conductor \(\mathfrak o\), identified in the standard way with a generic character of $N_r$.  Let \(\Pi\) be a generic irreducible admissible representation of \(\mathrm{GL}_{n + 1}(F)\), realized in its \(\psi^{-1}\)-Whittaker model \(\mathcal W(\Pi,\psi^{-1})\).  Must there exist \(W\in \mathcal W(\Pi,\psi^{-1})\) with the following property?

Let $\pi$ be a generic irreducible admissible representation of \(\mathrm{GL}_{n}(F)\), realized in its $\psi$-Whittaker model \(\mathcal W(\pi,\psi)\).  Let $\mathfrak{q}$ denote the conductor ideal of $\pi$, let \(Q\in F^\times\) be a generator of \(\mathfrak q^{-1}\), and set
\[
  u_Q := I_{n+1} + Q\,E_{n,n+1} \in \mathrm{GL}_{n + 1}(F),
\]
where \(E_{i, j}\) is the matrix with a \(1\) in the \((i, j)\)-entry and \(0\) elsewhere.  For some \(V\in \mathcal W(\pi,\psi)\), the local Rankin--Selberg integral
\[
  \int_{N_n\backslash \mathrm{GL}_{n}(F)} W(\operatorname{diag}(g,1) u_Q)\,V(g)\,|\det g|^{s-\frac12}\,dg
\]
is finite and nonzero for all \(s\in\mathbb C\).

\subparagraph{Statement}

Let \(F\) be a non-archimedean local field with ring of integers \(\mathfrak o\). Let \(\psi:F\to \mathbb C^\times\) be a nontrivial additive character of conductor \(\mathfrak o\).  We write
\begin{equation*}
  G_r:=\operatorname{GL}_r(F),
\end{equation*}
and let \(N_r<G_r\) denote the subgroup of upper-triangular unipotent elements.  We embed \(G_n \hookrightarrow G_{n + 1}\) as the upper-left block.  We write \(E_{i j}\) for the matrix with a \(1\) in the \((i, j)\)-entry and \(0\) elsewhere.

A more precise form of the following ``lemma'' will appear in forthcoming joint work with Subhajit Jana.  It says informally that pure unipotent translates of fixed vectors in the Whittaker model of a representation of \(G_{n+1}\) may serve as test vectors for Rankin--Selberg integrals against all representations of \(G_n\) with a given conductor.

\begin{filedtheorem}\label{f04:thm:main}
  Let \(\Pi\) be a generic irreducible admissible representation of \(G_{n+1}\), realized in its \(\psi^{-1}\)-Whittaker model \(\mathcal W(\Pi,\psi^{-1})\). Then there exists \(W\in \mathcal W(\Pi,\psi^{-1})\) with the following property.  Let $\pi$ be a generic irreducible admissible representation of \(G_n\), realized in its $\psi$-Whittaker model \(\mathcal W(\pi,\psi)\).  Let $\mathfrak{q}$ denote the conductor ideal of $\pi$, let \(Q\in F^\times\) be a generator of \(\mathfrak q^{-1}\), and set
  \[
    u_Q := I_{n+1} + Q\,E_{n,n+1} \in G_{n + 1}.
  \]
  There exists \(V\in \mathcal W(\pi,\psi)\) so that the local Rankin--Selberg integral
  \[
    \int_{N_n\backslash G_n} W(\operatorname{diag}(g,1) u_Q)\,V(g)\,|\det g|^{s-\frac12}\,dg
  \]
is finite and nonzero for all \(s\in\mathbb C\).
\end{filedtheorem}

\subparagraph{Context}

Rankin--Selberg local zeta integrals arise as proportionality factors relating global Rankin--Selberg integrals and $L$-functions.  The above result provides test vectors, obtained via pure translates of fixed vectors, that work simultaneously for all representations of the smaller group having some given conductor.  Such results are sometimes useful in global applications because they relate problems concerning $L$-functions (subconvexity, moment asymptotics, ...) to problems concerning automorphic forms (quantitative equidistribution, ...).  The $n = 1$ case follows from standard properties of Gauss sums and stationary phase analysis in one variable; it has been applied in, e.g., \cite{f04:MR780071, f04:michel-2009}.  For general \(n\), \cite{f04:MR4053059} contains a similar result, but with an average over many unipotent translates rather than just one.

\subparagraph{Proof}\label{f04:sec:proof}
We first sketch the argument.  The basic idea is to apply the Godement--Jacquet functional to the Whittaker function on the smaller group.  This is readily seen to relate the unipotent-shifted Rankin--Selberg integral to an integral involving a translate of the standard congruence subgroup $K_1(\mathfrak{q}) \leq \mathrm{GL}_n(\mathfrak{o})$, consisting of matrices whose last row is congruent to \((0,\dotsc,0,1)\) modulo \(\mathfrak q\).  We then conclude via newvector theory.

Turning to details, we recall that $F$ is a non-archimedean local field, with ring of integers $\mathfrak{o}$.  We denote by $\mathfrak{p}$ the maximal ideal and $q$ the residue field cardinality.  We set $K_r := \mathrm{GL}_r(\mathfrak{o})$ and equip $G_r$ and $N_r$ with the Haar measures assigning volume one to $K_r$ and $N_r \cap K_r$, respectively.  As in the theorem statement, we write $\Pi$ (resp.\ $\pi$) for a generic irreducible representation of $G_{n + 1}$ (resp.\ $G_n$).

We continue to denote by $\mathfrak{q}$ the conductor ideal of $\pi$, defined to be the smallest ideal for which $\pi$ has a nonzero vector fixed by $K_1(\mathfrak{q})$.  We choose a generator $Q$ for $\mathfrak{q}^{-1}$, so that $\lvert Q \rvert =[\mathfrak{o} : \mathfrak{q}]$.  We recall (see \cite{f04:MR620708, f04:MR3138844}) that $\lvert Q \rvert$ (and hence $\mathfrak{q}$) may also be characterized in terms of the local $\varepsilon$-factor of $\pi$:
\begin{equation}\label{f04:eq:cuipy1wony}
  \varepsilon(\tfrac{1}{2} + s, \pi, \psi) = \lvert Q \rvert^{- s} \varepsilon(\tfrac{1}{2}, \pi, \psi).
\end{equation}

We recall the functional equation of Godement--Jacquet \cite[Theorem 3.3]{f04:MR0342495}.

\begin{filedlemma}\label{f04:lemma:cq4itap0v1}
  Let $f$ be a matrix coefficient of $\pi$, and let $\phi \in \mathcal{S}(M_n(F))$.  For $s \in \mathbb{C}$, the local zeta integral
  \begin{equation}\label{f04:eq:cq73uvikdb}
    Z(\phi, f, s) := \int_{G_n}
    \phi(g) f(g) \lvert \det g \rvert^{\frac{n - 1}{2} + s} \, d g,
  \end{equation}
  converges absolutely for $\Re(s)$ sufficiently large.  It extends to a meromorphic function on the complex plane for which the ratio
  \begin{equation*}
    \frac{Z(\phi, f, s)}{L(s, \pi)}
  \end{equation*}
  is holomorphic.  It satisfies the local functional equation
  \begin{equation}\label{f04:eq:cq4itiiq9o}
    \gamma(s, \pi, \psi) Z(\phi, f, s)
    = Z(\phi^\wedge, f^\vee, 1-s),
  \end{equation}
  where
  \begin{equation*}
    \gamma(s, \pi, \psi) = \varepsilon(s, \pi, \psi) \frac{L(1 - s, \tilde{\pi})}{L(s, \pi)},
  \end{equation*}
  with $\tilde{\pi}$ the contragredient of $\pi$, and where the Fourier transform is defined by
  \begin{equation*}
    f^\vee (g) := f(g^{-1}),
  \end{equation*}
  \begin{equation*}
    \phi^\wedge(x) := \int_{M_n(F)}
    \phi(y) \psi(\operatorname{trace}(x y)) \, d y,
  \end{equation*}
  with $M_n$ the space of $n \times n$ matrices and the Haar measure normalized to be self-dual with respect to $\psi$.  Moreover, both of the zeta integrals in \eqref{f04:eq:cq4itiiq9o} converge absolutely provided that, e.g., $\pi$ is unitary and generic and $\Re(s) = 1/2$.
\end{filedlemma}
We recall that a matrix coefficient of $\pi$ is a linear combination of functions of the form $f(g) = \ell(g v)$, where $v \in \pi$ and $\ell$ lies in the contragredient of $\pi$ (i.e., the admissible dual).  The conclusions of Lemma \ref{f04:lemma:cq4itap0v1} remain valid for more general coefficients of $\pi$.  For instance, suppose more generally that $f$ is of the same form, but with $\ell$ allowed to be any linear functional on $\pi$ (not necessarily in the admissible dual).  Given $\phi$ as above, we may choose a compact open subgroup $U$ of $G_n$ under which $\phi$ is bi-invariant.  The integrals in question do not change if we then replace $f$ by its two-sided average with respect to $U$, which has the effect of replacing $v$ by its average $v^U \in \pi^U$ and $\ell$ with its projection $\ell^U$ to the dual of $\pi^U$, extended by zero on the kernel of the averaging operator $\pi \rightarrow \pi^U$.  In particular, by specializing to the case that $\ell$ is a Whittaker functional on $\pi$, we see that such identities remain valid when $f$ is a Whittaker function for $\pi$.

We denote by $\mathcal{S}^{e}(F^\times)$ the space of all Schwartz--Bruhat functions $\beta \in \mathcal{S}(F^\times)$ such that $\beta(x y) = \beta(x)$ whenever $\lvert y \rvert = 1$, or equivalently, for which $\beta(x)$ depends only upon $\lvert x\rvert$.  We note that each $\beta \in \mathcal{S}^{e}(F^\times)$ satisfies the Mellin inversion formula
\begin{equation}\label{f04:eq:cq73uvecsz}
  \beta(y) = \int_{(\sigma)} \tilde{\beta}(s) \lvert y \rvert^s
  \, d s,
  \qquad \tilde{\beta}(s) := \int_{F^\times} \beta(y) \lvert y \rvert^{-s} \,d^\times y.
\end{equation}

For $\beta \in \mathcal{S}^e(F^\times)$, we define the transform $\beta^\sharp := \beta^{\sharp, \pi}$ of $\beta$ by
\begin{equation*}
    \beta^\sharp(y) :=
    \int_{(\sigma)}
    \frac{
      \tilde{\beta}(s)
      \lvert y \rvert^{- s}
      \, d s}{\gamma(\tfrac{1}{2} + s, \pi, \psi)},
  \end{equation*}
initially for $\sigma$ large enough.
\begin{filedlemma}\label{f04:example:cq4jqru7rh}
  Define $\beta$ via Mellin inversion \eqref{f04:eq:cq73uvecsz} by
  \begin{equation*}
    \tilde{\beta}(s) := \frac{\varepsilon(\tfrac{1}{2} + s, \pi, \psi)}{L(\tfrac{1}{2} + s, \pi)}.
  \end{equation*}
  Then:
  \begin{enumerate}
  \item $\beta$ is supported on $\{y : \lvert Q \rvert \leq \lvert y \rvert \leq \lvert Q \rvert q^n \}$ and takes the value $\varepsilon(\tfrac{1}{2}, \pi, \psi)$ on $\{y : \lvert y \rvert = \lvert Q \rvert\}$.
  \item $\beta^\sharp$ is supported on $\{y : 1 \leq \lvert y \rvert \leq q^n\}$ and takes the value $1$ on $\{y : \lvert y \rvert = 1\} = \mathfrak{o}^\times$.
  \end{enumerate}
\end{filedlemma}
\begin{proof}
  We appeal to the characterization \eqref{f04:eq:cuipy1wony} of $\lvert Q \rvert$.  We note first that $\beta^\sharp$ has Mellin transform
  \begin{equation*}
    \widetilde{\beta^\sharp} (s) = \frac{1}{L(\tfrac{1}{2} + s, \tilde{\pi})}.
  \end{equation*}
  Since the inverse $L$-values appearing above are monic polynomials in $q^{- s}$ of degree at most $n$, we see by Mellin inversion that $\beta$ and $\beta^\sharp$ have the claimed properties.
\end{proof}

\begin{filedlemma}\label{f04:lemma:cq4itbu91o}
  Assume that $\pi$ is unitary and generic.  We then have the identity of absolutely convergent integrals
  \begin{equation}\label{f04:eq:cq4jqjygcl}
    \int_{G_n}
    \phi(g)
    f(g)
    \beta(\det g)
    \lvert \det g \rvert^{\frac{n}{2}}
    \, d g
    =
    \int_{G_n}
    \phi^\wedge(g)
    f^\vee(g)
    \beta^\sharp(\det g) \lvert \det g \rvert^{\frac{n}{2}} \, d g.
  \end{equation}
\end{filedlemma}
\begin{proof}
  Starting with the left hand side, we insert the Mellin expansion of $\beta$, with $\sigma = 0$.  The resulting double integral over $g$ and $s$ converges absolutely, so we may swap the order.  We recognize the result as the integral $\int_{(0)} \tilde{\beta}(s) Z(\phi, f, \tfrac{1}{2} + s) \, d s$ involving the Godement--Jacquet zeta integral \eqref{f04:eq:cq73uvikdb}.  We now apply the local functional equation and expand the result as
  \begin{equation*}
    \int_{(0)}
    \frac{\tilde{\beta}(s)}{\gamma(\tfrac{1}{2} + s, \pi, \psi)}
    \left(
      \int_{G_n}
      \phi^\wedge(g) f^\vee(g) \lvert \det g \rvert^{\frac{n}{2} - s}
      \, d g
    \right)
    \, d s.
  \end{equation*}
  This double integral again converges absolutely, so we may rearrange it to obtain the stated identity.
\end{proof}
For the same reasons as indicated following the statement of Lemma \ref{f04:lemma:cq4itap0v1}, such identities persist for more general coefficients than matrix coefficients, and in particular, when $f$ is a Whittaker function.

Recall that we embed $G_n \hookrightarrow G_{n + 1}$ as the upper-left block.  We set
\begin{equation}\label{f04:eq:cq4jqk9bz0}
  W_0(g):=\int_{N_n} 1_{K_n}(x g)\,\psi(x)\,d x,
\end{equation}
which defines a Whittaker function on $G_n$ and extends, by the theory of the Kirillov model \cite{f04:MR748505}, to an element of $\mathcal{W}(\Pi,\psi^{-1})$ on $G_{n+1}$.

For $x \in F$ and $y \in F^\times$, we set
\begin{equation*}
  d_y:=\operatorname{diag}(1,\dotsc,1,y)\in G_n \hookrightarrow G_{n + 1},
  \qquad
  u_x:=I_{n+1}+ x E_{n,n+1} \in N_{n + 1}.
\end{equation*}
We then define
\begin{equation*}
  t_Q:= d_Q^{-1} u_Q = u_1 d_Q^{-1}.
\end{equation*}
\begin{filedlemma}\label{f04:proposition:cq4jqgh7yi}
  There exist $\beta \in \mathcal{S}^e(F^\times)$ and $\phi \in \mathcal{S}(M_n(F))$ so that for all $g \in G_n$, we have
  \begin{equation}\label{f04:eq:cq4jqjeuyf}
    \int_{N_n} \beta(\det x g) \phi(x g) \psi(x) \, d x = \varepsilon(\tfrac{1}{2}, \pi, \psi) W_0(g t_{Q})
  \end{equation}
  and
  \begin{equation}\label{f04:eq:cq4jqjy05j}
    \beta^\sharp(\det g) \phi^\wedge(g) = \lvert Q \rvert^n 1_{K_1(\mathfrak{q})}(g).
  \end{equation}
\end{filedlemma}

\begin{proof}
  We set
  \begin{equation*}
    \phi_0 := 1_{M_n(\mathfrak{o})},
  \end{equation*}
  \begin{equation}\label{f04:eq:cq4jqv74tm}
    \phi(x) := \psi(- x_{nn}) \phi_0(x d_{Q}^{-1}).
  \end{equation}
  and take $\beta$ as in Lemma \ref{f04:example:cq4jqru7rh}, so that in particular,
  \begin{equation}\label{f04:eq:cq73v9ep9j}
    \beta|_{Q \mathfrak{o}} = \varepsilon(\tfrac{1}{2}, \pi, \psi) 1_{Q \mathfrak{o}^\times}
  \end{equation}
  and
  \begin{equation}\label{f04:eq:cq73v91fxa}
    \beta^\sharp |_{\mathfrak{o}} = 1_{\mathfrak{o}^\times}.
  \end{equation}
  We must verify the relations \eqref{f04:eq:cq4jqjeuyf} and \eqref{f04:eq:cq4jqjy05j}.

  We start with \eqref{f04:eq:cq4jqjeuyf}.  Recall from \eqref{f04:eq:cq4jqk9bz0} that $W_0$ is the $\psi^{-1}$-Whittaker function $W_0(g) = \int_{N_n} 1_{K_n}(x g) \psi(x) \, d x$.  In particular,
  \begin{equation}\label{f04:eq:cq4jqk3rcx}
    W_0(g t_{Q}) = W_0(g u_1 d_{Q}^{-1}) = \psi(-g_{n n}) W_0(g d_{Q}^{-1}).
  \end{equation}
  Using this identity, we may rewrite the desired relation \eqref{f04:eq:cq4jqjeuyf} as
  \begin{equation}\label{f04:eq:cq6s1gjyh4}
    \int_{N_n} \beta(\det (x g)) \phi(x g) \psi(x) \, d x = \varepsilon(\tfrac{1}{2}, \pi, \psi) \psi(- g_{n n}) W_0(g d_{Q}^{-1}).
  \end{equation}
  We verify this as follows.  First, we see from the definition \eqref{f04:eq:cq4jqv74tm} and the identity $(xg)_{nn} = g_{nn}$ that for $x \in N_n$ and $g \in G_n$, we have
  \begin{equation}\label{f04:eq:relation-phi-phi0}
    \phi(x g) = \psi(- g_{n n}) \phi_0(x g d_Q^{-1}).
  \end{equation}
  Next, we have
  \begin{align*}
    \beta(\det g) \phi_0(g d_Q^{-1}) &= \varepsilon(\tfrac{1}{2}, \pi, \psi)
                                1_{Q \mathfrak{o}^\times}(\det g)
                                \phi_0(g d_Q^{-1}) \\
                              &=
                                \varepsilon(\tfrac{1}{2}, \pi, \psi)
                                1_{K_n}(g d_Q^{-1}).
  \end{align*}
  (In the first step, we use that $\phi_0(g d_Q^{-1})$ is nonzero only if $\det(g) \in Q \mathfrak{o}$ and apply \eqref{f04:eq:cq73v9ep9j}.  In the second step, we use that $1_{K_n}(g) = 1_{\mathfrak{o}^\times}(\det g) \phi_0(g)$ and $\det(d_Q) = Q$, which gives $1_{Q \mathfrak{o}^\times}(\det g) \phi_0(g d_Q^{-1}) = 1_{K_n}(g d_Q^{-1})$.)  Combining the above identities, we obtain
  \begin{equation*}
    \beta(\det(x g)) \phi(x g) = \varepsilon(\tfrac{1}{2}, \pi, \psi) \psi(- g_{n n}) 1_{K_n}(x g d_Q^{-1}).
  \end{equation*}
  Integrating both sides against $\psi(x) \, d x$ gives \eqref{f04:eq:cq6s1gjyh4}, as required.

  We verify \eqref{f04:eq:cq4jqjy05j} as follows (here $E_{ij}$ denotes the elementary matrix):
  \begin{align*}
    \beta^\sharp(\det g) \phi^\wedge(g)
    &=
      1_{\mathfrak{o}^\times}(\det g) \phi^\wedge(g)\\
    &=
      1_{\mathfrak{o}^\times}(\det g)
      \lvert Q \rvert^n \phi_0^\wedge(d_{Q}(g - E_{nn})) \\
    &=
      \lvert Q \rvert^n
      1_{\mathfrak{o}^\times}(\det g)
      1_{M_n(\mathfrak{o})}(d_{Q}(g - E_{nn})) \\
    &=
      \lvert Q \rvert^n
      1_{K_1(\mathfrak{q})}(g).
  \end{align*}
  Here, for the first step, we observed that $\phi^\wedge(x)$ is nonzero only if $x \in E_{nn} + d_Q^{-1} M_n(\mathfrak{o}) \subseteq M_n(\mathfrak{o})$, so that, in particular, $\det x \in \mathfrak{o}$; we then applied \eqref{f04:eq:cq73v91fxa}.  For the second step, we applied the general Fourier analytic calculation
  \begin{equation}\label{f04:eq:cq4jqkb0q8}
    \phi^\wedge(x) = \lvert Q \rvert^n \phi_0^\wedge(d_{Q}(x - E_{nn})).
  \end{equation}
  For the third, we applied the Fourier self-duality $\phi_0^\wedge = \phi_0 = 1_{M_n(\mathfrak{o})}$.  For the final step, we use that $K_1(\mathfrak{q})$ consists of all $x \in M_n(F)$ for which $d_Q (x - E_{nn}) \in M_n(\mathfrak{o})$ and $\det x \in \mathfrak{o}^\times$.
\end{proof}

For $W \in \mathcal{W}(\Pi, \psi^{-1})$, $V \in \mathcal{W}(\pi, \psi)$, and $s \in \mathbb{C}$, we define the Rankin--Selberg integral
\begin{equation}\label{f04:eq:cq4fl8vokh}
  \ell_{\mathrm{RS}}(s, W, V) := \int_{N_n \backslash G_n} W (\operatorname{diag}(g, 1)) \, V(g) \, \lvert \det g \rvert^{s - \frac{1}{2}} \, d g.
\end{equation}
The following result verifies Theorem \ref{f04:thm:main} in a more precise form.
\begin{filedproposition}\label{f04:proposition:cq4jqgikw1}
  Let $W_0 \in \mathcal{W}(\Pi, \psi^{-1})$ be such that for all $g \in G_n$, we have
  \begin{equation*}
    W_0(g) = \int_{N_n} 1_{K_n}(x g) \psi(x) \, d x.
  \end{equation*}
  Let $V \in \mathcal{W}(\pi, \psi)$ denote the normalized newvector (i.e., the unique $K_1(\mathfrak{q})$-invariant vector for which $V(1) = 1$, see \cite{f04:MR620708, f04:MR3138844}).  Then for all $s \in \mathbb{C}$, we have
  \begin{equation}\label{f04:eq:cq4jqi2rfd}
    \ell_{\mathrm{RS}}(s, u_Q W_0, d_Q V) = c \lvert Q \rvert^{- \frac{n}{2}},
  \end{equation}
  where
  \begin{equation}\label{f04:eq:cq4jqlhqy9}
    c := \varepsilon(\tfrac{1}{2}, \pi, \psi)^{-1} \lvert Q \rvert^n \operatorname{vol}(K_1(\mathfrak{q})) \asymp 1.
  \end{equation}
\end{filedproposition}
\begin{proof}
  We note first that, by a change of variables, we have the homogeneity property
  \begin{equation}\label{f04:eq:cq4jqpwhnv}
    \ell_{\mathrm{RS}}(s, u_{Q} W_0, d_{Q}V) = \lvert Q \rvert^{- \left( s - \frac{1}{2} \right)} \ell_{\mathrm{RS}}(s, t_{Q} W_0, V).
  \end{equation}
  In view of this, the desired identity \eqref{f04:eq:cq4jqi2rfd} is equivalent to
  \begin{equation}\label{f04:eq:cq4jqjdmzj}
    \ell_{\mathrm{RS}}(s, t_{Q} W_0, V) = c \lvert Q \rvert^{s - \frac{n+1}{2}}.
  \end{equation}
  Next, since $W_0$ is supported on $\det^{-1}(\mathfrak{o}^\times)$, we see that the translate $t_{Q} W_0$ is supported on $\det^{-1}(Q \mathfrak{o}^\times)$, so the left hand side of \eqref{f04:eq:cq4jqjdmzj} is a constant multiple of $\lvert Q \rvert^s$.  For this reason, it suffices to verify \eqref{f04:eq:cq4jqjdmzj} for (say) $s = \tfrac{n+1}{2}$, where our task is to check that $\ell_{\mathrm{RS}}(\tfrac{n+1}{2}, t_{Q} W_0, V) = c$.
  Inserting definitions and unfolding, we obtain, with $f(g) := V(g)$,
  \begin{align*}
    \varepsilon(\tfrac{1}{2}, \pi, \psi) \ell_{\mathrm{RS}}(\tfrac{n+1}{2}, t_{Q} W_0, V)
    &\overset{\eqref{f04:eq:cq4fl8vokh}}{=}
      \varepsilon(\tfrac{1}{2}, \pi, \psi)
      \int_{N_n \backslash G_n} W_0(g t_Q) V (g) \lvert \det (g) \rvert^{\frac{n}{2}} \, d g
    \\
    &\overset{\eqref{f04:eq:cq4jqjeuyf}}{=}
      \int_{G_n}
      \phi(g)
      f(g)
      \beta(\det g)
      \lvert \det g \rvert^{n/2}
      \, d g \\
    &\overset{\eqref{f04:eq:cq4jqjygcl}}{=}
      \int_{G_n}
      \phi^\wedge(g)
      f^\vee(g)
      \beta^\sharp(\det g) \lvert \det g \rvert^{n/2} \, d g \\
    &\overset{\eqref{f04:eq:cq4jqjy05j}}{=}
      \lvert Q \rvert^n \int_{K_1(\mathfrak{q})} V(g^{-1}) \lvert \det g \rvert^{n/2} \, d g \\
    &= \lvert Q \rvert^n \operatorname{vol}(K_1(\mathfrak{q})),
  \end{align*}
  where in the final step, we use the $K_1(\mathfrak{q})$-invariance of $V$, the normalization $V(1) = 1$, and the fact that $\lvert \det g \rvert = 1$ on $K_1(\mathfrak{q})$.  Thus \eqref{f04:eq:cq4jqjdmzj} holds.
\end{proof}


\def\cprime{$'$} \def\cprime{$'$} \def\cprime{$'$} \def\cprime{$'$}

{\PatchBibSection

\EndPatchBibSection}
\endgroup

\subsection{Question 3: Lauren Williams}
\textit{Authors:} Houcine Ben Dali; Lauren Kiyomi Williams\par\medskip
\textit{Title:} A probabilistic interpretation for interpolation Macdonald polynomials\par\medskip
\begingroup
\allowdisplaybreaks
\newtheorem{filecthm}{Theorem}[section]
\newtheorem{fileccor}[filecthm]{Corollary}
\newtheorem{fileclem}[filecthm]{Lemma}
\newtheorem{filecprop}[filecthm]{Proposition}
\newtheorem{filecconj}{Conjecture}
\newtheorem{filecdefprop}[filecthm]{Definition-Proposition}
\theoremstyle{definition}
\newtheorem{filecdefinition}[filecthm]{Definition}
\newtheorem{filecremark}[filecthm]{Remark}
\newtheorem{filecexample}[filecthm]{Example}
\newtheorem{filecnotation}[filecthm]{Notation}

\AtBeginEnvironment{appendices}{\crefalias{section}{appendix}}

\makeatletter
\renewcommand*\theHfilecthm{\thesection.\arabic{filecthm}}
\renewcommand*\theHfilecdefinition{\thesection.\arabic{filecdefinition}}
\renewcommand*\theHfilecremark{\thesection.\arabic{filecremark}}
\renewcommand*\theHfilecexample{\thesection.\arabic{filecexample}}
\renewcommand*\theHfilecprop{\thesection.\arabic{filecprop}}
\renewcommand*\theHfileclem{\thesection.\arabic{fileclem}}
\renewcommand*\theHfileccor{\thesection.\arabic{fileccor}}
\renewcommand*\theHfilecnotation{\thesection.\arabic{filecnotation}}
\makeatother

\providecommand{\centeredl}[1]{\parbox[][1.5cm][c]{3.5cm}{\centering #1}   }
\providecommand{\centered}[1]{\parbox[][1.5cm][c]{7cm}{\centering #1}   }

\providecommand{\bfp}{\boldsymbol{p}}
\providecommand{\bfx}{\boldsymbol{x}}
\providecommand{\ie}{\textit{i.e.}}

\providecommand{\MLQ}{\operatorname{MLQ}}
\providecommand{\SMLQ}{\operatorname{MLQ^{\pm}}}
\providecommand{\Tab}{\operatorname{Tab}}
\providecommand{\inv}{\operatorname{inv}}
\providecommand{\wt}{\operatorname{wt}}
\providecommand{\wtp}{\textup{wt}'}
\providecommand{\skipped}{\operatorname{skip}}
\providecommand{\free}{\operatorname{free}}
\providecommand{\emp}{\operatorname{empty}}
\providecommand{\hook}{\operatorname{hook}}
\providecommand{\classic}{\operatorname{cl}}
\providecommand{\primed}{\operatorname{pr}}
\providecommand{\sign}{\operatorname{sign}}
\providecommand{\inc}{\operatorname{inc}}
\providecommand{\odd}{\operatorname{odd}}
\providecommand{\even}{\operatorname{even}}
\providecommand{\SSYT}{\operatorname{SSYT}}

\providecommand{\negative}{\operatorname{neg}}
\providecommand{\coinv}{\operatorname{coinv}}
\providecommand{\maj}{\operatorname{maj}}
\providecommand{\arm}{\operatorname{arm}}
\providecommand{\armp}{\textup{arm}'}
\providecommand{\leg}{\operatorname{leg}}
\providecommand{\ball}{\operatorname{ball}}
\providecommand{\pair}{\operatorname{pair}}
\providecommand{\rev}{\operatorname{rev}}
\providecommand{\Span}{\operatorname{Span}}
\providecommand{\Pack}{\operatorname{Pack}}
\providecommand{\Supp}{\operatorname{Supp}}
\providecommand{\tH}{\widetilde{H}^{(q,t)}}
\providecommand{\xx}{\mathbf x}
\providecommand{\J}{{J}^{(\alpha)}}
\providecommand{\Jsh}{{J}^{*}}
\providecommand{\JJack}{{\mathfrak J}^{(\alpha,\gamma)}}
\providecommand{\normJ}{{j}^{(q,t)}}
\providecommand{\C}{\mathcal C}
\providecommand{\mcP}{\mathcal P}
\providecommand{\Ni}{\mathcal N^{(i)}}
\providecommand{\E}{\mathcal E}
\providecommand{\ZZ}{\mathbb Z}
\providecommand{\QQ}{\mathbb Q}
\providecommand{\QT}{\mathcal{Q}}
\providecommand{\VV}{V_{\lambda}^*}
\providecommand{\G}{{Q^{\pm}}}
\providecommand{\mcG}{\mathcal{G}}
\providecommand{\Gbar}{\widebar{\mathcal{G}}}
\providecommand{\Qbar}{\widebar{Q}}
\providecommand{\DD}{D} 
\providecommand{\NN}{\mathbb N}
\providecommand{\YY}{\mathbb Y}
\providecommand{\OO}{\mathbb{O}}
\providecommand{\On}{S_n^{\pm}}
\providecommand{\PP}{\mathbb P}
\providecommand{\MM}{T}
\providecommand{\tPP}{\widetilde{\mathbb P}}
\providecommand{\PPone}[1]{\mathbb P^{(1)}_{#1}}
\providecommand{\PPtwo}[1]{\mathbb P^{(2)}_{#1}}
\providecommand{\mfp}{\mathfrak p}
\providecommand{\mfq}{\mathfrak q}
\providecommand{\mfr}{\mathfrak r}
\providecommand{\hatfstar}[1]{\widehat{F^*_{#1}}}
\providecommand{\Estar}[2]{{E^{*#1}_{#2}}}
\providecommand{\tb}{c}

\providecommand{\fstar}[2]{{F^{*#1}_{#2}}}
\providecommand{\WS}[1]{{\widetilde{#1}}}
\definecolor{green}{RGB}{43,92,47}
\definecolor{blue}{RGB}{40,68,104}
\definecolor{red}{RGB}{254, 113, 96}
\definecolor{purple}{RGB}{102,0,51}
\definecolor{gray}{RGB}{224,224,224}
\definecolor{lightpurple}{RGB}{255, 249, 242}
\definecolor{blue}{RGB}{40,68,104}
\hypersetup{colorlinks = true, urlcolor=blue}
\let\underbrace\LaTeXunderbrace

\def\LW#1{(\textcolor{teal}{LW:#1})}
\def\HBD#1{(\textcolor{violet}{HBD:#1})}

\setcounter{equation}{0}
The following problem and solution have since appeared as part of 
\cite{f03:BenDaliWilliams2026}.
\subparagraph{The problem}

Let $\lambda=(\lambda_1 > \dots > \lambda_n \geq 0)$ be a partition with distinct parts.  Assume moreover that 
$\lambda$ is \emph{restricted}, in the sense that it has a unique part of size $0$ and no part of size $1$.
Does there exist a nontrivial Markov chain on $S_n(\lambda)$  whose stationary distribution is given by
$$\frac{F^*_{\mu}(x_1,\dots,x_n; q=1,t)}{P^*_{\lambda}(x_1,\dots,x_n;
q=1,t)} \text{ for }\mu\in S_n(\lambda)$$
where $F^*_{\mu}(x_1,\dots,x_n; q,t)$ and
$P^*_{\lambda}(x_1,\dots,x_n;q,t)$ are the
interpolation ASEP polynomial and interpolation Macdonald polynomial,
respectively?  If so, prove that the Markov chain you construct has the
desired stationary distribution.  By ``nontrivial'' we mean that the
transition probabilities of the Markov chain should not be described
using the polynomials $F_{\mu}^*(x_1,\dots,x_n; q,t)$.

\subparagraph{The solution}

The answer to the question is yes, as we 
explain below.  
 For $1\leq k\leq n$, we define
\begin{equation}\label{f03:eq:1minuspk}
\mfp_k:=\frac{t^{-n+1}(1-t)}{x_k-t^{-n+2}}
\in \QQ(t,x_1,\dots,x_n) \quad \text{and}\quad \mfq_k:= \frac{(1-t)x_k}{x_k-t^{-n+2}}\in \QQ(t,x_1,\dots,x_n).
\end{equation}

If  $0<t<1$ and $x_i>t^{-n+1}$ for $1\leq i\leq n$, then $\mfp_k$ and $\mfq_k$ are probabilities. 

\begin{filecdefinition}
\label{f03:def:intpushTASEP}
Fix a partition $\lambda=(\lambda_1 \geq \dots \geq \lambda_n)$ with $\lambda_n=0$.
The \emph{interpolation $t$-Push TASEP} with 
\emph{content $\lambda$} is a Markov chain
on $S_n(\lambda)$; we think of its states as configurations of particles on a ring labeled by $\lambda_1,\dots, \lambda_n$, where state $\eta$ corresponds to having  a particle labeled $\eta_j$ at position $j$.  Moreover, there is a \emph{bell} attached to each particle.
The transitions from $\eta\in S_n(\lambda)$ are as follows.
\begin{enumerate}[label=(Step \arabic*)]
   \item[(Step 0)]\label{f03:Step0} The bell at position $j$ rings with probability 
   $$P_j  =\frac{\prod_{k<j}\left(x_k-\frac{1}{t^{n-2}}\right)\prod_{k>j}\left(x_k-\frac{1}{t^{n-1}}\right)}{e^*_{n-1}(\bfx;t)}, $$
   where 
   $e^*_{n-1}(\bfx;t)=\sum_{j=1}^n\prod_{k<j}\left(x_k-\frac{1}{t^{n-2}}\right)\prod_{k>j}\left(x_k-\frac{1}{t^{n-1}}\right)$.
      \item\label{f03:Step1} 
      The particle at position $j$, say with label $a$, is activated, and starts traveling clockwise according to the rules of the \emph{$t$-Push TASEP}.  That is, suppose  there are $m$ ``weaker'' particles in the system, i.e. particles whose labels are less than $a$,  including vacancies (label $0$).
   Then with probability $\frac{t^{k-1}}{[m]_t}$ the activated particle will move to the location of the $k$th of these weaker particles.  If this location contains a particle with positive label, then that particle becomes active, and chooses a weaker particle to displace in the same way.  The procedure continues until the active particle arrives at a vacancy.  

   At the end of this step, position $j$ is vacant, and 
   we regard this vacancy as a particle labeled $a:=0$.

   \item\label{f03:Step2}
   The particle labeled $a:=0$ 
   now goes to position $1$ and starts traveling clockwise.

   When it gets to site $k$ for $1\leq k\leq j-1$ containing a particle with label $b \geq 0$, it 
   skips over that site with probability

       $1-\mfp_k$ 
       if $b\geq a$,
       and $1-\mfq_k$ if $b<a$;

      otherwise it settles at that site, activating/ displacing the site's particle.

Once it activates a new particle, 

the old particle settles at site $k$ and the new active particle continues to travel clockwise towards position $j$, activating a new particle according to the rule above. 
   The active particle stops once it displaces/activates another particle or arrives at position $j$, in which case it settles in position $j$.  
\end{enumerate}
We denote the resulting configuration by
$\nu$ and the transition probability by 
$\PP(\eta,\nu)$.

Moreover, we let $\PPone{\lambda,j}=\PPone{j}$ and $\PPtwo{\lambda,j}=\PPtwo{j}$ denote the transition probabilities associated with \ref{f03:Step1} and \ref{f03:Step2}, respectively. We then have, for $\mu,\nu\in S_n(\lambda)$, 
$$\PP(\mu,\nu)=\sum_{1\leq j\leq n}P_j\sum_{\rho\in S_n(\lambda):\rho_j=0}\PPone{j}(\mu,\rho)\PPtwo{j}(\rho,\nu).$$
\end{filecdefinition}

\begin{filecthm}\label{f03:thm:main}
    In the interpolation $t$-Push TASEP with content
    $\lambda=(\lambda_1,\dots,\lambda_n)$ and parameters
    $\bfx=(x_1,\dots,x_n)$ and $t$, the stationary probability of $\mu\in S_n(\lambda)$ is given by 
    $$\pi^*_{\lambda}(\mu) = \frac{F^*_{\mu}(\bfx; 1, t)}{P^*_{\lambda}(\bfx; 1, t)}.$$
\end{filecthm}

\subparagraph{The proof} 

Recall the notion of classical two-line queues 
from \cite{f03:CorteelMandelshtamWilliams2022} and signed two-line queues from \cite{f03:BenDaliWilliams2025} together with their weight functions.  (Here we specialize  $q=1$.)

Let  $\mathcal{Q}_\kappa^\eta$ denote the set of 
classical two-line queues with top row $\eta=(\eta_1,\dots,\eta_n)$ and bottom row $\kappa=(\kappa_1,\dots,\kappa_n)$, and  let
$a^\eta_{\kappa}$  
denote  the weight generating function of $\mathcal{Q}_{\kappa}^{\eta}$.
\begin{equation}\label{f03:eq:generating_function0}
    a_\kappa^\eta=a_\kappa^\eta(t):=\sum_{Q\in\mathcal{Q}_\kappa^\eta}\wt_{\pair}(Q).
\end{equation}

Let $\mcG^\alpha_\mu$ denote the
set of signed two-line queues with 
top row $\alpha=(\alpha_1,\dots,\alpha_n)$ and 
bottom row $\mu=(\mu_1,\dots,\mu_n)$, and let
$b_\mu^\alpha$  denote the weight generating function of $\mcG^\alpha_\mu$. 
\begin{equation}\label{f03:eq:generating_function1}
    b_\mu^\alpha=b_\mu^\alpha(t):=\sum_{Q\in\mcG_\mu^\alpha}\wt_{\pair}(Q).
\end{equation}

Let $\wt(Q):=\wt_{\pair}(Q)\wt_{\ball}(Q)$ be the product of the pair weight and \emph{the ball weight}.

We obtain
\begin{equation}\label{f03:eq:generating_function2}
    \wt_\alpha b_\mu^\alpha=\sum_{Q\in\mcG_\mu^\alpha}\wt(Q),
    \ \text{ where } \ 
  \wt_\alpha:=\prod_{k:\,\alpha_k>0}x_k  \prod_{k:\, \alpha_k<0}\frac{-1}{t^{n-1}}.
\end{equation}

\begin{filecdefinition}\label{f03:def:Gbar}
Given a signed two-line queue $Q\in \mcG^\alpha_{\mu}$, we associate to it an \emph{unsigned version} $\Qbar$ obtained by forgetting the signs of the balls in the top row. The composition we read in the bottom row (respectively the top row) of $\Qbar$ is $\mu$ (respectively $\lVert \alpha\rVert)$, where 
$$\lVert \alpha\rVert=(|\alpha_1|,\dots,|\alpha_n|).$$
We then define $\Gbar_\mu^{\kappa}$ as the set of paired ball systems obtained by applying this operation on $Q\in \mcG^\alpha_\mu$, where $\alpha\in \ZZ^n$ satisfying $\lVert\alpha\rVert=\kappa$.
\end{filecdefinition}

This leads us to define the following weights. Fix $\Qbar\in \Gbar^{\kappa}_\mu$:
\begin{itemize}
    \item A nontrivial pairing $p$ in $\Qbar$ has the weight
    \begin{equation}\label{f03:eq:pairing_weight_unsigned}
  \wt(p)=(1-t)t^{\skipped(p)}.      
    \end{equation}

    \item Let $B$ be a ball labeled $a>0$ in column $k$ and such that the ball below is labeled $b$ (If $B$ has a vacancy below it, we take $b=0$.) We define the weight of $B$ by:
\begin{equation}\label{f03:eq:weights_unsigned_balls}
\wt(B):=
\begin{cases}
    x_k-\frac{1}{t^{n-1}}&\text{if $b=a$,}\\
    x_k&\text{if $b>a$,}\\
    \frac{1}{t^{n-1}}&\text{if $b<a$}.
\end{cases}    
\end{equation}    
The weight of $\Qbar$ is defined by
$$\wt(\Qbar):=\prod_{B \text{ in the top row}}\wt(B) \prod_{p \text{ nontrivial pairing}}\wt(p).$$
\end{itemize}
 We then have the following lemma.

 \begin{fileclem}\label{f03:lem:forget_signs}
 Fix a partition $\lambda$ with distinct parts and two compositions $\kappa,\mu\in S_n(\lambda)$. Let $\Qbar\in \Gbar^{\kappa}_\mu$. Then
 $$\wt(\Qbar)=\sum_{Q}\wt(Q),$$
where the sum is taken over all signed two-line queues $Q$ from which $\Qbar$ is obtained by forgetting signs.
 \end{fileclem}
\begin{proof}
    We consider all the possible ways of ``adding signs'' to the balls in the top row of $\Qbar$ to obtain a signed two-line queue. Fix such a ball $B$ labeled $a>0$:
    \begin{itemize}
        \item if $B$ has below it a vacancy or a ball labeled $b<a$, then 

        we must assign a $-$ sign to $B$. 
        \item if $B$ has a ball labeled $b>a$ below it, then 

        we must assign a $+$ sign to $B$.
        \item if $B$ has a ball labeled $b=a$ below it, then 

        we can give $B$ a $+$ or $-$ sign.
    \end{itemize}
We then check that the possible signs for each ball $B$ is consistent with the choice of weights in \cref{f03:eq:weights_unsigned_balls}. In particular, one notices that when a ball $B$ is given a $-$ sign, the ball weight should be multiplied by $-1$ when we go from $\Qbar$ to $Q$, but the weight of the pairing connected to $B$ is also multiplied by $-1$.

\end{proof}

Given $\kappa\in S_n(\nu)$, we define $\tb_\nu^{\kappa}$ by 
\begin{equation}\label{f03:eq:def_c}
  \tb_{\nu}^{\kappa}:=\sum_{\alpha:\, \lVert \alpha\rVert =\kappa}\wt_\alpha b^{\alpha}_{\nu}.  
\end{equation}

 We get the following corollary obtained by combining \cref{f03:eq:generating_function2} and \cref{f03:lem:forget_signs}.
\begin{fileclem}\label{f03:lem:c}
 Fix $\lambda$ a partition with distinct parts, and $\kappa,\mu\in S_n(\lambda)$. Then
 $$c^\kappa_\mu=\sum_{\Qbar\in \Gbar^{\kappa}_\mu} \wt(\Qbar).$$
\end{fileclem}
    Since $\lambda$ has distinct parts, $\Gbar^{\kappa}_\nu$ is either empty or contains exactly one element.

Fix a weakly order-preserving function $\phi:\NN\rightarrow\NN$.  Fix two partitions $\lambda$ and $\kappa$ such that $\phi(\lambda)=\kappa$.
For $\eta\in S_n(\kappa)$, define 
$$G^*_\eta(\bfx;t):=\sum_{\rho\in S_n(\lambda):\, \phi(\rho)=\eta} F^*_\rho(\bfx;1,t).$$ Let $G_{\eta}$ be the top homogeneous part of $G^*_\eta$.

The following is an analogue of \cite[Theorem 4.18]{f03:AyyerMartinWilliams2025}, and 
can be proved in essentially the same way, using interpolation analogues of results from \cite{f03:AlexanderssonSawhney2019}.

\begin{filecthm}\label{f03:thm:weak-reordering}
    Fix $\lambda$ and $\kappa$ as above. For all $\eta\in S_n(\kappa)$, we have at $q=1$ that 
    $$\frac{G^*_\eta(\bfx;t)}{P^*_\lambda(\bfx;1,t)}=\frac{F^*_\eta(\bfx;1,t)}{P^*_\kappa(\bfx;1,t)}.$$
\end{filecthm}

Given a composition $\rho$, let $\rho^-:=(\rho^-_1,\dots,\rho^-_n)$,
where
$\rho^-_i=\max(\rho_i-1,0)$.
\begin{fileccor}\label{f03:cor:eta_minus}
    Consider a composition $\rho$ with $\rho_i\neq 1$ for any $1\leq i\leq n$. Let $k$ be the number of non-zero parts of $\rho$. Set $\eta=\rho^-$.
    We then have at $q=1$,
$$F^*_{\rho}(\bfx;1,t)=F^*_{\eta}(\bfx;1,t)\cdot e^*_{k}(\bfx;t).$$
\end{fileccor}
\begin{proof}
Let $\lambda$ and $\kappa$ be the two partitions obtained by reordering $\rho$ and $\eta$, respectively.
  Consider the weakly order-preserving function $\phi:i\mapsto \max(i-1,0)$. We then have $\phi(\rho)=\eta$. Since $\lambda$ does not have parts of size 1, and $\phi$ is bijective from $\{0,2,3,\dots\}$ to $\{0,1,2,\dots\}$, then $\rho$ is the unique composition in $S_n(\lambda)$ such that $\phi(\rho)=\eta$ and we have $G^*_\eta=F^*_\rho$. It follows then from \cref{f03:thm:weak-reordering} that
  $$\frac{F^*_{\rho}(\bfx;1,t)}{P^*_{\lambda}(\bfx;1,t)}=\frac{F^*_{\eta}(\bfx;1,t)}{P^*_{\kappa}(\bfx;1,t)}.$$
We now recall that at $q=1$, we have from \cite{f03:Dolega2017, f03:BenDaliWilliams2025} that 

    \begin{equation}\label{f03:eq:factorization}
      P^*_\lambda(x_1,\dots,x_n;1,t)=\prod_{1\leq i\leq \lambda_1}P^*_{\lambda'_i}(x_1,\dots,x_n;1,t)=\prod_{1\leq i\leq \lambda_1}e^*_{\lambda'_i}(x_1,\dots,x_n;t),
    \end{equation}
    where $\lambda'$ is the partition conjugate to $\lambda$.  Using this plus the fact that $\kappa$ is obtained from $\lambda$ by removing the largest column (of size $k$), we get that
$$\frac{P^*_{\lambda}(\bfx;1,t)}{P^*_{\kappa}(\bfx;1,t)}=e^*_{k}(\bfx;t),$$
which implies that $F^*_{\rho}(\bfx;1,t)=F^*_{\eta}(\bfx;1,t)\cdot e^*_{k}(\bfx;t).$
\end{proof}

\begin{filecprop}\label{f03:prop:Ptwo}
     Fix $\rho,\nu\in S_n(\lambda)$, and let $j$ be the index such that $\rho_j=0$. We have 
     $$\PPtwo{j}(\rho,\nu)=\frac{c^\rho_\nu}{\prod_{k<j}\left(x_k-\frac{1}{t^{n-2}}\right)\prod_{k>j}\left(x_k-\frac{1}{t^{n-1}}\right)},$$
     or equivalently,
     $$ P_j\cdotp \PPtwo{j}(\rho,\nu)=\frac{c^\rho_\nu}{e_{n-1}^*},$$
     where $c^\rho_\nu$ is the coefficient from \cref{f03:eq:def_c}, 
     i.e. 
     the generating function for the set $\Gbar_{\nu}^\rho$.
\end{filecprop}

The  idea of the proof below is that a signed two-line queue encodes Step 2 of the 
interpolation $t$-Push TASEP.  

\begin{proof}
 Note that \ref{f03:Step2} of \cref{f03:def:intpushTASEP} is encoded by 
an element of a set $\Gbar^{\rho}_{\nu}$ (see \cref{f03:def:Gbar}).  

Indeed, the transition in \ref{f03:Step2} from the configuration $\rho$ to the configuration $\nu$ is possible if and only there is an
element $\Qbar$ in $\Gbar^{\rho}_\nu$ (recall that this set contains at most one element).
More precisely,  a particle labeled $a>0$ which moved from position $k\in \llbracket n\rrbracket$ to a position $k'$, corresponds to a non trivial pairing in $\Qbar$ connecting a ball labeled $a$ in column $k$ of the top row to a ball labeled $a$ in column $k'$ of the bottom row. Particles which do not move correspond to trivial pairings. 

We now claim that $\wt(\Qbar)$ divided by 
$D:=\prod_{k<j}\left(x_k-\frac{1}{t^{n-2}}\right)\prod_{k>j}\left(x_k-\frac{1}{t^{n-1}}\right)$ gives $\PPtwo{j}(\rho,\nu)$.  We will prove the claim below by showing that each 
ball or pairing weight in $\wt(\Qbar)$,
divided by one of the factors in $D$, equals one of the skipping/ displacement
probabilities from \cref{f03:Step2}

(whose product is $\PPtwo{j}(\rho,\nu)$).  Note that in what follows,
instead of associating the weight 
$(1-t)t^{\skipped(p)}$ to each nontrivial pairing,
we will associate $(1-t)$ to the top ball in each nontrivial pairing,
and a factor of $t$ to each skipped ball.
\begin{itemize}
    \item Each ball in column $k>j$ of $\Qbar$ is necessarily trivially paired, since no ball in position $k>j$ get skipped or displaced in \ref{f03:Step2}.
      In $\Qbar$ this ball gets  weight $x_k-\frac{1}{t^{n-1}}$; when we divide this weight by the $k$th factor of $D$, we get $1$, which corresponds to the fact that balls in position $k>j$ do not contribute to $\PPtwo{j}(\rho,\nu)$.
        \item A ball in $\Qbar$ labeled $b$ in column $k<j$ which is trivially paired, and which is not skipped by a ball $a>b$, also has weight $x_k-\frac{1}{t^{n-1}}$. When we divide this weight by the $k$th factor of $D$, we get $1-\mfp_k$ (see \eqref{f03:eq:1minuspk}).  This is what we desired, because such a trivial pairing in $\Qbar$ corresponds to a particle labeled $b$ which is skipped over by a particle with a smaller label, and hence contributes $1-\mfp_k$
        to $\PPtwo{j}(\rho,\nu)$.
    \item A ball in $\Qbar$ labeled $b$ in column $k<j$ which is trivially paired, and which is skipped by a ball $a>b$, gets a weight $t(x_k-\frac{1}{t^{n-1}})$.  When we divide this weight by the 
    $k$th factor of $D$, we get $1-\mfq_k$
    (see \eqref{f03:eq:1minuspk}).
    This is what we desired, because such a trivial pairing corresponds to a particle labeled $b$ skipped over by a particle with a larger label, and hence contributes $1-\mfq_k$ to $\PPtwo{j}(\rho,\nu)$.
    \item A ball labeled $b$ in the top row of $\Qbar$ in column $k<j$ which has a ball labeled $a<b$ below it gets a weight $(1-t)\frac{1}{t^{n-1}}$
       (the factor $(1-t)$ is the nontrivial pairing weight). When we divide this weight by the $k$th factor of $D$, we get $\mfp_k$.  This is what we desired, because this pairing corresponds to a particle labeled $b$ being displaced by a particle with a smaller label, and hence contributing $\mfp_k$ 
    to $\PPtwo{j}(\rho,\nu)$.
    \item A ball labeled $b$ in the top row of 
    $\Qbar$ in column $k<j$ which has a ball labeled $a>b$ below it gets a weight $(1-t)x_k$ (the factor $(1-t)$ is the nontrivial pairing weight). 
     When we divide this weight by the $k$th factor of $D$, we get $\mfq_k$.  This is what we desired, because this pairing
    corresponds to a particle labeled $b$ being displaced by a particle with a larger label, and hence contributing $\mfq_k$ to $\PPtwo{j}(\rho,\nu)$.\qedhere
\end{itemize}
\end{proof}

\begin{filecprop}\label{f03:prop:intTASEP_SMLQ}
If $\lambda$ is restricted, and $\mu,\nu\in S_n(\lambda)$, then    
\begin{align*}
    \PP(\mu,\nu)=\sum_{\rho\in S_n(\lambda)}  \frac{a^\mu_\rho c^\rho_{\nu}}{e^*_{n-1}}.
\end{align*}
\end{filecprop}
\begin{proof}
    Combining 
    \cite[Lemma 5.4]{f03:AyyerMartinWilliams2025}
     and \cref{f03:prop:Ptwo}, we get
    \begin{align*}
    \PP(\mu,\nu)
    &=\sum_{1\leq j\leq n}P_j\sum_{\rho\in S_n(\lambda):\rho_j=0}\PPone{j}(\mu,\rho)\PPtwo{j}(\rho,\nu)\\
    &=\sum_{1\leq j\leq n}\sum_{\rho\in S_n(\lambda):\rho_j=0}  \frac{a^\mu_\rho c^\rho_{\nu}}{e^*_{n-1}}\\
    &=\sum_{\rho\in S_n(\lambda)}  \frac{a^\mu_\rho c^\rho_{\nu}}{e^*_{n-1}}.\qedhere
\end{align*}
\end{proof}

\begin{proof}[Proof of \cref{f03:thm:main}]
Fix a restricted partition $\lambda$.

Let $\nu\in S_n(\lambda)$.
From 
\cite[Theorem 1.15 and Lemma 5.6]{f03:BenDaliWilliams2025}, 
we have
 \begin{equation*}
   F^*_\nu(\bfx;1,t)=\sum_{\eta\in\NN^n}{F^{*\eta}_\nu}(\bfx;t) F^*_{\eta^{-}}(\bfx;1,t), 
 \end{equation*}
 where 
 $$F^{*\eta}_{\nu}(\bfx;t) :=\sum_{\alpha\in\ZZ^n}b_\nu^\alpha \wt_\alpha a^\eta_{\lVert \alpha\rVert}=\sum_{\kappa\in \NN^n}a_\kappa^\eta c^\kappa_\nu .$$
But we know from 

\cref{f03:cor:eta_minus} that
$$F^*_{\eta^-}(\bfx;1,t)=\frac{F^*_{\eta}(\bfx;1,t)}{e_{n-1}^*(\bfx;t)},$$
we use here the fact that $\eta$ has a unique part of size 0.

Hence
\begin{equation*}
   F^*_\nu(\bfx;1,t)=\sum_{\eta\in\NN^n}F^*_{\eta}(\bfx;1,t)\sum_{\kappa\in \NN^n}\frac{a_\kappa^\eta c^\kappa_\nu}{e_{n-1}^*(\bfx;t)},
 \end{equation*}
which can be rewritten using the transition probabilities of the interpolation $t$-Push TASEP (\cref{f03:prop:intTASEP_SMLQ}) we get
\begin{equation*}
   F^*_\nu(\bfx;1,t)=\sum_{\eta\in\NN^n}F^*_{\eta}(\bfx;1,t)\PP(\eta,\nu).
 \end{equation*}
This proves that $F^*_\mu(\bfx;1,t)$ are proportional to the stationary distribution of the interpolation $t$-Push TASEP $\pi^*_\lambda(\mu)$.
Finally, we use 
the fact that 
$P^*_\lambda=\sum_{\mu\in S_n(\lambda)}F^*_\mu$ to deduce that $\frac{F^*_\mu(\bfx;1,t)}{P^*_\lambda(\bfx;1,t)}=\pi^*_\lambda(\mu).$

\end{proof}

{\PatchBibSection
\noindent\emph{}\par
\EndPatchBibSection}
\endgroup

\subsection{Question 4: Nikhil Srivastava}
\textit{Authors:} Jorge Garza Vargas, Nikhil Srivastava, and Zack Stier\par\medskip
\textit{Title:} The finite free Stam inequality\par\medskip
\begingroup
\providecommand{\goe}{G^{N, \bR}}
\providecommand{\gse}{G^{N, \bH} }
\providecommand{\xgoe}{X_N^{\bR}}
\providecommand{\xgse}{X_N^{\bH}}
\providecommand{\pgoe}{\Phi^{\bR}}
\providecommand{\gue}{G^N}

\providecommand{\haaro}{O^{N, \bR}}
\providecommand{\haars}{O^{N, \bH}}

\providecommand{\yo}{Y_N^{\bR}}
\providecommand{\ys}{Y_N^{\bH}}

\providecommand{\infgoe}{\nu^{\mathbb{R}}}
\providecommand{\infgse}{\nu^{\mathbb{H}}}

\providecommand{\po}{\Psi_h^{\bR}}
\providecommand{\pse}{\Psi_h^{\bH}}
\providecommand{\numo}{f_h^{\bR}}
\providecommand{\deno}{g_q^{\bR}}

\providecommand{\bN}{\mathbb{N}}

\providecommand{\wgo}{\wg^O}

\providecommand{\trace}{\mathrm{Tr}}
\providecommand{\tr}{\textnormal{tr}}

\providecommand{\calA}{\mathcal{A}}
\providecommand{\calT}{\mathcal{T}}
\providecommand{\calB}{\mathcal{B}}
\providecommand{\calE}{\mathcal{E}}
\providecommand{\calX}{\mathcal{X}}
\providecommand{\calH}{\mathcal{H}}
\providecommand{\calK}{\mathcal{K}}
\providecommand{\calO}{\mathcal{O}}
\providecommand{\calD}{\mathcal{D}}
\providecommand{\calU}{\mathcal{U}}
\providecommand{\calP}{\mathcal{P}}
\providecommand{\B}{\mathcal{B}}
\providecommand{\calF}{\mathcal{F}}
\providecommand{\jcal}{\mathcal{J}}
\providecommand{\ecal}{\mathcal{E}}
\providecommand{\calM}{\mathcal{M}}
\providecommand{\calv}{\mathcal{V}}

\providecommand{\bZ}{\mathbb{Z}}
\providecommand{\bC}{\mathbb{C}}

\providecommand{\Span}{\mathrm{Span}}

\renewcommand{\P}{\mathbb{P}}
\providecommand{\E}{\mathbb{E}}

\providecommand{\supp}{\mathrm{supp}}

\providecommand{\diag}{\mathrm{diag}}

\providecommand{\Spec}{\mathrm{Spec}}
\providecommand{\spec}{\mathrm{spec}}

\providecommand{\dalpha}{\dot{\alpha}}
\providecommand{\dbeta}{\dot{\beta}}
\providecommand{\dgamma}{\dot{\gamma}}

\providecommand{\les}{\lesssim}

\providecommand{\bS}{\mathbb{S}}

\providecommand{\re}{\mathrm{Re}}
\providecommand{\im}{\mathrm{Im}}

\providecommand{\bw}{w}

\providecommand{\R}{\mathbb{R}}
\providecommand{\C}{\mathbb{C}}
\providecommand{\bR}{\mathbb{R}}

\providecommand{\tv}{\tilde{v}}

\providecommand{\num}[1]{\#(#1)}

\providecommand{\qplus}{\hat{q}}

\providecommand{\rplus}{\mathbb{R}_{>0}}

\providecommand{\op}{\mathrm{op}}

\providecommand{\labs}[1]{\left|#1\right|}

\providecommand{\dif}[1]{\left(\frac{d}{dx}\right)^{#1}}

\providecommand{\cov}{\mathrm{Cov}}

\providecommand{\difff}{\frac{d}{dx}}

\providecommand{\var}{\mathrm{Var}}

\providecommand{\dis}{\mathrm{Dis}}

\providecommand{\an}{A_N}
\providecommand{\bn}{B_N}

\providecommand{\bF}{\mathbb{F}}

\providecommand{\calS}{\mathcal{S}}

\providecommand{\sobrho}[2]{\|#1\|_{C^{#2}[-\rho, \rho]}}
\providecommand{\sobnorm}[2]{\|#1\|_{C^{#2}[-K, K]}}

\providecommand{\abs}[1]{|#1|}

\providecommand{\wg}{\mathrm{Wg}}

\providecommand{\bH}{\mathbb{H}}

\providecommand{\infnorm}[1]{\|#1\|_{C^0[-K, K]}}

\providecommand{\lpl}{K}

\providecommand{\gt}{\mathrm{GT}}

\providecommand{\leb}{\mathrm{Leb}}

\providecommand{\disc}{\mathrm{Disc}}

\providecommand{\vol}{\mathrm{Vol}}

\providecommand{\haar}{\mathrm{Haar}}

\allowdisplaybreaks

\providecommand{\calpha}{\overset{\circ}{\alpha}}
\providecommand{\ceta}{\overset{\circ}{\beta}}

\providecommand{\cu}{\overset{\circ}{u}}
\providecommand{\cw}{\overset{\circ}{w}}

\providecommand{\calV}{\mathcal{V}}
\providecommand{\calW}{\mathcal{W}}

\providecommand{\bone}{\mathbbm{1}}

\providecommand{\rs}{r^{(\sigma)}}

\providecommand{\gdown}{\gamma^{\downarrow}}

\providecommand{\palpha}{E[\alpha|\gamma]}
\providecommand{\pbeta}{E[\beta|\gamma]}

\providecommand{\jaca}{\frac{d\gamma}{d\alpha}}
\providecommand{\jacb}{\frac{d\gamma}{d\beta}}

\providecommand{\jacobian}{\jac}
\providecommand{\jac}{J_{\boxplus_n}}

\providecommand{\roots}{\Omega_{\boxplus_n}}
\providecommand{\rootsi}{\Omega_{\boxplus_n, i}}
\providecommand{\dermap}{ \Omega_{\partial_x}}
\providecommand{\dermapi}{\Omega_{\partial_x, i}}

\providecommand{\score}{\mathscr{J}_n}
\providecommand{\scoreder}{\mathscr{J}_{n-1}}

\providecommand{\jacder}{J_{\partial_x}}

\providecommand{\hessder}{H_{\partial_x}^{(i)}}
\providecommand{\hess}{\mathrm{Hess}}
\providecommand{\hessconv}{H_{\boxplus_n}^{(i)}}


\setcounter{equation}{0}
\numberwithin{equation}{section}

\newtheorem{fileqfourtheorem}{Theorem}[section]
\newtheorem{fileqfourdefinition}{Definition}[section]
\newtheorem{fileqfourconjecture}{Conjecture}[section]
\newtheorem{fileqfourlemma}{Lemma}[section]
\newtheorem{fileqfourcorollary}{Corollary}[section]
\newtheorem{fileqfourproposition}{Proposition}[section]
\newtheorem{fileqfourobservation}[fileqfourtheorem]{Observation}
\newtheorem{fileqfourremark}{Remark}[section]
\newtheorem{fileqfourproblem}[fileqfourtheorem]{Problem}
\newtheorem{fileqfourassumption}[fileqfourtheorem]{Assumption}
\newtheorem{fileqfourclaim}[fileqfourtheorem]{Claim}
\newtheorem{fileqfourfact}[fileqfourtheorem]{Fact}
\newtheorem{fileqfourexample}[fileqfourtheorem]{Example}
Let $\boxplus_n$ and $\Phi_n(\cdot)$ be defined as in the problem statement.
In this note we prove the following result, which was conjectured by D. Shlyakhtenko.
\begin{fileqfourtheorem}
\label{f02:thm:stamineq}
    Let $p(x)$ and $q(x)$ be any two monic real-rooted polynomials of degree $n$. Then 
    $$\frac{1}{\Phi_n(p\boxplus_n q)} \geq  \frac{1}{\Phi_n(p)} + \frac{1}{\Phi_n(q)}.$$
\end{fileqfourtheorem}

\subparagraph{Notation and preliminaries}
\paragraph{Polynomials and the finite free convolution}

Given a polynomial $p(x)$ of degree $n$ we say that $\alpha=(\alpha_1, \dots, \alpha_n)$ is a vector of roots for $p(x)$ if the $\alpha_i$ are the roots of $p(x)$. We will say that $\alpha$ is ordered if $\alpha_1\geq \cdots \geq \alpha_n$.  Recall that for monic polynomials $p(x)$ and $q(x)$, $p(x)\boxplus_n q(x)$ may be expressed as:
\begin{equation}\label{f02:eqn:convroots} p(x)\boxplus_n q(x) = \sum_{\pi\in S_n} \prod_{i=1}^n(x-\alpha_i-\beta_{\pi(i)}),\end{equation}
where $\alpha$ and $\beta$ are vectors of roots for $p(x)$ and $q(x)$, respectively, and $S_n$ is the symmetric group on $n$ elements (see Theorem 2.11 of \cite{f02:marcus2022finite} for a proof). Walsh \cite{f02:walsh1922location} proved that if $p(x)$ and $q(x)$ are real-rooted, then so is $p(x)\boxplus_n q(x)$. Therefore, the finite free convolution induces a map
$$\roots: \bR^{n}\times \bR^n \to \bR^n,$$
where if $\alpha$ and $\beta$ are vectors of roots for $p(x)$ and $q(x)$, then $\roots(\alpha, \beta)$ is defined to be the ordered vector of roots for $p(x)\boxplus_n q(x)$.  

Other than the fact that $\boxplus_n$ preserves real-rootedness, our proof will crucially exploit each of the following well-known properties of the finite free convolution. In what follows we will use $\bone_n$ to denote the all-ones vector of dimension $n$. 
We will use the notation $$m_k(\alpha):= \frac{1}{n}\sum_{i=1}^n \alpha_i^k \quad \text{and} \quad \var(\alpha) := m_2(\alpha)-m_1(\alpha)^2.$$

\begin{fileqfourproposition}[Properties of $\boxplus_n$]
\label{f02:thm:ffconvolution}
    If $\alpha, \beta\in \bR^n$ and  $\gamma =\roots( \alpha, \beta)$, then:
    \begin{enumerate}[label=\roman*), ref=\roman*]
        \item \label{f02:item:additivity} (Additivity) $m_1(\gamma)=m_1(\alpha)+m_1(\beta)$ and $\var(\gamma) = \var(\alpha)+\var(\beta)$. 
        \item \label{f02:item:translations} (Commutation with translation)   For all $t\in \bR$, $\roots(\alpha+t\bone_n ,\beta) = \gamma +t\bone_n$ and  $\roots(\alpha , \beta + t \bone_n) = \gamma+t\bone_n$.  
    \end{enumerate}
\end{fileqfourproposition}
\begin{proof}
	(i) Follows from the definition of $p\boxplus_n q$ in terms of the coefficients of $p$ and $q$ and the Newton identities. (ii) Follows from \eqref{f02:eqn:convroots}. 
\end{proof}

\paragraph{The heat flow and the finite free Fisher information} 
Given a vector of roots $\alpha\in \bR^n$ we will define the its finite free score vector $\score(\alpha) \in (\bR\cup\{\infty\})^n$ as 
$$\score(\alpha) := \left( \sum_{j : j \neq i} \frac{1}{\alpha_i-\alpha_j} \right)_{i=1}^n. $$
Given a real-rooted polynomial $p(x)$ with vector of roots $\alpha$, define its finite free Fisher information as 
$$\Phi_n(p) := \|\score(\alpha)\|^2.$$
The following fact will allow us to write the finite free Fisher information of the polynomial $p(x)$ in terms of the dynamics of its roots under the reverse heat flow. It was shown to us by D. Shlyakhtenko.

\begin{fileqfourlemma}[Score vectors as derivatives]
\label{f02:lem:score}
	Assume $p(x)$ has simple roots. Let $p_t(x):= \exp\left( -\frac{t}{2} \partial_x^2  \right)p(x)$ and let $\alpha(t) = (\alpha_1(t), \dots, \alpha_n(t))$ be the ordered vector of roots of $p_t(x)$. Then 
     $$\alpha_i'(0) = \sum_{j : j\neq i} \frac{1}{\alpha_i-\alpha_j},$$
    and in particular $\alpha'(0) = \score(\alpha)$. 
\end{fileqfourlemma}
\begin{proof}
	Since the $\alpha_i(t)$ are continuous in $t$, the roots remain simple in a neighborhood of $t=0$. Implicitly differentiating the expression
	$$p(\alpha_i(t))-tp''(\alpha_i(t))/2+t^2R(\alpha_i(t),t)=0$$
	(where $R(x,t)$ is a polynomial) at $t=0$ one obtains
	$$ \alpha_i'(0)=\frac{1}{2}\frac{p''(\alpha_i)}{p'(\alpha_i)},$$
	which is equal to the advertised expression.
\end{proof}

\subparagraph{Proof of Stam's inequality}
We now prove Theorem \ref{f02:thm:stamineq}.  The following Lemma allows us to restrict attention to the case when $p,q$, and $p\boxplus_n q$ all have simple roots.

\begin{fileqfourlemma}[Approximation by Simple Rooted Polynomials] Let $\epsilon>0$ and define the differential operator $T_\epsilon:=(1-\epsilon \cdot d/dx)^n$. If $p(x)$ is a monic real-rooted polynomial of degree $n$, then 
	\begin{enumerate}[label=\roman*), ref=\roman*]
	\item $(T_\epsilon p)(x)$ is monic and real-rooted of degree $n$ with simple roots. 
	\item $ \Phi_n(T_\epsilon p)\rightarrow \Phi_n(p)$ as $\epsilon\rightarrow 0$. 
	\item  $(T_\epsilon p)\boxplus_n (T_\epsilon q)= T_\epsilon^2 (p\boxplus_n q)$. 
	\end{enumerate}
\end{fileqfourlemma}
	\begin{proof} (i) was shown in \cite{f02:nuij1968note}. (ii) is because $\Phi_n$ is continuous in the roots of $p$, which are continuous in $\epsilon$. (iii) follows because $\boxplus_n$ commutes with differential operators (see e.g. \cite{f02:marcus2022finite}.)\end{proof}
		Thus, establishing Theorem \ref{f02:thm:stamineq} for the simple case implies the general case by using (iii) above and taking $\epsilon\rightarrow 0$. In what follows, $p(x)$ and $q(x)$ are monic real-rooted polynomials, $\alpha$ and $\beta$ are vectors of roots for $p(x)$ and $q(x)$, $\gamma:= \roots(\alpha, \beta)$, and $\alpha,\beta,\gamma$ all have distinct entries, implying that they are smooth functions of the coefficients of the corresponding polynomials. Let $\jac$ denote the Jacobian of $\roots$ at the point $(\alpha, \beta)$.

Our proof can be separated into three steps. The second step is  the most substantial one and we will defer its detailed discussion to Section \ref{f02:sec:understandingjacob}. 

\textbf{Step 1 (Jacobians and score vectors).} We first note that the following relation between score vectors holds.

\begin{fileqfourobservation}[Relating score vectors]
\label{f02:obs:scorevecs}
Using the above notation, for any $a, b \geq 0$ 
$$\jac (a\score(\alpha), b \score(\beta))  = (a+b)\score(\gamma). $$
\end{fileqfourobservation}

\begin{proof}
For every $t\geq 0$ let $p_t(x) = \exp(-\tfrac{t}{2} \partial_x^2) p(x)$, let $\alpha(t)$ be the ordered vector of roots of $p_t$, and define $q_t, r_t$ and $\beta(t), \gamma(t)$ in an analogous way. Since the finite free convolution commutes with any differential operator,  it follows that
      $$r_{(a+b)t} = p_{at} \boxplus_n q_{bt}. $$
    Hence  $\gamma((a+b)t) = \roots(\alpha_{at}, \beta_{bt})$ for every $t$. So, if we differentiate this  relation with respect to $t$, using the chain rule for the right-hand side,  we get
      $$(a+b)\gamma'(0) =  \jac \left(\begin{array}{c}
           a \cdot \alpha'(0)  \\ b \cdot\beta'(0)
      \end{array} \right).$$  
    A direct application of  Lemma \ref{f02:lem:score} concludes the proof. 
\end{proof}

\textbf{Step 2 (Understanding the Jacobian).} The substance of our proof lies in understanding  $\jac$. In particular, we will show the following. 

\begin{fileqfourproposition}
\label{f02:prop:singularvalues}
  If $u, v \in \bR^n$ are orthogonal to $\bone_n$ then
    $$\|\jac (u, v)\|^2 \leq \|u\|^2 + \|v\|^2.$$
\end{fileqfourproposition}

This proposition will be proven in Section \ref{f02:sec:understandingjacob}, for now we show how it is used.

\textbf{Step 3 (Proof of Theorem \ref{f02:thm:stamineq} \`a la Blachman).} With Observation \ref{f02:obs:scorevecs} and Proposition \ref{f02:prop:singularvalues} in hand we can conclude the proof using the same argument that Blachman used in \cite{f02:blachman2003convolution}. 

\begin{proof}[Proof of Theorem \ref{f02:thm:stamineq}]
    First note that 
    $$\sum_{i=1}^n \sum_{j : j\neq i} \frac{1}{\alpha_i-\alpha_j}=0,$$
   since each term in the sum appears once with a plus and once with a minus. Therefore $\score(\alpha)$ is orthogonal to $\bone_n$ and, arguing analogously,  $\score(\beta)$ is orthogonal to $\bone_n$. So,  Proposition \ref{f02:prop:singularvalues} implies 
   $$\|\jac(a\score(\alpha), b \score(\beta))\|^2 \leq a^2\|\score(\alpha)\|^2+b^2 \|\score(\beta)\|^2.$$
   Combining this with Observation \ref{f02:obs:scorevecs} yields 
   $$(a+b)^2 \|\score(\gamma)\|^2 \leq a^2\|\score(\alpha)\|^2+ b^2\|\score(\beta)\|^2. $$
   Now, by choosing $a= \frac{1}{\|\score(\alpha)\|^2}$ and $b = \frac{1}{\|\score(\beta)\|^2}$, the above inequality turns into 
   $$\left(\frac{1}{\|\score(\alpha)\|^2} + \frac{1}{\|\score(\beta)\|^2} \right)^2 \|\score(\gamma)\|^2 \leq \frac{1}{\|\score(\alpha)\|^2} + \frac{1}{\|\score(\beta)\|^2}, $$
   which after simple algebraic manipulations can be turned into the inequality claimed in Theorem \ref{f02:thm:stamineq}. 
\end{proof}

\paragraph{Understanding $\jac$}
\label{f02:sec:understandingjacob}
Let $(\Omega_{\boxplus_n, 1}, \dots, \Omega_{\boxplus_n, n})$ be the coordinate functions of $\roots$, that is $\gamma_i = \rootsi(\alpha, \beta)$. The starting point of our approach to proving Proposition \ref{f02:prop:singularvalues} is the  observation that the matrix $\jac\jac^*$ is related to the Hessians of the functions $\rootsi$. It will be helpful to introduce the notation
$$\hessconv:= \mathrm{Hess}_{\rootsi}.$$  For this discussion it will prove useful to define the $(2n-2)$-dimensional subspace 
$$\calV = \left\{ (u, v) \in \bR^{n}\times \bR^n : u^* \bone_n = v^*\bone_n=0 \right\}. $$
And, given $w\in \bR^{n}\times \bR^n$ and $f: \bR^{n}\times \bR^n\to \bR^n$ we will use $D_wf$ to denote the directional derivative of $f$ in the direction of $w$, that is $D_w = \sum_i w_i\partial_i$.  
\begin{fileqfourlemma}[The Hessian of $\roots$]
\label{f02:lem:Hessian}
Using the above notation
    \begin{equation}
    \label{f02:eq:Hessian}
    w^* \jac \jac^* w = w^* \Big(I_{n}\oplus I_n- \sum_{i=1}^n \gamma_i \hessconv \Big)w,\quad \forall w\in \calV. 
      \end{equation}
\end{fileqfourlemma}

\begin{proof}
Fix $w=(u, v)\in \calV$ and define 
$$\alpha(t):= \alpha+ tu, \quad \beta(t):=\beta+tv,  \quad \text{and}\quad  \gamma(t):=\roots(\alpha(t), \beta(t)),$$ and note that the variance additivity  from Proposition \ref{f02:thm:ffconvolution} \ref{f02:item:additivity}) implies that
$$m_2(\gamma(t))-m_1(\gamma(t))^2 = m_2(\alpha(t))+ m_2(\beta(t))-(m_1(\alpha(t))^2+m_1(\beta(t))^2). $$
Now, the fact that $(u,v)\in \calV$  implies that the means $m_1(\alpha(t))$ and $m_1(\beta(t))$ are a constant function of $t$ and therefore, again by Proposition \ref{f02:thm:ffconvolution} \ref{f02:item:additivity}), the mean $m_1(\gamma(t))$ is also a constant function of $t$. So, differentiating the above equation twice with respect to $t$ we get
\begin{equation}
\label{f02:eq:secderivatives}
 \left.\partial_t^2 m_2(\gamma(t))\right|_{t=0} = \left.\partial_t^2 \big(m_2(\alpha(t))+m_2(\beta(t))\big)\right|_{t=0} .
 \end{equation}
Now we inspect both sides of the above equation. First
\begin{align}
\nonumber \left. n\, \partial_t^2m_2(\gamma(t))\right|_{t=0} & = \sum_{i=1}^n D_w^2(\gamma_i^2) 
\\ \nonumber & = 2\sum_{i=1}^n \left( (D_w\gamma_i)^2+ \gamma_iD_w^2\gamma_i\right)
\\ \label{f02:eq:secondderofgamma} & = 2\left(w^* \jac \jac^* w+ \sum_{i=1}^n\gamma_i w^* \hessconv w \right). 
\end{align}
Second
\begin{align}
\nonumber n\,\partial_t^2(m_2(\alpha(t))+m_2(\beta(t))) & = \partial_t^2((\alpha+tu)^*(\alpha+tu)+(\beta+tv)^*(\beta+tv)) 
 \\ \nonumber & = 2(u^*u +v^*v) 
 \\ \label{f02:eq:secondderofalphanadbeta} & = 2w^* w.
 \end{align}
Finally, plugging (\ref{f02:eq:secondderofgamma}) and (\ref{f02:eq:secondderofalphanadbeta}) back into (\ref{f02:eq:secderivatives}) yields 
$$w^* \jac \jac^* w+ \sum_{i=1}^n\gamma_i w^* \hessconv w = w^* w,$$
which is equivalent to the advertised result. 
\end{proof}

We now apply a result of Bauschke et al. \cite[Corollary 3.3]{f02:bauschke2001hyperbolic}.

\begin{fileqfourtheorem}[Bauschke et al.]
\label{f02:thm:convexity}
    Let $f\in \bR[x_1, \dots, x_m]$ be a hyperbolic polynomial in the direction $w\in \bR^m$ and for every $a\in \bR^m$ let $\lambda_1(a)\geq  \dots\geq  \lambda_m(a)$ be the roots of $g_a(t):=f(a+tw)$.  Then, for every $k=1, \dots, m$, the function $\sigma_k(a) := \sum_{i=1}^k \lambda_i(a)$ is convex in $a$. 
\end{fileqfourtheorem}

In our context this implies the following. 

\begin{fileqfourcorollary}
\label{f02:cor:psdness}
   For any real numbers $c_1\geq \cdots \geq c_n $, the matrix 
   $\sum_{i=1}^n c_i \hessconv$ is PSD. 
\end{fileqfourcorollary}

\begin{proof}
    Define the multivariate polynomial 
    $$f(x, a_1,\dots, a_n, b_1, \dots,  b_n) := \sum_{\pi\in S_n} \prod_{i=1}^n (x-a_i-b_{\pi(i)}).$$
    Since the above polynomial is homogeneous and the finite free convolution preserves real rootedness, $f$ is hyperbolic in the direction $e_1=(1, 0\cdots, 0)$. Now, by Theorem \ref{f02:thm:convexity} the functions
    $$\sigma_k(x, a, b) = \sum_{i=1}^k \lambda_i(x, a, b)$$
    are convex, where $\lambda_1(x, a, b)\geq \cdots \geq \lambda_n(x, a, b)$ denote the roots of $f((x, a, b)+t e_1)$. And, because the $c_i$ are ordered we moreover have that the function 
    $$L(x, a, b):=\sum_{i=1}^n c_i \lambda_i(x,a, b)$$
    is convex, as it can be written as a positive linear combination of the $\sigma_k$. It follows that $\mathrm{Hess}_{L} = \sum_{i=1}^n c_i \mathrm{Hess}_{\lambda_i}$ at any $(x, a, b)$ is PSD. But, on the other hand, when $x=0$, $a=\alpha$ and $b = \beta$, we have that $\mathrm{Hess}_{\lambda_i} = \hessconv$, which in turn gives that $\sum_{i=1}^n c_i \hessconv$ is PSD. 
\end{proof}

We can now complete the proof of Proposition \ref{f02:prop:singularvalues}. 

\begin{proof}[Proof of Proposition \ref{f02:prop:singularvalues}]
  Let $(u, v)\in \calV$. Then 
  $$\|\jac(u, v)\|^2 = (u, v)^* \jac \jac^* (u, v) = \|u\|^2+\|v\|^2-\sum_{i=1}^n \gamma_i (u, v)^* \hessconv (u, v),$$
where the last equality follows from Lemma \ref{f02:lem:Hessian}. Now, applying Corollary \ref{f02:cor:psdness} with $c_i=\gamma_i$ gives that $\sum_{i=1}^n \gamma_i \hessconv$ is PSD, and hence 
$$\sum_{i=1}^n \gamma_i (u, v)^* \hessconv (u, v)\geq 0.$$
The proof follows from putting the two expressions together. 
\end{proof}

{\PatchBibSection

\EndPatchBibSection}
\endgroup

\subsection{Question 5: Andrew J. Blumberg}
\textit{Authors:} Andrew J.~Blumberg; Michael A.~Hill; Tyler Lawson\par\medskip
\textit{Title:} Generalized equivariant slice categories\par\medskip
\begingroup
\usetikzlibrary{positioning}
\usetikzlibrary{arrows}
\usetikzlibrary{arrows}
\usetikzlibrary{decorations.pathreplacing}
\usetikzlibrary{calc}

\providecommand{\mycases}[1]{\left\{\begin{array}{ll}#1\end{array}\right.}
\providecommand{\Z}{{\mathbb  Z}}
\providecommand{\N}{{\mathbb N}}
\providecommand{\R}{{\mathbb R}}
\providecommand{\F}{{\mathbb F}}
\providecommand{\MUR}{MU_{\R}}
\providecommand{\MSpin}{MSpin}
\providecommand{\Spin}{Spin}
\providecommand{\XiO}{\Xi O}
\providecommand{\smashover}[1]{\underset{#1}{\wedge}}
\providecommand{\timesover}[1]{\underset{#1}{\times}}
\providecommand{\Boxover}[1]{\underset{#1}{\Box}}
\providecommand{\otimesover}[1]{\underset{#1}{\otimes}}

\providecommand{\Hom}{\operatorname{Hom}}
\providecommand{\Der}{\operatorname{Der}}
\providecommand{\Ext}{\operatorname{Ext}}
\providecommand{\Tor}{\operatorname{Tor}}
\providecommand{\coker}{\operatorname{coker}}
\providecommand{\Map}{\operatorname{Map}}
\providecommand{\Sym}{\operatorname{Sym}}
\providecommand{\holim}{\operatorname{holim}}
\providecommand{\colim}{\operatorname{colim}}
\providecommand{\Stab}{\operatorname{Stab}}
\providecommand{\Spec}{\operatorname{Spec}}
\providecommand{\Tel}{\operatorname{Tel}}
\providecommand{\Sub}{\operatorname{Sub}}

\providecommand*{\RO}{RO}

\providecommand{\norm}{N}
\providecommand{\normR}{{}_R\norm}

\providecommand{\res}{res}
\providecommand{\tr}{tr}
\providecommand{\Tr}{tr}

\providecommand{\Cp}[1]{C_{p^{#1}}}
\providecommand{\Cpn}{\Cp{n}}
\providecommand{\cp}[1]{p^{#1}}

\providecommand{\m}[1]{{\protect\underline{#1}}}
\providecommand{\mZ}{\m{\Z}}
\providecommand{\mM}{\m{M}}
\providecommand{\mG}{\m{G}}
\providecommand{\mR}{\m{R}}
\providecommand{\mB}{\m{B}}
\providecommand{\mJ}{\m{J}}
\providecommand{\mH}{\m{H}}
\providecommand{\mN}{\m{N}}
\providecommand{\mpi}{\m{\pi}}
\providecommand{\mn}{\m{n}}
\providecommand{\mA}{\m{A}}
\providecommand{\mX}{\m{X}}
\providecommand{\mSpec}{\m{\Spec}}
\providecommand{\mY}{\m{Y}}
\providecommand{\mGamma}{\m{\Gamma}}
\providecommand{\mSet}{\m{\Set}}
\providecommand{\mcC}{\m{\cC}}
\providecommand{\mTr}{\m{\Tr}}
\providecommand{\hull}{\operatorname{hull}}

\providecommand{\cc}[1]{\mathcal #1}
\providecommand{\cC}{\cc{C}}
\providecommand{\ccD}{\cc{D}}
\providecommand{\cF}{\cc{F}}
\providecommand{\cO}{\cc{O}}
\providecommand{\ccZ}{\cc{Z}}
\providecommand{\cP}{\cc{P}}
\providecommand{\cA}{\cc{A}}

\providecommand{\aC}{\cC}
\providecommand{\aD}{\ccD}
\providecommand{\ob}{\textrm{ob}}
\providecommand{\id}{\textrm{id}}

\providecommand{\TC}{\textnormal{TC}}
\providecommand{\THH}{\textnormal{THH}}
\providecommand{\AQ}{\textnormal{AQ}}
\providecommand{\TAQ}{\textnormal{TAQ}}
\providecommand{\CoInd}{\textnormal{CoInd}}
\providecommand{\Ind}{\textnormal{Ind}}
\providecommand{\HH}{\textnormal{HH}}
\providecommand{\GL}{\textnormal{GL}}
\providecommand{\Sp}{\mathcal Sp}
\providecommand{\Emb}{\textnormal{Emb}}
\providecommand{\Set}{\mathcal Set}
\providecommand{\Com}{\mathcal Comm}
\providecommand{\Comm}{\mathcal Comm}
\providecommand{\Nuca}{\mathcal NUCA}
\providecommand{\Ninfty}{N_\infty}
\providecommand{\Coef}{\mathcal Coef}
\providecommand{\Ens}{\mathcal Ens}
\providecommand{\Top}{\mathcal Top}
\providecommand{\Ab}{\mathcal Ab}
\providecommand{\Alg}{\mathcal Alg}
\providecommand{\cOrb}{\mathcal Orb}
\providecommand{\Scheme}{\mathcal Scheme}
\providecommand{\Tamb}{\mathcal Tamb}
\providecommand{\OTamb}{\cO\mhyphen\Tamb}
\providecommand{\Mackey}{\mathcal Mackey}
\providecommand{\Cat}{\mathcal Cat}
\providecommand{\Symm}{\mathcal Sym}
\providecommand{\mRmod}{\m{R}\mhyphen\mathcal Mod}
\providecommand{\GComm}{G\mhyphen\Comm}
\providecommand{\Green}{\mathcal Green}
\providecommand{\mScheme}{\m{\Scheme}}
\providecommand{\Burn}{\mathcal Burn}
\providecommand{\Coeff}{\mathcal Coeff}
\providecommand{\mCoeff}{\m{\Coeff}}
\providecommand{\Mod}{\mathcal Mod}
\providecommand{\Mods}[1]{#1\mhyphen\Mod}
\providecommand{\RMod}{\Mods{R}}
\providecommand{\mRMod}{R\mhyphen\m{\Mod}}
\providecommand{\STamb}{\m{S}\mhyphen\Tamb}
\providecommand{\SAb}{\m{S}\mhyphen\Ab}
\providecommand{\SMackey}{\m{S}\mhyphen\Mackey}
\providecommand{\Aug}[1]{#1\mhyphen\mathcal Aug}
\providecommand{\OmegaOneO}{\Omega^{1,\mathcal O}_{\m{R}/\m{S}}}
\providecommand{\OmegaOneG}{\Omega^{1,G}_{\m{R}/\m{S}}}
\providecommand{\semidirect}{\ltimes}
\providecommand{\conn}{\operatorname{conn}}
\providecommand{\gconn}{\operatorname{gconn}}
\providecommand{\Fin}{\mathcal Fin}
\providecommand*{\IntOp}{\mathcal IntOp}
\providecommand*{\Tran}{\mathcal Tran}
\providecommand*{\SubPo}{\mathcal SubPo}
\providecommand{\claw}{\operatorname{claw}}
\providecommand{\Aut}{\operatorname{Aut}}
\providecommand{\Cata}{\operatorname{Cat}}
\providecommand{\MRC}{MRC}
\providecommand{\MRCcOMackey}{\MRC_{\cO}\mhyphen\Mackey}
\providecommand{\End}{\operatorname{End}}

\providecommand{\mgconn}{\operatorname{g\m{\conn}}}

\providecommand{\leftadjoint}{\dashv}
\providecommand{\To}{\Rightarrow}

\providecommand{\AN}[1]{{}_{#1}N}

\providecommand{\defn}[1]{{\underline{#1}}}

\providecommand{\mC}{\m{\cC}}

\mathchardef\mhyphen=45

\providecommand{\EM}{Eilenberg-Mac~Lane}

\numberwithin{equation}{section}

\newtheorem{fileftheorem}{Theorem}[section]
\newtheorem{fileflemma}[fileftheorem]{Lemma}
\newtheorem{filefcorollary}[fileftheorem]{Corollary}
\newtheorem{filefcor}[fileftheorem]{Corollary}
\newtheorem{filefproposition}[fileftheorem]{Proposition}
\newtheorem{filefconjecture}[fileftheorem]{Conjecture}
\newtheorem{filefquestion}[fileftheorem]{Question}

\newtheorem*{fileftheoremstar}{Theorem}

\theoremstyle{remark}
\newtheorem{filefremark}[fileftheorem]{Remark}
\newtheorem{filefexample}[fileftheorem]{Example}
\newtheorem{filefnotation}[fileftheorem]{Notation}
\newtheorem{filefobservation}[fileftheorem]{Observation}
\newtheorem{filefrant}[fileftheorem]{Rant}
\newtheorem{filefwarning}[fileftheorem]{Warning}

\theoremstyle{definition}
\newtheorem{filefassumption}[fileftheorem]{Assumption}
\newtheorem{filefaxiom}[fileftheorem]{Axiom}
\newtheorem{filefdefinition}[fileftheorem]{Definition}

\providecommand{\SESM}{Eilenberg--Mac~Lane}
\providecommand{\eqvr}{equivariant}
\providecommand{\defemph}[1]{\textbf{#1}} 
\providecommand{\Ho}{\textrm{Ho}}

\providecommand{\sma}{\otimes}
\providecommand{\sP}{\mathcal{P}}
\providecommand{\htp}{\simeq}
\providecommand{\aO}{\mathcal{O}}
\providecommand{\aU}{\mathcal{U}}
\providecommand{\bE}{\mathbb{E}}
\providecommand{\bO}{\mathbb{O}}
\providecommand{\bU}{\mathbb{U}}
\providecommand{\GTop}{\textrm{GTop}}
\providecommand{\bN}{\mathbb{N}}
\providecommand{\Ev}{\textrm{Ev}}

\setcounter{equation}{0}
\subparagraph{Indexed slice categories}\label{f06:sec: Indexed Slice
  Categories}

(Excerpt from ``Generalized equivariant slice categories'', with Mike
Hill and Tyler Lawson.)

\paragraph{Transfer and indexing systems}

We begin with an ahistorical but geodesic summary of transfer systems
and indexing systems.

\begin{filefdefinition}[{\cite{f06:BBR}, \cite{f06:Rubin21}}]
    A {\emph{transfer system}} on \(G\) is a partial order we will denote by \(\to\) on \(\Sub(G)\) satisfying three properties:
    \begin{enumerate}
        \item it refines subgroup inclusion: if \(H\to K\), then \(H\subseteq K\),
        \item it is conjugation invariant: if \(H\to K\) and \(g\in G\), then \(gHg^{-1}\to gKg^{-1}\), and
        \item it is closed under restriction: if \(H\to K\) and \(J\subseteq K\), then \(H\cap J\to J\).
    \end{enumerate}
\end{filefdefinition}

The collection of all transfer systems on \(G\) forms a poset under refinement, and we will use \(\leq\) for the partial order here.

\begin{filefdefinition}
    Let \(\cO\) be a transfer system on \(G\). A finite \(H\)-set 
    \[
        T=\coprod_{i} H/K_i
    \]
    is {\emph{admissible}} for \(\cO\) if for all \(i\), \(K_i\to H\). The collection of admissible \(H\)-sets for \(\cO\) will be denoted \(\cO(H)\). The collection of all \(\cO(H)\) as \(H\) varies gives an {\emph{indexing system}}.
\end{filefdefinition}

The admissible sets of $\cO$ are closely connected to the norms
structured by an \(\Ninfty\) operad; we will usually also abusively
denote the operad by \(\cO\). Here $i^H_* \colon \Sp^G \to \Sp^H$
denotes the pullback functor along the inclusion $H \to G$ and $N_H^G
\colon \Sp^H \to \Sp^G$ denotes the Hill-Hopkins-Ravenel
norm~\cite{f06:HHR}.

\begin{filefdefinition}
    For a finite \(G\)-set \(T\), we define the \(T\)-norm
    \[
        N^T\colon\Sp^G\to\Sp^G
    \]
    inductively by the formulas
    \begin{enumerate}
        \item \(N^{G/H}(E)=N_H^G i_H^\ast(E)\), and
        \item \(N^{T_0\amalg T_1}(E)=N^{T_0}(E)\otimes N^{T_1}(E)\).
    \end{enumerate}
\end{filefdefinition}

\paragraph{\texorpdfstring{\(\cO\)}{O}-slice filtration}

We now define the slice filtration relative to an indexing system
$\cO$.  We are going to use equivariant localization (more
specifically, nullification) to construct the relative slice towers.
Recall that in the equivariant context, we define local and acylic
objects in terms of conditions on the $G$-space of maps rather than
the non-equivariant space of maps.  The acylic objects form an
equivariant localizing subcategory.  Recall that given a set of
objects in $\Sp^G$, we define the equivariant localizing subcategory
generated by these objects to be the full subcategory of $\Sp^G$
constructed as the closure under homotopy colimits, retracts, and
tensors with orbit spectra.

\begin{filefdefinition}
    If \(\cO\) is an indexing system, then let \(\tau_{\geq n}^{\cO}\) be the equivariant localizing subcategory of $\Sp^G$ generated by 
    \[
    \Big\{ G_+\otimesover{H} N^T S^1\mid T\in \cO(H), |T|\geq n\Big\}.
    \]
    This is the category of \emph{\(\cO\)-slice \(n\)-connective spectra}.
\end{filefdefinition}

\begin{filefremark}
    Given a finite \(G\)-set \(T\), we have an equivariant homeomorphism
    \[
        N^T S^1\cong S^{\mathbb R\cdot T},
    \]
    the representation sphere associated to the permutation representation of \(T\). This means that the \(\cO\)-slice \(n\)-connective spectra can be equivalently viewed as being generated by the representation spheres associated to the permutation representations for admissible sets of cardinality at least \(n\). 
\end{filefremark}

Viewing this instead as a diagram of localizing subcategories (i.e., as a categorical Mackey functor), we are forming the equivariant localizing subcategory generated at \(G/H\) by \(N^T S^1\) for all admissible \(H\)-sets \(T\) of cardinality at least \(n\).

For the next definition, recall that the nullification at a set of objects $\{S_i\}$ in $\Sp^G$ is the left Bousfield localization at the set of terminal maps $\{S_i \to \ast\}$.

\begin{filefdefinition}
    If \(\cO\) is an indexing system, then:
    \begin{itemize}

    \item The \(n\)th \(\cO\)-slice truncation is the functor
    \[
        P^n_{\cO}\colon \Sp^G_{\geq 0}\to\Sp^G_{\geq 0}
    \]
    that is the nullification killing \(\tau_{\geq (n+1)}^{\cO}\). 

\item The \(n\)th \(\cO\)-slice cover is the functor
\[
P_n^{\cO}\colon \Sp^G_{\geq 0}\to\Sp^G_{\geq 0}
\]
defined to be the (homotopy) fiber of the natural map \(Id\Rightarrow P^{n-1}_{\cO}\). 
\end{itemize}
\end{filefdefinition}

The truncation functors are related in the evident fashion as $n$ varies.

\begin{filefproposition}
    For each \(n\geq 0\), we have a natural transformation 
    \[
    P^n_{\cO}(\mhyphen)\Rightarrow P^{n-1}_{\cO}(\mhyphen).
    \]
    These are compatible with the natural nullification functors 
    \[
    Id\Rightarrow P^n_{\cO}(\mhyphen).
    \]
    
    For a connective $G$-spectrum $E$, the natural map 
    \[
        E\to \lim_{\longleftarrow} P^n_{\cO}(E)
    \]
    is always a weak equivalence.
\end{filefproposition}
\begin{proof}
    The inclusion of categories \(\tau_{n+2}^{\cO}\subset\tau_{n+1}^{\cO}\) induces a natural transformation the other way of nullification functors.  Since we can factor the nullification functor \(P^n_{\cO}\) via this inclusion, the first two statements follow.
    
    For the second, we note that the Postnikov connectivity of \(G_+\otimesover{H} N^T S^1\) for a finite \(H\)-set $T$ is \(|T/H|\). As \(n\) goes to infinity, this also does (at worst as \(|T|/|H|)\). In particular, the map
    \[
        E\to P^n_\cO(E)
    \]
    has coconnectivity going to infinity.
\end{proof}

For any bounded below spectrum \(K\), the same argument shows that the natural map
\[
    K\sma E\to \lim_{\longleftarrow}\big(K\sma P^{n}E\big)
\]
is an equivalence. 

\begin{filefdefinition}
    A \(G\)-spectrum \(E\) is an \(\cO\)-\(n\)-slice if 
    \begin{enumerate}
        \item it is in \(\tau_{\geq n}^{\cO}\), and 
        \item the natural map
    \[
    E\to P^n_{\cO}E
    \]
    is an equivalence.
    \end{enumerate}
\end{filefdefinition}

\begin{filefproposition}
For any indexing system \(\cO\), the ordinary suspension yields maps
    \[
        \Sigma\colon\tau_{\geq k}^{\cO}\to\tau_{\geq (k+1)}^{\cO}.
    \]
\end{filefproposition}
\begin{proof}
    Since suspension commutes with homotopy colimits and induction, it suffices to show this on the generators \(N^T S^1\) as $T$ varies over the admissible sets of $\cO$.  Since $\Sigma N^{T} S^1 \htp N^{T\amalg \ast}S^1$, the result follows: if $T$ is admissible and of cardinality at least \(k\) then $T \amalg \ast$ is admissible and has cardinality at least $k+1$.
\end{proof}

\begin{filefcorollary}
    For any \(k\geq 0\), the \(\infty\)-category of \(\cO\)-\(k\)-slices is discrete.
\end{filefcorollary}
\begin{proof}
    If \(E,E'\) are \(\cO\)-\(k\)-slices, then they are both in \(\tau_{\geq k}^{\cO}\). By the usual adjunctions, for all \(n\geq 1\), the higher homotopy group \(\pi_n\) of the mapping space are given by
    \[
        \pi_n\Map(E,E')=[\Sigma^n E,E']^G=0,
    \]
    since the preceding proposition implies that \(\Sigma^n E\in\tau_{\geq (k+n)}^{\cO}\).
\end{proof}

\begin{filefdefinition}
We define {\emph{\(n\)\textsuperscript{th} \(\cO\)-slice}} of a connective \(G\)-spectrum \(E\), denoted \(P^n_{n,\cO}(E)\), to be the homotopy fiber of the natural map
\[
P^n_{\cO}(E)\to P^{n-1}_{\cO}(E).
\]
\end{filefdefinition}

\subparagraph{Characterizing slice towers via connectivity}\label{f06:sec: connectivity}

\paragraph{Geometric fixed points and slice connectivity}

We can detect slice connectivity in terms of the connectivity of the geometric fixed points~\cite{f06:HillYarnall, f06:WilsonSlice}.  To express this, it is convenient to define the following function capturing the structure of the indexing system.

\begin{filefdefinition}
    For any transfer system \(\cO\), we define the {\em characteristic function} of $\cO$ 
    \[
        \chi^{\cO}\colon\Sub(G)\to\Sub(G)
    \]
    by the formula
    \[
        \chi^{\cO}(H)=\min\{K\mid K\to H\}=\bigcap_{K\to H} K.
    \]
\end{filefdefinition}

\subparagraph{The geometric fixed points of \(\tau_{\geq n}^{\cO}\)}

Stable equivalences in $\Sp^G$ can be detected as maps that induce non-equivariant stable equivalences on passage to geometric fixed points for all (closed) subgroups of $G$.  It should thus be very plausible that the connectivity of geometric fixed points is a central notion.

\begin{filefdefinition}
    For a \(G\)-spectrum \(E\), let the {\em geometric connectivity}, denoted \(\mgconn(E)\), be the function from subgroups of \(G\) to \(\Z\cup\{\pm\infty\}\) defined by
    \[
        \mgconn(E)(H):= \conn\big(\phi^H(E)\big).
    \]
\end{filefdefinition}

\begin{fileflemma}\label{f06:lem: Slice Connectivity implies Geometric connectivity}
    Let \(\cO\) be a transfer system. If \(E\in\tau_{\geq n}^{\cO}\), then for all \(H\subset G\), 
    \[
        [H:\chi^{\cO}(H)]\cdot \mgconn(E)(H)\geq n.
    \]
\end{fileflemma}
\begin{proof}
    By restriction, it suffices to show this for \(H=G\). Since the geometric fixed points preserve homotopy colimits and extensions, it suffices to show this for generators.  Next, since geometric fixed points applied to an induced $G$-spectrum vanish, we are reduced to considering the case of \(N^T S^1\) for \(T\) an admissible \(G\)-set of cardinality at least \(n\). Decompose $T$ as
    \[
    T=\sum_{H} n_H G/H.
    \]
    The geometric fixed points of \(N^T S^1\) are \(S^{|T/G|}\), and in this case, we have
    \[
    |T/G|=\sum_H n_H.
    \]
    We have by assumption
    \[
    |T|=\sum_{H} n_H[G:H]\geq n,
    \]
    and by definition, \([G:\chi^{\cO}(H)]\) is the maximal element in 
    \[
    \big\{[G:H]\mid G/H\in\cO(\ast)\big\}
    \]
    (and in fact, all others divide it). This gives inequalities
    \[
    [G:\chi^{\cO}(G)]\cdot \sum_{H} n_H\geq \sum_{H} n_H [G:H]\geq n,
    \]
    as desired.
\end{proof}

\begin{filefremark}
    If \(\chi^\cO(G)=\{e\}\), then we recover \cite[Theorem 2.5]{f06:HillYarnall}.
\end{filefremark}

For the converse, we can again use isotropy separation, studying the cofiber sequence
\[
E\mathcal F_+\sma E\to E\to \tilde{E}\mathcal F\sma E.
\]  
The spectrum \(E\mathcal F_+\sma E\) is built out of pieces of the form \(G/H_+\sma E\), so this is in a localizing subcategory if and only if the restrictions are.

\begin{fileflemma}\label{f06:lem: Localizing Cats and Families}
    Let \(\mathcal F\) be a family, and let \(\tau\) be an equivariant localizing subcategory. If \(E\) is any \(G\)-spectrum such that for all \(H\in\mathcal F\), \(i_H^\ast E\in i_H^{\ast}\tau\), then
    \[
    (E\mathcal F_+\sma E)\in\tau
    \]
\end{fileflemma}
\begin{proof}
    This follows by the same proof as \cite[Lemma 2.4]{f06:HillYarnall}: the spectrum \(E\mathcal F_+\otimes E\) is in the localizing category generated by \(G/H_+\otimes E\) for \(H\in\mathcal F\). By assumption, we have an inclusion
    \[
        G/H_+\otimes E\cong G_+\otimesover{H} i_H^\ast E\in\tau.
    \]
\end{proof}

\subparagraph*{The \texorpdfstring{\(\cO\)}{O}-slices of geometric spectra}
Our argument will use downward induction on the subgroup lattice, so we will need to understand the \(\cO\)-slice connectivity of \(\tilde{E}\mathcal P\otimes E\), where \(\mathcal P\) is the family of proper subgroups of \(G\).  Recall that a \(G\)-spectrum \(E\) is called ``geometric'' if the natural map
\[
	E\to \tilde{E}\mathcal P\otimes E
\]
is an equivalence \cite[Definition 6.10]{f06:HillSlicePrimer}, and a Mackey functor \(\mM\) is geometric if \(H\mM\) is. The proof of \cite[Theorem 6.7]{f06:HillSlicePrimer} goes through essentially without change to show the following.

\begin{fileflemma}\label{f06:lem: Slice Connectivity of Geometric Things}
    Let \(\mM\) be a geometric Mackey functor. For any \(\cO\), 
    \[
        \Sigma^k H\mM
    \]
    is a \(k\cdot [G:\chi^{\cO}G]\)-\(\cO\)-slice.
\end{fileflemma}
\begin{proof}
    Since \(\mM\) is geometric, we have that for any finite \(G\)-set \(T\), the natural map
   \[
    S^{|T/G|}\hookrightarrow N^T S^1
   \]
   given by the inclusion of fixed points induces an equivalence
   \[
   S^{|T/G|}\otimes H\mM\to N^TS^1\otimes H\mM.
   \]
   We can bound the \(\cO\)-slice connectivity from below by choosing an \(\cO\)-admissible \(T\) with \(|T|\) as large as possible so that \(|T/G|=k\) is fixed. This is again achieved by taking
   \[
   T=k G/\chi^{\cO}(G),
   \]
   since \(\chi^{\cO}(G)\) is the minimal subgroup \(H\) such that \(H\to G\). This shows us that
   \[
   \Sigma^k H\mM\in\tau_{\geq k[G: \chi^{\cO}(G)]}^{\cO}.
   \]
   
   For the upper bound, consider an admissible \(G\)-set \(T\) such that 
   \[
   |T|>k[G:\chi^{\cO}(G)].
   \]
   Since \(k[G:\chi^{\cO}(G)]\) is the largest cardinality of an admissible \(G\)-set with \(k\)-orbits, we deduce that \(|T/G|>k\). Since \(\mM\) is geometric, we therefore deduce
   \[
   [N^TS^1,\Sigma^kH\mM]^G\cong [\Phi^G N^TS^1,\Sigma^k H\mM(G/G)]
   \cong [S^{|T/G|},\Sigma^{k}H\mM(G/G)]=0.
   \]
   This shows that \(H\mM\) is a \(k[G: \chi^{\cO}(G)]\)-slice.
\end{proof}

\paragraph{Rewriting \texorpdfstring{\(\cO\)}{O}-slice connectivity}
Putting these together, we get the full \(\cO\)-slice version of \cite[Theorem 2.5]{f06:HillYarnall}. 

\begin{fileftheorem}\label{f06:thm: GFP Reformulation}
    A \(G\)-spectrum \(E\) is in \(\tau_{\geq n}^{\cO}\) if and only if for all \(H\subset G\),
    \[
    [H:\chi^{\cO}(H)]\cdot\mgconn(E)(H)\geq n.
    \]
\end{fileftheorem}

\begin{proof}
    The proof is essentially that of \cite[Theorem 2.5]{f06:HillYarnall}. The forward direction is  Lemma~\ref{f06:lem: Slice Connectivity implies Geometric connectivity}. 
    
    For the other direction, let \(E\) be a spectrum with the prescribed geometric connectivities. Consider the isotropy separation sequence
    \[
        EP_+\otimes E\to E\to \tilde{E}\mathcal P\otimes E.
    \]
    By Lemma~\ref{f06:lem: Slice Connectivity of Geometric Things}, the \(\cO\)-slice connectivity of \(\tilde{E}\mathcal P\otimes E\) is at least \(n\). By induction on the subgroup lattice, Lemma~\ref{f06:lem: Localizing Cats and Families} shows that \(EP_+\otimes E\) also has \(\cO\)-slice connectivity \(n\). Since localizing categories are closed under extensions, this implies that \(E\) has \(\cO\)-slice connectivity \(n\).   
\end{proof}

Rewriting this slightly, we have a way to describe the slice connectivity of an arbitrary \(0\)-connective spectrum.

\begin{filefcorollary}\label{f06:cor: Determining Slice Connectivity}
    If \(E\in\Sp^{G}_{\geq 0}\), then let 
    \[
        n=\min_{H\subseteq G}\big\{[H:\chi^{\cO}(H)]\cdot\mgconn(E)(H)\big\}.
    \]
    Then \(E\in\tau_{\geq n}^{\cO}\).
\end{filefcorollary}

{\PatchBibSection

\EndPatchBibSection}
\endgroup

\subsection{Question 6: Daniel Spielman}
\textit{Author:} Dan Spielman\par\medskip
\textit{Title:} Light Sets of Vertices\par\medskip
\begingroup
\def\trace#1{\mathrm{Tr} \left(#1 \right)}
\def\calA{\mathcal{A}}
\def\calB{\mathcal{B}}
\def\Ltil{\widetilde{L}}
\def\union{\cup}
\def\intersect{\cap}
\def\floor#1{\left\lfloor #1 \right\rfloor}
\def\norm#1{\left\| #1 \right\|}
\def\ddelta{\boldsymbol{\delta}}
\def\pleq{\preccurlyeq}
\def\pgeq{\succcurlyeq}
\def\defeq{\stackrel{\mathrm{def}}{=}}
\def\setof#1{\left\{#1  \right\}}
\def\sizeof#1{\left|#1  \right|}

\def\trs#1{\mathrm{Tr}_{\sigma} \left(#1 \right)}
\def\Trs{\mathrm{Tr}_{\sigma}}

\newtheorem{filehtheorem}{Theorem}[section]
\newtheorem{filehcorollary}[filehtheorem]{Corollary}
\newtheorem{filehlemma}[filehtheorem]{Lemma}
\newtheorem{filehobservation}[filehtheorem]{Observation}
\newtheorem{filehproposition}[filehtheorem]{Proposition}
\newtheorem{filehdefinition}[filehtheorem]{Definition}
\newtheorem{filehclaim}[filehtheorem]{Claim}
\newtheorem{filehfact}[filehtheorem]{Fact}
\newtheorem{filehassumption}[filehtheorem]{Assumption}
\newtheorem{filehwarning}[filehtheorem]{Warning}
\newtheorem{filehconjecture}[filehtheorem]{Conjecture}


\renewcommand{\dim}[1]{\mathsf{dim}(#1)}
\providecommand{\codegen}[1]{\mathsf{codegen}(#1)}

\providecommand{\disc}[1]{\mathsf{disc}(#1)}
\providecommand{\adj}[1]{\mathsf{adj}(#1)}

\providecommand{\spn}[1]{\mathsf{span}(#1)}
\providecommand{\nul}[1]{\mathsf{null}(#1)}

\setcounter{equation}{0}

Throughout this note, $G = (V,E,w)$ will be a weighted graph with $n$ vertices. 
For an edge $(s,t) \in E$, we let $w(s,t)$ be its weight.
For two vertex sets, $S$ and $T$, the subgraph $G_{S,T}$ of $G$ has vertex set $V$,
  but only the edges going between vertices in $S$ and $T$.
We write $G_{S}$ for the graph that only contains the edges between vertices in $S$.

The matrix $L$ is the Laplacian of $G$, which we recall may be defined by
\[
  L = \sum_{(s,t) \in E} w(s,t) (\ddelta_s - \ddelta_t) (\ddelta_s - \ddelta_t)^T,
\]
where $\ddelta_s$ is the elementary unit vector with a 1 in position $s$.
We let $L_{S}$ denote the Laplacian of $G_{S}$.
As $G_{S}$ and $G$ have been defined to have the same vertex set,
  $L_{S}$ has the same dimension as $L$.

\begin{filehlemma}\label{f08:lem:light}
For every weighted graph $G = (V,E,w)$ with $n$ vertices, and for every $0 < \epsilon < 1$,
  there is an $S \subseteq V$ of size at least $\epsilon n / 42$
  so that
\[
  \epsilon L \pgeq  L_{S} .
\]
\end{filehlemma}

We call such a set of vertices $S$ an $\epsilon$-light set.
A set $S$ is $0$-light if and only if it is independent, and we could view lightness as a qualitative measure of independence.
We might have called it ``spectral independence,'' if that term were not already in use.

This lemma was proved by Daniel Spielman while working on the paper ``Sparsified Cholesky Solvers for SDD linear systems'', written with Richard Peng and Yin-Tat Lee~\cite{f08:LeePengSpielman}. 
We decided not to include the lemma in that paper because, while it could be used to obtain interesting variants of some results, it was not necessary for the main results in that paper.
That paper evolved into the paper ``Sparsified Cholesky and Multigrid Solvers for Connection Laplacians,'' written with Rasmus Kyng, Yin Tat Lee, Richard Peng and Sushant Sachdeva~\cite{f08:kyng2016sparsified}.

\subparagraph{Proof Strategy}
We define $L_{S,T}$ to be the Laplacian of $G_{S,T}$.
For a vertex $t$ and a subset of vertices $S$, we define $L_{S,t}$ to be the Laplacian of $G_{S , \setof{t}}$.

For a matrix $L$, we write its pseudo-inverse as $L^\dagger$. 
We write $L^{\dagger/2}$ for the square root of the pseudo-inverse.
We will prove the following statement that is equivalent to Lemma~\ref{f08:lem:light}
\[
  \norm{ L^{\dagger/2} L_{S} L^{\dagger/2}} \leq \epsilon .
\]

We will find it convenient to multiply all Laplacian matrices on the left and right by $L^{\dagger/2}$. 
So, we define 
\[
  \Ltil_S = L^{\dagger/2} L_{S} L^{\dagger/2},
  \quad 
  \Ltil_{S,T} = L^{\dagger/2} L_{S,T} L^{\dagger/2},
  \quad 
  \Ltil_{S,t} = L^{\dagger/2} L_{S,t} L^{\dagger/2},
\]
and recall that $L^{\dagger/2} L L^{\dagger/2} \defeq \Pi $ 
  is a symmetric projection matrix.

We are going to build up $S$ in a greedy fashion.
We will begin with a singleton set, and then add one vertex at a time.
As we add vertices to $S$, we will need to maintain bounds on two quantities: a modification of the upper barrier function from \cite{f08:BSS} and the sum of the leverage scores of edges between $S$ and $V \setminus S$. 

The leverage score of an edge $(s,t)$ is defined to be $w(s,t)$ times the effective resistance between $s$ and $t$:
\[
  \ell(s,t  ) = 
  w(s,t) (\ddelta_s - \ddelta_t)^T L^{\dagger} (\ddelta_s - \ddelta_t)
  = \trace{w(s,t) (\ddelta_s - \ddelta_t) (\ddelta_s - \ddelta_t)^T L^{\dagger} }
  = \trace{L_{\setof{s}, \setof{t}} L^{\dagger}}.
\]
For vertices $s$ and $t$ for which $(s,t)$ is not an edge, we define $\ell(s,t) = 0$.
For subsets of vertices $S$ and $T$, we define
\[
  \ell(S,T) \defeq \sum_{s \in S} \sum_{t \in T} \ell(s,t)
  = \sum_{s \in S} \sum_{t \in T : (s,t) \in E} \ell(s,t),
\]
and 
\[
  \ell(S) \defeq \ell(S,V - S).
\]

\begin{filehclaim}\label{f08:clm:ellSt}
For $S$ and $T$ subsets of vertices, $\ell(S, T) = \trace{\Ltil_{S,T}}$.
\end{filehclaim}
\begin{proof}
From the definition of the Laplacian of a graph, we have 
$
  L_{S,T} = \sum_{s \in S} \sum_{t \in T} L_{\setof{s}, \setof{t}}
$.
So, 
\begin{multline*}
  \trace{\Ltil_{S,T}} 
    = \trace{L^{\dagger/2} L_{S,T} L^{\dagger/2} } 
  = \trace{L_{S,T} L^\dagger} \\
= 
\sum_{s \in S} \sum_{t \in T} \trace{ L_{\setof{s}, \setof{t}} L^\dagger }
= 
\sum_{s \in S} \sum_{t \in T} \ell(s,t)
= \ell(S,T).
\end{multline*}
\end{proof}

We modify the BSS barrier function to make it better suited to matrices of rank at most $\sigma$ by only incorporating the largest $\sigma$ eigenvalues of the matrix. 
For a matrix $A$ with eigenvalues $\lambda_1 \geq \lambda_2 \geq \dotsb \geq \lambda_n$, and a $u > \lambda_1$, we define
\[
  \Phi^{u}_{\sigma} (A) \defeq \sum_{i=1}^{\sigma}\frac{1}{u - \lambda_{i}}.
\]
If $u \leq \lambda_1$, we define $\Phi^{u}_{\sigma} (A) = \infty$.
We overload the definition of $\Phi $ by setting
\[
  \Phi^{u}_{\sigma} (S) \defeq \Phi^{u}_{\sigma} (\Ltil_{S}).
\]
Our objective is to find a set $S$ of size $\sigma$ so that $\Phi^{\epsilon}_{\sigma} (S) < \infty$.

We deal with this barrier function by considering a modified trace of a matrix that only sums the largest $\sigma$ eigenvalues of its argument:
\[
  \trs{A} \defeq \sum_{i=1}^{\sigma}\lambda_{i},
\]
where the eigenvalues of $A$ are $\lambda_1 \geq \lambda_2 \geq \dotsb \geq \lambda_n$.
We then have $\Phi^{u}_{\sigma} (A) = \trs{\left(u I - A\right)^{-1}}$.
In all cases we consider, the argument of $\Trs$ is a diagonalizable matrix with real eigenvalues.

For the rest of this note, define 
\[
  \delta \defeq  \frac{21}{n},
\quad 
  \phi \defeq \frac{n}{21},
  \quad \mathrm{and} \quad 
  \sigma \defeq \floor{\epsilon n / 42}.
\]

We will prove Lemma \ref{f08:lem:light} by iteratively applying the following lemma. 

\begin{filehlemma}\label{f08:lem:oneStep}
If $\sizeof{S} \leq \sigma$, $\ell(S) \leq 4 \sizeof{S}$, and
  $\Phi^{u}_{\sigma} (S) \leq \phi$,
  then there is a $t \not \in S$ so that
\[
  \Phi^{u+\delta}_{\sigma} (S \union \setof{t}) \leq \phi 
  \quad \mathrm{and} \quad 
  \ell(S \union \setof{t}) \leq \ell(S) + 4.
\]
\end{filehlemma}
\begin{proof}
Lemma \ref{f08:lem:ell} says that for more than half the $t \not \in S$, $\ell(S \union \setof{t}) \leq \ell(S) + 4$.
And, under the conditions of the lemma, 
  Lemma~\ref{f08:lem:barrier} says that for at least half the $t \not \in S$, $\Phi^{u}_{\sigma} (S \union \setof{t}) \leq \phi$.
So, there is a $t \not \in S$ that satisfies both conditions.
\end{proof}

\begin{proof}[Proof of Lemma \ref{f08:lem:light}]
Set $u_{0} = \epsilon / 2$ and let $S_0 = \setof{v_0}$ an arbitrary $v_0 \in V$.
As $G_{S_0}$ has no edges, 
\[
  \Phi^{u_{0}}_{\sigma} (S_{0}) = \sigma / u_{0} \leq \frac{n}{21} = \phi .
\]
By applying Lemma~\ref{f08:lem:oneStep} $\sigma$ times, we inductively construct a set $S$
  of $\sigma + 1$ vertices so that $\ell(S) \leq 4 \sigma$ and
  $\Phi^{u_{0} + \sigma \delta}_{\sigma} (S) \leq  \phi$.
This implies that all of the eigenvalues of $\Ltil_{S}$ are at most
\[
 u_{0} + \sigma \delta = \frac{\epsilon}{2} + \sigma \frac{21}{n} 
 \leq \epsilon .
\]

\end{proof}

\subparagraph{Proofs}

\begin{filehlemma}\label{f08:lem:ell}
Let $S \subset V$. 
Then, for more than half the $t$ not in $S$,
\[
 \ell(S \union \setof{t}) \leq \ell(S) + 4.
\]
\end{filehlemma}
\begin{proof}
Recall $\ell(S \union \setof{t}) 
= \ell(S \union \setof{t}, V - (S \union \setof{t}))$.
For $t \not \in S$, we use the inequality 
\[
\ell(S \union \setof{t}, V - (S \union \setof{t}))
\leq 
\ell(S \union \setof{t}, V - S)
= \ell(S) + \ell(t, V - S).
\]
So, it suffices to show that for more than half the $t \not \in S$, 
$\ell(t, V - S) \leq 4$.  
This follows from the non-negativity of $\ell$ and Claim~\ref{f08:clm:ell} which shows that 
\[
\sum_{t \in V - S} \ell(t, V - S) < 2 \sizeof{V - S} .
\]
\end{proof}

\begin{filehclaim}\label{f08:clm:ell}
For every $T \subset V$, 
\[ \sum_{t \in T} \ell(t, T) \leq 2 (\sizeof{T} - 1). \]
\end{filehclaim}
\begin{proof}
\begin{align*}
\sum_{t \in T} \ell(t, T) 
= \sum_{t \in T} \trace{L_{\setof{t}, T} L^\dagger} 
= 2 \trace{L_{T} L^\dagger}.
\end{align*}
To show that $\trace{L_T L^\dagger} < \sizeof{T}$, observe that $L_T \pleq L$, so all the eigenvalues of $L_T L^\dagger$ are between 0 and 1.
Because $L_T$ has rank at most $\sizeof{T}-1$, at most $\sizeof{T}-1$ eigenvalues of $L_T L^\dagger$ are non-zero.

\end{proof}

For convenience, we now state a few key properties of the function $\Trs$ of a matrix.
We begin with its defect: it is not additive. 
But, Ky Fan's eigenvalue inequality (see Theorem 4.3.47a of \cite{f08:HornJohnsonMatrix}) tells us that it is subadditive:
\begin{equation}\label{f08:eqn:modtrace:subadditive}
\trs{A + B} \leq \trs{A} + \trs{B}.
\end{equation}

Most of the properties of $\Trs$ that we find helpful follow from the fact that, for matrices $A$ and $B$, $AB$ has the same non-zero eigenvalues as $BA$, counted with multiplicity.

\begin{filehproposition}\label{f08:prop:modtrace}
For symmetric matrices $A$ and $B$,
  \begin{enumerate}[label=\alph*., ref=\alph*]
    \item $\trs{A} = \max_{U} \trace{U A U^T}$, where the maximum is taken over all orthogonal matrices of rank $\sigma$. \label{f08:prop:modtrace:max}
    \item If $A$ is positive semidefinite, then $\trs{AB} = \trs{BA}$. \label{f08:prop:modtrace:rotate}
    \item If $A$ and $B$ are positive semidefinite, then $\trs{AB} \geq 0$. \label{f08:prop:modtrace:positive}
    \item If $A \pleq B$, then $\trs{A} \leq \trs{B}$. \label{f08:prop:modtrace:monotone}
    \item If $C$ is positive semidefinite and $A \pleq B$, then $\trs{AC} \leq \trs{BC}$. \label{f08:prop:modtrace:monotone2}
  \end{enumerate}
\end{filehproposition}
\begin{proof}
Part \ref{f08:prop:modtrace:max} is Ky Fan's maximum principle, proved in \cite{f08:KyFanWeyl}. Part \ref{f08:prop:modtrace:rotate} is a direct consequence of the facts that $AB$ has $n$ real eigenvalues if $A$ is positive semidefinite, and $AB$ and $BA$ have the same non-zero eigenvalues. 
Part \ref{f08:prop:modtrace:positive} follows from the fact that all eigenvalues of the product of positive semidefinite matrices are non-negative.
Part \ref{f08:prop:modtrace:monotone} follows from using \eqref{f08:eqn:modtrace:subadditive} to show $\trs{A} \leq \trs{B} + \trs{A - B} \leq \trs{B},$ using the fact that $A - B$ is negative semidefinite and so $\trs{A - B} \leq 0$.  
To derive part \ref{f08:prop:modtrace:monotone2} from part \ref{f08:prop:modtrace:monotone}, let $V$ be a matrix so that $V^T V = C$, and apply \ref{f08:prop:modtrace:rotate} to show the conclusion is equivalent to $\trs{V A V^T} \leq \trs{V B V^T}$, which follows from $V A V^T \pleq V B V^T$.
\end{proof}

Note that $\Ltil_{S \union \setof{t}} = \Ltil_{S} + \Ltil_{S,t}$.
To show that we can choose a $t \not \in S$ that does not increase the barrier function, 
  we employ the following adaptation of 
 Lemma~19 of \cite{f08:CarliSilvaHS}, which in turn is an adaptation of Lemma 3.3 from \cite{f08:BSS}.
We include a proof for completeness.

\begin{filehlemma}\label{f08:lem:upperBarrier}
Let $A$ and $B$ be positive semidefinite matrices, $\delta > 0$, and let $M = (u+\delta)I - A$.
If $\Phi^{u}_{\sigma} (A) < \infty$ and
\begin{equation}\label{f08:eqn:upperBarrier}
\frac{\trs{M^{-2} B}}
     {\Phi^{u}_{\sigma} (A) - \Phi^{u+\delta }_{\sigma} (A) }
+ 
\trs{M^{-1} B}
< 1, 
\end{equation}
then $\Phi^{u+\delta}_{\sigma} (A + B) \leq \Phi^{u}_{\sigma} (A)$.
\end{filehlemma}

\begin{proof}
Our assumption that $\Phi^{u}_{\sigma} (A) < \infty$ implies that 
  $M$, $M^{-1}$, and $M^{-2}$ are all positive definite.
Thus, Proposition~\ref{f08:prop:modtrace}\ref{f08:prop:modtrace:positive} implies that both terms in \eqref{f08:eqn:upperBarrier} are non-negative.
Let $C$ be a matrix for which  $B = C C^T$,
  and so by
 Proposition~\ref{f08:prop:modtrace}\ref{f08:prop:modtrace:rotate}
   $\trs{M^{-1} B} = \trs{C^T M^{-1} C} < 1$.

Recall $\Phi^{u+\delta}_{\sigma} (A + B) 
=  
\trs{(M - C C^T)^{-1}}.$
By the Sherman-Morrison-Woodbury formula, 
\[ 
(M - C C^T)^{-1} 
= 
M^{-1} + M^{-1} C (I - C^T M^{-1} C)^{-1} C^T M^{-1}.
\]
As $\norm{C^T M^{-1} C} \leq \trs{C^T M^{-1} C} < 1$, we know that right-hand term is positive definite, and thus all eigenvalues of $A+B$ are less than $u + \delta$.
Now, \eqref{f08:eqn:modtrace:subadditive} implies 
\[
\Phi^{u+\delta}_{\sigma} (A + B) 
\leq 
\trs{M^{-1}} + \trs{M^{-1} C (I - C^T M^{-1} C)^{-1} C^T M^{-1}}.
\]
By Propositon~\ref{f08:prop:modtrace}\ref{f08:prop:modtrace:rotate},
\[
\trs{M^{-1} C (I - C^T M^{-1} C)^{-1} C^T M^{-1}} 
= 
\trs{ (I - C^T M^{-1} C)^{-1} C^T M^{-2} C}
\]
As $\norm{C^T M^{-1} C} \leq \trs{C^T M^{-1} C} < 1$,  
$(I - C^T M^{-1} C)^{-1} \pleq  (1 - \trs{C^T M^{-1} C})^{-1} I$, and 
by Proposition~\ref{f08:prop:modtrace}\ref{f08:prop:modtrace:monotone}, 
\[
\trs{ (I - C^T M^{-1} C)^{-1} C^T M^{-2} C}
\leq 
\frac{\trs{C^T M^{-2} C}}{1 - \trs{C^T M^{-1} C}}.
\]
Writing $\trs{M^{-1}} = \Phi^{u}_{\sigma} (A) - (\Phi^{u}_{\sigma} (A) - \Phi^{u+\delta}_{\sigma} (A))$, 
we obtain 
\[
\Phi^{u+\delta}_{\sigma} (A + B) 
\leq 
\Phi^{u}_{\sigma} (A) - (\Phi^{u}_{\sigma} (A) - \Phi^{u+\delta}_{\sigma} (A)) 
+ 
\frac{\trs{C^T M^{-2} C}}{1 - \trs{C^T M^{-1} C}},
\]
which \eqref{f08:eqn:upperBarrier} and Proposition~\ref{f08:prop:modtrace}\ref{f08:prop:modtrace:rotate} imply is at most $\Phi^{u}_{\sigma} (A)$.
\end{proof}

We will apply this result with $A = \Ltil_{S}$ and $B = \Ltil_{S,t}$.
When these terms, along with $u$ and $\delta$ are given, 
it will be convenient to write 
\[
U (S,t) \defeq 
\frac{\trs{M^{-2} \Ltil_{S,t}}}
     {\Phi^{u}_{\sigma} (S) - \Phi^{u+\delta }_{\sigma} (S) }
     + \trs{M^{-1} \Ltil_{S,t}}.
\]

\begin{filehlemma}\label{f08:lem:barrier}
If $\sizeof{S} \leq \sigma$, $\Phi^{u}_{\sigma} (S) \leq \phi$, and $\ell(S) \leq 4 \sizeof{S}$, then for at least half the $t \not \in S$, 
\[
U (S,t)
<
1
\]
\end{filehlemma}
\begin{proof}
We will prove that 
\[
\sum_{t \not \in S} U(S,t) \leq \frac{5}{\delta} + 5 \phi.
\]
As $U(S,t)$ is non-negative, this implies that for at least half the $t \not \in S$, 
\[
U(S,t) \leq \frac{2}{n - \sizeof{S}} \left( \frac{5}{\delta} + 5 \phi \right)
\leq \frac{2}{n} \frac{42}{41} \left( \frac{5 n}{21} + \frac{5 n}{21} \right)
< 1.
\]  

We need to upper bound the terms $\trs{M^{p} \Ltil_{S,t}}$ for $p \in \setof{-1,-2}$.
We do this by breaking each term into two parts.
Let $\Pi_{S}$ be the symmetric  projection onto the span of $\Ltil_{S}$
  and let $\Pi_{T} = I - \Pi_{S}$.
As $M = (u + \delta) (\Pi_{S} + \Pi_{T}) - \Ltil_{S}$, 
  $\Pi_{T} \Pi_{S} = \Pi_{T} \Ltil_{S} = 0$,
  and $\Pi_{S}^p = \Pi_{S}$, 
\[ 
  M^p = (u + \delta)^p \Pi_{T} +  
    \left( (u + \delta) \Pi_{S} - \Ltil_{S} \right)^p.
\]  
By the subadditivity of $\Trs$ 
 we conclude
\[ 
\trs{M^{p} \Ltil_{S,t}}
\leq 
\trs{ (u + \delta)^p \Pi_{T} \Ltil_{S,t}}
+
\trs{ \left( (u + \delta) \Pi_{S} - \Ltil_{S} \right)^p \Ltil_{S,t}}.
\]
The term invovling $\Pi_{S}$ is addressed by Claim~\ref{f08:clm:barrierSPart}, which says
\[
  \sum_{t \not \in S} 
  \trs{\left( (u+\delta)\Pi_{S} - \Ltil_{S} \right)^{p} \Ltil_{S,t}}
\leq
  \trs{M^{p}}.
\]

For the other term, we recall that $\Pi_{T}$ and $\Ltil_{S,t}$ are positive semidefinite and so their product has only non-negative eigenvalues to show
\[
\trs{ (u + \delta)^p \Pi_{T} \Ltil_{S,t}}
\leq 
\trace{ (u + \delta)^p \Pi_{T} \Ltil_{S,t}}
\\ = 
(u + \delta)^p \trace{ \Pi_{T} \Ltil_{S,t}}
\leq 
(u + \delta)^p \trace{\Ltil_{S,t}}.
\]
Claim~\ref{f08:clm:ellSt} tells us that this equals $(u + \delta)^p \ell(S,t)$, giving
\[
\sum_{t \not \in S} \trs{ (u + \delta)^p \Pi_{T} \Ltil_{S,t}}
\leq 
(u + \delta)^p \sum_{t \not \in S} \ell(S,t) 
= 
(u + \delta)^p \ell(S)
\leq 
(u + \delta)^p 4 \sizeof{S}.
\]

To combine these terms, note that all the eigenvalues of $M$ are at most $(u+\delta)$, and thus for $p < 0$ all the eigenvalues of $M^p$ are at least $(u+\delta)^p$. 
This tells us that $\trs{M^p} \geq \sigma (u+\delta)^p \geq \sizeof{S} (u+\delta)^p$. 
We conclude that
\[
\sum_{t \not \in S} \trs{M^p \Ltil_{S,t}}
\leq 
5 \trs{M^p}.
\]

To finish, we return to  
\[
\sum_{t \not \in S}
U (S,t)
= 
\sum_{t \not \in S}
\frac{\trs{M^{-2} \Ltil_{S,t}}}
     {\Phi^{u}_{\sigma} (S) - \Phi^{u+\delta }_{\sigma} (S) }
+ 
\sum_{t \not \in S}
\trs{M^{-1} \Ltil_{S,t}}
\leq
\frac{5 \trs{M^{-2}}}
     {\Phi^{u}_{\sigma} (S) - \Phi^{u+\delta }_{\sigma} (S) }
+ 
5 \trs{M^{-1}}.
\]
The right-hand term is at most $5 \Phi^{u+\delta}_{\sigma} (S)$,
  and Claim~\ref{f08:clm:barrierPhi} shows that
  the left-hand term is at most $\frac{5}{\delta}$.
Summing these together gives the result.
\end{proof}

\begin{filehclaim}\label{f08:clm:barrierSPart}
Assume that $\sizeof{S} \leq  \sigma$. 
For $M = (u+\delta)I - \Ltil_{S}$, and nonzero real  $p$,
\[
  \sum_{t \not \in S} 
  \trs{\left( (u+\delta)\Pi_{S} - \Ltil_{S} \right)^{p} \Ltil_{S,t}}
\leq
  \trs{M^{p}}.
\]
\end{filehclaim}
\begin{proof}
Because both $\Ltil_{S,t}$ and $\left( (u+\delta)\Pi_{S} - \Ltil_{S} \right)^{p}$ are positive semidefinite, the eigenvalues of their product are nonnegative, and so 
\[
  \trs{\left( (u+\delta)\Pi_{S} - \Ltil_{S} \right)^{p} \Ltil_{S,t}}
\leq 
  \trace{\left( (u+\delta)\Pi_{S} - \Ltil_{S} \right)^{p} \Ltil_{S,t}}.
\] 
As $\sum_{t \not \in S} \Ltil_{S,t} = \Ltil_{S,T} \pleq I$, 
  Proposition~\ref{f08:prop:modtrace}\ref{f08:prop:modtrace:monotone} implies
\begin{multline*}
  \sum_{t \not \in S} 
  \trace{\left( (u+\delta)\Pi_{S} - \Ltil_{S} \right)^{p} \Ltil_{S,t}}
=
  \trace{\left( (u+\delta)\Pi_{S} - \Ltil_{S} \right)^{p} \Ltil_{S,T}}
\\
  \leq 
  \trace{\left( (u+\delta)\Pi_{S} - \Ltil_{S} \right)^{p}}
  = 
    \trace{\Pi_{S} \left( (u+\delta) I - \Ltil_{S} \right)^{p} \Pi_{S}}
    = 
        \trace{\Pi_{S} M^{p} \Pi_{S}}.
\end{multline*}
By Ky Fan's maximum principle (Proposition~\ref{f08:prop:modtrace}\ref{f08:prop:modtrace:max})
  this latter term is at most $\trs{M^{p}}$.
\end{proof}

\begin{filehclaim}\label{f08:clm:barrierPhi}
\[
\Phi^{u}_{\sigma} (S) - \Phi^{u+\delta }_{\sigma} (S) 
\geq 
\delta \trs{M^{-2}}.
\]
\end{filehclaim}
\begin{proof}
Let $\lambda_{1}, \dotsc , \lambda_{\sigma}$
  be the largest $\sigma$ eigenvalues of $\Ltil_{S}$.
Then,
\begin{align*}
\Phi^{u}_{\sigma} (S) - \Phi^{u+\delta }_{\sigma} (S)
& =
  \sum_{i=1}^{\sigma} \frac{1}{u - \lambda_{i}}
- 
  \sum_{i=1}^{\sigma} \frac{1}{u + \delta - \lambda_{i}}
\\
& =
  \sum_{i=1}^{\sigma} \frac{\delta }{(u - \lambda_{i}) (u + \delta - \lambda_{i})}
\\
& \geq 
  \sum_{i=1}^{\sigma} \frac{\delta }{(u + \delta - \lambda_{i})^{2}}.
\\
& = 
 \delta \trs{M^{-2}}.
\end{align*}
\end{proof}

\providecommand{\etalchar}[1]{$^{#1}$}

{\PatchBibSection

\EndPatchBibSection}
\endgroup

\subsection{Question 7: Shmuel Weinberger}
\textit{Authors:} Sylvain Cappell, S. Weinberger, and M. Yan\par\medskip
\textit{Title:} Fowler's theorem for involutions\par\medskip
\begingroup

\setcounter{equation}{0}
Fowler, in his Ph.D. thesis, proved that if $\Gamma$ is a uniform lattice in a real semisimple group with odd torsion in $\Gamma$ then there is no compact closed manifold $M$ whose universal cover is rationally acyclic. A proof can be found in [W2]. We show that the same is true for $\Gamma$ with 2-torsion.

Without loss of generality (by considering a normal subgroup of finite index), it suffices to prove this for the special case where $\Gamma = \pi \rtimes \mathbb{Z}_{2}$ for a torsion free group $\pi$, a lattice in $G$, for which there is an involution on $M = K\backslash G/\pi$ (by isometries with the locally symmetric metric) whose fixed set $F$ is not empty. ($F$ might be disconnected; for simplicity we will write what follows just for the connected case -- there are no differences in the general case.)

Now suppose that $X^{m}$ is a manifold with fundamental group $\Gamma$, $Y$ its 2-fold cover, and suppose that the universal cover of $X$ (and therefore $Y$) are rationally acyclic. We will consider the symmetric signatures of $Y$ in the (symmetric = quadratic L-group) $L(\mathbb{R}\pi)$, where $\mathbb{R}$ is the real numbers. There is an equivalence $f:Y \to M$ which (while not degree one) gives an equivalence of symmetric signatures (because over $\mathbb{R}$, all degrees have square roots, so the symmetric signature is only sensitive to the sign of the degree of the map). Since the Novikov conjecture is true for $\pi$, the assembly map from $H_{m}(B\pi; L(\mathbb{R})) \to L_{m}(\mathbb{R}\pi)$ is injective, and this detects in the degree $m$ piece $H_{m}(B\pi; \mathbb{Z})$ the class that these manifolds represent in group homology. It follows that this map is degree one. $f_{*}[Y] = [M]$.

Now we use a cobordism argument from [W1]. We now consider the image of the fundamental class of any manifold $Z$ with fundamental group $\pi$ involution inducing this automorphism of $\pi$ and the image of $[Z]$ in $H_{m}(B\Gamma; \mathbb{Z}_{2})$. It follows from standard equivariant homotopy theory that $Z$ has an equivariant map, $g$, to $M$, and thus there is a map from its fixed set $Z^{\mathbb{Z}_2} \to F$. We claim that $g_{*}[Z] = g_{*}[Z^{\mathbb{Z}_2}]$ where we make use of the map from $\mathbb{Z}_{2} \times \pi_{1}F \to \Gamma$ (and the periodicity on the group homology of $\mathbb{Z}_{2}$ to raise the dimension from that of $F$ to $\dim M$).

This cobordism is between $Z$ and a projective space bundle over $Z^{\mathbb{Z}_2}$ -- namely the projectivized normal bundle to $Z^{\mathbb{Z}_2}$. (The fundamental class of the latter is the desired element by the Leray-Hirsch theorem.) It is explicitly $Z \times [0,1]$ and on $Z \times \{1\}$ mod out in the complement of the equivariant regular neighborhood of $Z^{\mathbb{Z}_2}$ the $\mathbb{Z}/2$ action.

Thus for $Y$, this image is 0, since the action is free. For $M$ however, this is always nonzero. The action by $\mathbb{Z}_{2}$ by isometries has fixed set which is aspherical and indeed the Borel construction for the action on $M$ shows that $\mathbb{Z}_{2}\times F \to \Gamma$ induces an injection on homology in dimension $\dim(M/\mathbb{Z}_{2})$ (and an isomorphism in higher dimensions, see [B]). Since the fundamental class of an aspherical manifold is always nontrivial in its group homology, we have a contradiction.

\subsubsection*{References}
\begin{description}[leftmargin=2.5em,style=nextline]
\item[[B]] A. Borel, A seminar on transformation groups, Princeton University Press 1960
\item[[W1]] S. Weinberger, Group actions and higher signatures II, CPAM 1987
\item[[W2]] S. Weinberger, Variations on a theorem of Borel, Cambridge University Press 2022
\end{description}
\endgroup

\subsection{Question 8: Mohammed Abouzaid}
\textit{Author:} Mohammed Abouzaid\par\medskip
\textit{Title:} Smoothing Lagrangian Surface\par\medskip
\begingroup
\providecommand{\bR}{\mathbb{R}}
\providecommand{\co}{\colon}
\providecommand{\scrS}{\mathscr{S}}

\theoremstyle{definition}
\newtheorem{fileadefinition}{Definition}
\newtheorem{filealemma}{Lemma}
\newtheorem{fileaproposition}{Proposition}
\newtheorem{fileacorollary}{Corollary}
\newtheorem{filearemark}{Remark}

\setcounter{equation}{0}
\begin{filearemark}
  This note is expanded from a short motivating discussion in a research paper that is supposed to develop a theory of polyhedral Lagrangian submanifolds for the purpose of being able to use computers to explore conjectures in symplectic topology. It includes some details that would normally be omitted (e.g. the proof of Lemma \ref{f01:lem:local_triviality}, which is a linear algebra exercise, and much of the explanation about closed $1$-forms). The paper does not cite any references as the reader is assumed to be able to deduce all asserted results from standard references, e.g. \cite{f01:Hirsch, f01:MDS}.

  I would like to thank Kyler Siegel and Umut Varolgunes for helpful discussions around this circle of ideas.
\end{filearemark}

For the purpose of this note, we equip $\bR^4$ with coordinates $(q_1, q_2, p_1, p_2)$, and with the standard symplectic form $\omega = dp_1 \wedge dq_1 + dp_2 \wedge dq_2$.
\begin{fileadefinition}
  A polyhedral Lagrangian surface in $\bR^4$ is a finite polyhedral complex all of whose faces are Lagrangians, and which is a topological submanifold of $\bR^4$.
\end{fileadefinition}

\begin{fileaproposition} \label{f01:prop:existence_Hamiltonian_family}
If $K$ is a polyhedral Lagrangian surface with the property that exactly $4$ faces meet at every vertex, then there is a Hamiltonian isotopy $K_t$ of smooth Lagrangian submanifolds, parameterised by $(0,1]$, extending to a topological isotopy, parametrised by $[0,1]$, with endpoint $K_0 = K$.
  \end{fileaproposition}

  In order to prove this result, we need two preliminary results: a local statement asserting triviality near each vertex, and a global statement implying the compatibility of these local trivialisations.
  \begin{filealemma} \label{f01:lem:local_triviality}
    For each embedding $\bR^2 \to \bR^4$ which is linear on the four quadrants with Lagrangian image, and whose image $\Sigma$ is not contained in a plane, there is a linear symplectic transformation of $\bR^4$ which maps $\Sigma$ to the product of the union of the positive coordinate axes in $\bR^2_{p_1q_1}$ and  $\bR^2_{p_2q_2}$.
  \end{filealemma}
  \begin{proof}
Let $(v_1,v_2,u_1,u_2)$ denote tangent vectors at the origin to the edges of $\Sigma$, ordered so that cyclically adjacent vectors span the faces of $\Sigma$. The pairings $\omega(v_i,u_i)$ cannot vanish, for otherwise $\omega$ would identically vanish on a $3$-dimensional linear subspace. By swapping the pair of coordinates $(v_i,  u_i) $ if necessary, we may assume that both pairings are strictly positive, and by rescaling we may assume that they are $1$. We conclude that the vectors $(v_1,v_2,u_1,u_2)$ form a standard symplectic basis for $\bR^4$, and that the mapping $\partial_{p_i} \to v_i$ and $\partial_{q_i} \to u_i$ is the desired linear transformation.
  \end{proof}
  In the plane $\bR^2_{pq}$, the symplectic pairing projects the union of the positive axes homeomorphically to the dual of the line $p=q$. Taking the product, and applying the previous Lemma, we conclude:
  \begin{fileacorollary} \label{f01:cor:dual_plane_local}
    There exists a linear Lagrangian plane $L \subset \bR^4$ so that the symplectic pairing $\bR^4 \to L^\vee$ defines a homeomorphism $\Sigma \to L^\vee$. \qed
  \end{fileacorollary}
  The previous corollary in particular equips $\Sigma$ with a smooth structure arising from its projection to $L^\vee$. This smooth structure will be fixed for the remainder of the discussion.

 Given a choice of plane $L$, we say that a Lagrangian $\Lambda \subset \bR^4$ is \emph{graphical} if the symplectic pairing defines a diffeomorphism $ \Lambda \cong L^\vee$. If $\Sigma$ were smooth, the standard description of Lagrangians in cotangent bundles would imply that such Lagrangians bijectively correspond to smooth closed $1$-forms, which, because $\Sigma$ is contractible and hence every closed form on it is exact, can be identified with smooth functions modulo addition of constants. We shall formulate a replacement for this correspondence that accounts for the singularities of $\Sigma$.
  
To this end, let us choose further a Lagrangian splitting of the projection $\bR^4 \to L^\vee$; we shall later see that our constructions are independent of this choice. The splitting gives a direct sum decomposition $\bR^4 \cong L \oplus L^\vee$  (polarization), with respect to which the image of each quadrant is graphical over $L^\vee$. Graphical (linear) Lagrangians bijectively correspond to quadratic forms, so we obtain quadratic forms $\{ q_{ij}\}_{i,j \in \pm}$ on $L^\vee$ whose graphs contain the corresponding faces of $\Sigma$. The restriction of the quadratic forms associated to any two faces agree to first order along the images in $L^\vee$ of the edges of $\Sigma$. Via the identification $\Sigma \cong L^\vee$ from the previous corollary, we write $q_{\Sigma}$ for the $C^1$-function on $\Sigma$ whose restriction to each face is given by the composition of $q_{ij}$ with the projection to $L^\vee$.   We use this to obtain an explicit description of the desired local smoothings, which will be essential in establishing the required global smoothability:
  \begin{fileadefinition} \label{f01:def:smoothing_function}
The space $\scrS(\Sigma)$ of  \emph{smoothing functions} for $\Sigma$ is the space of $C^1$ functions $f \co \Sigma \to \bR$ satisfying the property that the function on $f + q_{\Sigma}$  is infinitely differentiable.
\end{fileadefinition}
It follows immediately from the definition that  $\scrS(\Sigma)$ is invariant under addition of smooth functions, which will be used in the next result:
\begin{filealemma} \label{f01:lem:independence_dual_Lagrangian}
    The space of smoothing functions  $\scrS(\Sigma)$ depends only on $L$ (and not on the splitting of the projection $\bR^4 \to L^\vee$).
  \end{filealemma}
  \begin{proof}
A different choice of complementary subspaces correspond to adding a quadratic form $q'$ to $q_{ij}$, and the corresponding smooth function on $\Sigma$ to $q_{\Sigma}$.
\end{proof}

 We shall now associate a graphical Lagrangian to each smoothing function: the construction relies on the fact that the union of all translates of $L$ passing through a face of $\Sigma$ is canonically symplectomorphic to the cotangent bundle of $\Sigma$, with the cotangent fibre at $z \in \Sigma$ corresponding to the translate of $L$ passing through $z$. In this way, a smoothing function $f$ determines a Lagrangian $\Lambda_{df} \subset \bR^4$, piecewise as the graph of the restriction of the differential $df$ to each face. 
  \begin{filealemma} \label{f01:lem:functions_with_corners}
    The assignment $f \mapsto \Lambda_{df} $  determines a bijective correspondence between graphical Lagrangians and smoothing functions on $\Sigma$ up to addition of constants.
  \end{filealemma}
  \begin{proof}
    In terms of the polarization from the discussion preceding Definition \ref{f01:def:smoothing_function}, the Lagrangian $\Lambda_{df}$ corresponds to the graph of the differential of the function $f + q_{\Sigma}$ considered as a function on $L^\vee$ via the projection map, because each face of $\Sigma$ is the graph of $dq_{ij}$. The result now follows from the fact that graphical Lagrangians over $L^\vee$ are graphs of differentials of  smooth functions.
  \end{proof}
Note that while the proof uses the polarization, the construction does not. As in Lemma \ref{f01:lem:independence_dual_Lagrangian}, we conclude that this bijection depends only on the choice of Lagrangian $L$.

 The above completes our local analysis near vertices. Near edges, the analysis is much simpler:
 \begin{filealemma} \label{f01:lem:transverse_plane_near_edge}
   If $\Sigma$ consists of a pair of linear Lagrangian half-planes in $\bR^4$ meeting along a line $\ell$, then the space of Lagrangian subspaces $L$, satisfying the property that the symplectic pairing $\Sigma \to L^\vee$ is a homeomorphism, is contractible.
 \end{filealemma}
 \begin{proof}
   The submanifold $\Sigma$ is equivalent by (affine) linear symplectic transformations to the symplectic product of the real axis in an $\bR^2$ factor with the piecewise Lagrangian consisting of the positive axes in another. If the projection $\Sigma \to L^\vee$ is a homeomorphism, then $L$ must be transverse to both Lagrangian half-planes comprising $\Sigma$. This implies that the symplectic reduction of $L$ along  $\ell$ (i.e. the image under the quotient by $\ell$ of the intersection of $L$ with the symplectic annihilator $\ell^\perp$) is a line transverse to two coordinate lines in $\ell^\perp/\ell \cong \bR^2 $, and $\Sigma$ projects homeomorphically to $L^\vee$ if and only if this reduction intersects the interior of the positive quadrant, which is a contractible condition. The argument is completed by noting that the space of Lagrangian lifts of a line $\ell'$ in $\bR^2$ is contractible: any two lifts to $\ell^\perp$ differ by the graph of a map from $\ell'$ to $\ell$, and $L$ is determined up to contractible choice by $L \cap \ell^\perp$, since it must lie in the symplectic orthogonal of this line, and the space of planes in $\bR^3$ containing a given line (in this case $L \cap \ell^\perp$) and avoiding another line (in this case $\ell$) is contractible.
 \end{proof}
 Extending Definitions \ref{f01:def:smoothing_function} and \ref{f01:def:conormal_fibration_local} verbatim to the case of a pair of edges, we obtain the analogue of Lemma \ref{f01:lem:functions_with_corners}, using a splitting into factors as in the above proof.

In the global setting, we cannot work with translates with a single Lagrangian, so we need to consider a family $L_z$ of Lagrangian planes, passing through each point $z \in \Sigma$, which are not necessarily translates of each other. We shall require four properties of such a family, the first three of which are easy to state:
\begin{enumerate}
\item $L_z$ consists of translates of a single Lagrangian near the origin.
\item $L_z$ varies smoothly along the edges.
\item $L_z$ varies smoothly along the faces.
\end{enumerate}
To formulate the last property, say that $\sigma$ and $\sigma'$ are faces meeting along an edge $\tau$, and let $z$ be a point on $\tau$. The choice of $L_z$ determines an identification
\begin{equation*}
  T_z \sigma \cong L_z^\vee \cong T_z \sigma'
\end{equation*}
which is compatible with the inclusion of $T_z \tau$ on both sides. A \emph{matched normal field along $\tau$} is a choice of sections of $T \sigma|_{\tau}$ and $T \sigma'|_{\tau} $ which are inward pointing, and are opposite vectors under the above identification. For simplicity, we require this normal field, at the origin $\tau$, to point along the direction of the edge of $\sigma$ (or $\sigma'$) which meets $\tau$. Because the faces of $\Sigma$ are flat, this choice therefore determines an embedding $\tau \times [0,\epsilon) \to \sigma$, which is a collar neighbourhood (and similarly for $\sigma'$). 
   \begin{fileadefinition} \label{f01:def:conormal_fibration_local}
   A \emph{conormal fibration dual to $\Sigma$} is a family $L_z$ of (affine)-linear Lagrangian planes in $\bR^4$, parametrised by $z \in \Sigma$, satisfying the above three properties and so that, in a collar of each edge, the Lagrangians in the normal direction are translates of the Lagrangians along the edge. 
 \end{fileadefinition}
 The choice of collars in the above construction determines a smooth structure on $\Sigma$ by using negative coordinates on one of the collars as well as the identification $(-\epsilon,0] \cup [0,\epsilon) \cong (-\epsilon,\epsilon) $. This is an a priori different way of constructing a smooth structure than our earlier formulation, and the next result asserts the compatibility of these contructions; in this setting, we choose an affine-linear Lagrangian $\Lambda_z$ passing through $z$, which is transverse to $L_z$, and consider the (locally defined) map from $\Sigma$ to $\Lambda_z$ which assigns to $z' \in \Sigma$ near $z$ the intersection points  $L_{z'} \cap \Lambda_z$ which is unique because $L_z$ is close to $L_{z'}$.
 \begin{filealemma}
The projection map to $\Lambda_z$ is a local diffeomorphism.
 \end{filealemma}
 \begin{proof}
The only case that needs to be discussed is when $z$ lies on an edge $\tau$. The condition that $L_{z'}$ be given by translates along the collar direction implies that this map may be written along the collar of $\tau$ in a face $\sigma$ as $(t,s) \mapsto \gamma(t) + s \cdot \nu_{\sigma}(t)$, where $t$ is the coordinate along $\tau$ and $s \in [0,\epsilon)$ is the coordinate in the normal direction. The requirement that the normal fields are matched is equivalent to the condition that $\nu_{\sigma} = - \nu_{\sigma'} $ if $\sigma$ and $\sigma'$ are the two faces meeting  along $\tau$. The smoothness of the map is immediate from this description.
 \end{proof}

 Whenever the family $L_z$ does not consist of translates, the Lagrangians $L_z$ will have non-empty intersections. However, such intersections always take place outside some open neighbourhood $\nu \Sigma$ of $\Sigma$, which we now fix. As before, the fibration $L_z$ determines a projection $\nu \Sigma \to \Sigma$.  We say that a Lagrangian is \emph{graphical} with respect to $L_{z}$ if it is contained in this neighbourhood, and its projection to $\Sigma$ is a diffeomorphism. 
 \begin{filealemma} \label{f01:lem:graphical_Lagrangians_local}
   Every graphical Lagrangian with respect to $L_{z}$ arises as the graph of a smoothing function. Moreover, any smoothing function whose differential is sufficiently small defines a graphical Lagrangian.
  \end{filealemma}
  \begin{proof}
The correspondence between graphical Lagrangians and smoothing functions is local on $\Sigma$. It thus suffices to consider a point $z \in \Sigma$, and observe that a Lagrangian plane $L^\vee_z$ which is transverse to $L_z$ at $z$ will also be transverse to nearby fibres, so that a neighbourhood of $z$ in $\nu \Sigma$ is modelled after the conormal bundle of $L^\vee_z$, by Weinstein's tubular neighbourhood theorem. The result then follows by the standard construction of Lagrangians as graphs of closed $1$-forms.
  \end{proof}

  In order for the previous result to be helpful, we need to be able to produce the desired functions; this is not completely obvious because the space of smoothing functions is not invariant under rescaling:
  \begin{filealemma} \label{f01:lem:local_approximate_Sigma}
    There exist smoothing functions of arbitrarily small $C^1$-norm.
  \end{filealemma}
  \begin{proof}
As a preliminary step, choose a partition of unity $\sum_\sigma \chi_\sigma =1 $ on $\Sigma$, of bounded $C^k$-norms for all $k$, indexed by the strata of $\Sigma$, so that $ \chi_\sigma$ vanishes outside a small neighbourhood of $\sigma$ and its restriction to $\sigma$ is identically $1$ in the complement of a small neighbourhood of the boundary of $\sigma$. If $ \chi^\epsilon_\sigma  $ is the composition of $\chi_\sigma$ with the dilation of the plane by $1/\epsilon$, we obtain a family of partitions of unity which are uniformly bounded, and whose $C^1$-norms are bounded by a constant multiple of $1/\epsilon$.

We now choose a Lagrangian plane $\Lambda_\sigma$ which contains each stratum $\sigma \subset \Sigma$, and which is transverse to $L$, and let $f_\sigma$ denote the corresponding smoothing function. Note that the tangency conditions imply that the functions $f_\sigma$ and $f_{\sigma'}$ agree to first order along $\sigma \cap \sigma'$. Let $f^\epsilon$ denote the function $\sum \chi^\epsilon_\sigma f_\sigma $. The fact that $ f^\epsilon_\sigma$ is a family of smoothing functions follows from the partition of unity, and the fact that the $C^1$-norm is bounded follows from the product rule and the observation that, while the norm of the gradient of $\chi^\epsilon_\sigma $ grows like $1/\epsilon$, it is supported in a region where the difference between $f_\sigma$ and $f_{\sigma'}$ is bounded by a constant multiple of $\epsilon^2$.
\end{proof}

We now proceed with the global part of the argument, and thus return to the setting where $K$ is a polyhedral Lagrangian surface in $\bR^4$. The first step is to globalise the choice of $L$:
 \begin{fileadefinition}
   A \emph{conormal fibration dual to $K$} is a smoothly varying family $L_z$ of (affine)-linear Lagrangian planes in $\bR^4$, parametrised by $z \in K$, which locally satisfies the properties from Definition \ref{f01:def:conormal_fibration_local}. 
 \end{fileadefinition}
\begin{filealemma}
The surface $K$ admits a dual conormal fibration which, near vertex, agrees with the choice given by Corollary \ref{f01:cor:dual_plane_local}.
 \end{filealemma}
 \begin{proof}
 Lemma \ref{f01:lem:transverse_plane_near_edge} implies that the choices near the vertices may be extended to the edges. Choosing a normal vector field to one of the faces that meets along an edge determines matched normals, and the extension to the interior of the faces is then standard, as the space of Lagrangian planes transverse to a given one is contractible.
 \end{proof}
 The conormal fibration determines a subset $\scrS(K)$ of the space of $C^1$-functions consisting of those functions which are smooth in the interior of each face, and which are smoothing functions in the sense of Definition \ref{f01:def:smoothing_function} near each edge and vertex.
 \begin{filealemma} \label{f01:lem:global_smoothing_function}
   There exist smoothing functions for $K$ of arbitrarily small $C^1$-norm.
 \end{filealemma}
 \begin{proof}
Choose a partition of unity $\sum_\alpha \rho_{\alpha} =1 $ on $K$, indexed by the strata of $K$, so that $\rho_\alpha$ is supported in the open star of $\alpha$ (the union of all strata adjacent to it). Lemma \ref{f01:lem:local_approximate_Sigma} asserts the existence of smoothing functions $f_\alpha$ of arbitrarily small $C^1$-norm defined on the open star of $\alpha$. The function $\sum_\alpha \rho_{\alpha} f_\alpha$ satisfies the desired property.  
\end{proof}

We now arrive at the proof of the main result, which mostly consists of assembling together all the previous steps:
\begin{proof}[Proof of Proposition \ref{f01:prop:existence_Hamiltonian_family}]
  We have a neighbourhood $\nu K$ of $K$ in $\bR^4$ in which the conormal fibres $L_z$ are disjoint. The statement of Lemma \ref{f01:lem:graphical_Lagrangians_local} and its proof apply verbatim to this space, replacing $K$ by $\Sigma$. The existence of sufficiently many global smoothing functions is guaranteed by Lemma \ref{f01:lem:global_smoothing_function}.

  As a consequence, we obtain a sequence $K_i$ of smooth embedded Lagrangians, which are all isotopic to $K$ by a piecewise smooth isotopy and converge to it, that are moreover graphs of differentials of smooth functions (over each other) with respect to the fibration $\{ L_z \}$. This graphical description yields a smooth Hamiltonian path of graphical Lagrangians connecting $K_i$ to $K_{i+1}$, and smoothing the concatenation of these paths yields the desired result.
\end{proof}

{\PatchBibSection

\EndPatchBibSection}
\endgroup

\subsection{Question 9: Joe Kileel}
\textit{Authors:} Work by D.~Miao, G.~Lerman, J.~Kileel\par\medskip
\begingroup
\newtheorem{fileelemma}{Lemma}
\setcounter{equation}{0}
Yes, such algebraic relations do exist.  
Assemble the various tensors $\{Q^{(\alpha \beta \gamma \delta)} : \alpha, \beta, \gamma, \delta \in [n] \}$ into one tensor $\mathbf{Q} \in \mathbb{R}^{3n \times 3n \times 3n \times 3n}$, thought of as an $n \times n \times n \times n$ block tensor where the $(\alpha, \beta, \gamma, \delta)$-block is $Q^{(\alpha \beta \gamma \delta)} \in \mathbb{R}^{3 \times 3 \times 3 \times 3}$. 
Let $\mathbf{F}$ be the polynomial map sending $\{Q^{(\alpha \beta \gamma \delta)} : \alpha, \beta, \gamma, \delta \in [n] \}$ to the $5 \times 5$ minors of the four $3n \times 27n^3$ matrix flattenings of $\mathbf{Q}$.
We will prove that $\mathbf{F}$ satisfies the desired properties.

A key point is to discover the following algebraic identity.

\begin{fileelemma}
\label{f05:lem:identity}
Consider $\mathbf{Q} \in \mathbb{R}^{3n \times 3n \times 3n \times 3n}$ as above.  
It admits a Tucker tensor decomposition
\begin{equation}\label{f05:eq:tucker}
\mathbf{Q} = \mathcal{C} \times_1 \mathbf{A} \times_2 \mathbf{A} \times_3 \mathbf{A} \times_4 \mathbf{A},
\end{equation}
for $\mathcal{C} \in \mathbb{R}^{4 \times 4 \times 4 \times 4}$ and $\mathbf{A} \in \mathbb{R}^{3n \times 4}$.  
Explicitly, we can take 
\begin{equation*}
\mathcal{C}_{abcd} = 
\begin{cases}
\operatorname{sgn}(abcd) & \text{ if } a,b,c,d \in [4] \text{ are distinct} \\
0 & \text{ otherwise},
\end{cases}
\end{equation*}
where sgn is parity of a permutation, and $\mathbf{A}$ to be the vertical concatenation $[A^{(1)}; \dots ; A^{(n)}]$.
\end{fileelemma}
\begin{proof}
Let $[n] \times [3]$ stand for the indices of $\mathbf{Q}$ in each mode and for the row indices of $\mathbf{A}$.
By definition of Tucker product, for all $(\alpha, i), (\beta, j), (\gamma, k), (\delta, \ell) \in [n] \times [3]$ we have
\begin{align*}
  &\left(\mathcal{C} \times_1 \mathbf{A} \times_2 \mathbf{A} \times_3 \mathbf{A} \times_4 \mathbf{A}\right)_{(\alpha, i), (\beta, j), (\gamma, k), (\delta, \ell)}  \\[0.25em] 
  &= \sum_{a,b,c,d \in [4]} \mathcal{C}_{abcd} \mathbf{A}_{(\alpha, i), a} \mathbf{A}_{(\beta, j), b} \mathbf{A}_{(\gamma, k), c} \mathbf{A}_{(\delta, \ell), d} \\[0.25em]
  &= \sum_{a,b,c,d \in [4] \text{ distinct}} \operatorname{sgn}(abcd) A^{(\alpha)}_{ia} A^{(\beta)}_{jb} A^{(\alpha)}_{kc} A^{(\alpha)}_{\ell d} \\[0.25em]
  &= \det [A^{(\alpha)}(i, :); A^{(\beta)}(j, :); A^{(\gamma)}(k, :); A^{(\delta)}(\ell, :)] \\[0.25em]
  &= Q^{(\alpha \beta \gamma \delta)}_{i j k \ell} = \mathbf{Q}_{(\alpha, i), (\beta, j), (\gamma, k), (\delta, \ell)}. \qedhere
\end{align*} 
\end{proof}

The lemma explains why $\mathbf{F}$ captures algebraic relations between the tensors $\{Q^{(\alpha \beta \gamma \delta)} : \alpha, \beta, \gamma, \delta \in [n] \}$.  
Indeed, the block tensor $\mathbf{Q}$ has multilinear rank bounded by $(4,4,4,4)$ due to the Tucker decomposition in \eqref{f05:eq:tucker}.
Therefore, all $5 \times 5$ minors in $\mathbf{F}$ vanish.

Below we break up the proof of the third property into two directions.  
The other properties are clear.   
Throughout the proof, for $\lambda \in \mathbb{R}^{n \times n \times n \times n}$ we let $\lambda \odot_{b} \mathbf{Q} \in \mathbb{R}^{3n \times 3n \times 3n \times 3n}$ denote blockwise scalar multiplication, i.e., 
the $(\alpha, \beta, \gamma, \delta)$-block of $\lambda \odot_{b} \mathbf{Q}$ is $\lambda_{\alpha \beta \gamma \delta} Q^{(\alpha \beta \gamma \delta)} \in \mathbb{R}^{3 \times 3 \times 3 \times 3}$.
Roughly speaking, we need to show that a blockwise scaling of $\mathbf{Q}$ preserves multilinear rank if and only if the scaling is a rank-1 tensor off the diagonal.

\medskip

\paragraph*{``If" Direction}
This follows easily from Lemma~\ref{f05:lem:identity}.
Assume $\lambda \in \mathbb{R}^{n \times n \times n \times n}$ agrees off-diagonal with $u \otimes v \otimes w \otimes x$ for $u,v,w,x \in (\mathbb{R}^{*})^n$ and is $0$ on the diagonal.
Then
\begin{equation*}
\lambda \odot_b \mathbf{Q} = (u \otimes v \otimes w \otimes x) \odot_b \mathbf{Q},
\end{equation*}
because the diagonal blocks of $\mathbf{Q}$ vanish.  That is,
$
Q^{(\alpha \alpha \alpha \alpha)} = 0
$
since each entry of $Q^{(\alpha \alpha \alpha \alpha)}$ is the determinant of a matrix with a repeated row.
Note that blockwise scalar product with a rank-1 tensor with nonzero entries is equivalent to Tucker product with invertible matrices:
\begin{equation*}
(u \otimes v \otimes w \otimes w) \odot_b \mathbf{Q} = \mathbf{Q} \times_1 D_u \times_2 D_v \times_3 D_w \times_4 D_x.
\end{equation*}
Here $D_u \in \mathbb{R}^{3n \times 3n}$ is the diagonal matrix triplicating the entries of $u$ and likewise for $D_v, D_w, D_x$.
Thus $\lambda \odot_b \mathbf{Q}$ and $\mathbf{Q}$ have the same multilinear rank, and from the lemma $\mathbf{F}(\lambda_{\alpha \beta \gamma \delta} Q^{(\alpha \beta \gamma \delta)} : \alpha, \beta, \gamma, \delta \in [n]) = 0$.

\medskip

\paragraph*{``Only If" Direction}

The converse takes more work. 
Let $\lambda \in \mathbb{R}^{n \times n \times n \times n}$ have nonzero entries precisely off the diagonal and assume $\mathbf{F}(\lambda_{\alpha \beta \gamma \delta} Q^{(\alpha \beta \gamma \delta)} : \alpha, \beta, \gamma, \delta \in [n]) = 0$.
We further assume $\lambda_{\alpha 1 1 1 } = \lambda_{1 \beta 1 1} = \lambda_{1 1 \gamma 1} = \lambda_{1 1 1 \delta}=1$ for all $\alpha, \beta, \gamma, \delta \in \{2, \ldots, n\}$.
We reduce to this case by replacing $\lambda$ by its entrywise product with $\bar{u} \otimes \bar{v} \otimes \bar{w} \otimes \bar{x}$, where 
\begin{equation*}
\bar{u}_{\alpha} = \begin{cases} 
1 & \text{for } \alpha = 1\\
\lambda_{\alpha 1 1 1}^{-1} & \text{for } \alpha \in \{2, \dots, n\}, 
\end{cases}
\end{equation*}
and $\bar{v}, \bar{w}, \bar{x}$ are defined similarly using the second, third and fourth modes respectively. 
The replacement preserves the multilinear rank of $\lambda \odot_b \mathbf{Q}$ and whether or not $\lambda$ agrees off-diagonal with a rank-$1$ tensor.  Hence it is without loss of generality.

Through some explicit calculations, we will prove there exists $c \in \mathbb{R}^*$ such that
\begin{itemize}
    \item $\lambda_{\alpha \beta \gamma \delta} = c$ if exactly two of $\alpha, \beta, \gamma, \delta$ equal $1$
    \item $\lambda_{\alpha \beta \gamma \delta} = c^2$ if exactly one of $\alpha, \beta, \gamma, \delta$ equals $1$
    \item $\lambda_{\alpha \beta \gamma \delta} = c^3$ if none of $\alpha, \beta, \gamma, \delta$ equal $1$ and $\alpha, \beta, \gamma, \delta$ are not identical.
\end{itemize}
This will establish the ``only if" direction, as setting $u = v= w= (1, c, \ldots, c)$ and $x= (\tfrac{1}{c}, 1, \ldots, 1)$ gives $\lambda_{\alpha \beta \gamma \delta} = u_{\alpha} v_{\beta} w_{\gamma} x_{\delta}$ whenever $\alpha, \beta, \gamma, \delta$ are not identical.
Our proof strategy is to examine appropriate coordinates 
of $\mathbf{F}(\lambda_{\alpha \beta \gamma \delta} Q^{(\alpha \beta \gamma \delta)} : \alpha, \beta, \gamma, \delta \in [n]) = 0$ in order to constrain $\lambda$.
Equivalently, we will consider the vanishing of the determinants of certain well-chosen $5 \times 5$ submatrices of the flattenings of $\lambda \odot_b \mathbf{Q}$.  
Write $\mathbf{Q}_{(1)}$ and $(\lambda \odot_b \mathbf{Q})_{(1)}$ for mode-1 flattenings in $\mathbb{R}^{3n \times 27n^3}$.  
Rows correspond to the first tensor mode and are indexed by $(\alpha, i) \in [n] \times [3]$, while columns correspond to the other  modes and are indexed by $((\beta, j), (\gamma, k), (\delta, \ell)) \in ([n] \times [3])^{3}$.

\medskip

\underline{{Step 1}}:
The first submatrix of $(\lambda \odot_b \mathbf{Q})_{(1)}$ we consider has column indices
$((\alpha,1),(1,3),(1,2))$, 
$((1,2),(\beta,2),(1,1))$,
$((1,2),(\beta,3),(1,1))$,
$((1,3),(\beta,3),$ $(1,2))$,
$((1,1),(\beta,1),(1,3))$
and row indices 
$(1,1)$,
$(1,2)$,
$(1,3)$,
$(\alpha,1)$,
$(\alpha,2)$, 
where $\alpha, \beta \in \{2, \ldots, n\}$.
Explicitly, the submatrix \nolinebreak is  
\begin{equation*}
\begin{bmatrix}
{Q}^{(1\alpha11)}_{1132} & {Q}^{(11\beta1)}_{1221} & {Q}^{(11\beta1)}_{1231} & {Q}^{(11\beta1)}_{1332} & {Q}^{(11\beta1)}_{1113} \\[5pt] 
{Q}^{(1\alpha11)}_{2132} & {Q}^{(11\beta1)}_{2221} & {Q}^{(11\beta1)}_{2231} & {Q}^{(11\beta1)}_{2332} & {Q}^{(11\beta1)}_{2113} \\[5pt]
{Q}^{(1\alpha11)}_{3132} & {Q}^{(11\beta1)}_{3221} & {Q}^{(11\beta1)}_{3231} & {Q}^{(11\beta1)}_{3332} & {Q}^{(11\beta1)}_{3113} \\[4pt]
\lambda_{\alpha\alpha11}{Q}^{(\alpha\alpha11)}_{1132} & \lambda_{\alpha1\beta1}{Q}^{(\alpha1\beta1)}_{1221} & \lambda_{\alpha1\beta1}{Q}^{(\alpha1\beta1)}_{1231} & \lambda_{\alpha1\beta1}{Q}^{(\alpha1\beta1)}_{1332} & \lambda_{\alpha1\beta1}{Q}^{(\alpha1\beta1)}_{1113} \\[4pt]
\lambda_{\alpha\alpha11}{Q}^{(\alpha\alpha11)}_{2132} & \lambda_{\alpha1\beta1}{Q}^{(\alpha1\beta1)}_{2221} & \lambda_{\alpha1\beta1}{Q}^{(\alpha1\beta1)}_{2231} & \lambda_{\alpha1\beta1}{Q}^{(\alpha1\beta1)}_{2332} & \lambda_{\alpha1\beta1}{Q}^{(\alpha1\beta1)}_{2113}
\end{bmatrix},
\end{equation*}
which we abbreviate as 
\begin{equation} \label{f05:eq:my-det1}
\begin{bmatrix}
\ast & \ast & \ast & \ast & \ast \\
\ast & \ast & \ast & \ast & \ast \\
\ast & \ast & \ast & \ast & \ast \\
\lambda_{\alpha\alpha11}\ast & \lambda_{\alpha1\beta1}\ast & \lambda_{\alpha1\beta1}\ast & \lambda_{\alpha1\beta1}\ast & \lambda_{\alpha1\beta1}\ast \\
\lambda_{\alpha\alpha11}\ast & \lambda_{\alpha1\beta1}\ast & \lambda_{\alpha1\beta1}\ast & \lambda_{\alpha1\beta1}\ast & \lambda_{\alpha1\beta1}\ast \\
\end{bmatrix},
\end{equation}
with asterisk denoting the corresponding entry in $\mathbf{Q}_{(1)}$. 
We view the determinant of \eqref{f05:eq:my-det1} as a polynomial with respect to $\lambda$.
It has degree $\leq 2$ in the variables $\lambda_{\alpha \alpha 11}, \lambda_{\alpha 1 \beta 1}$. 
Observe that if $\lambda_{\alpha 1 \beta 1} = 0$, the bottom two rows of the matrix are linearly independent.  
Also if $\lambda_{\alpha 1 \beta 1} - \lambda_{\alpha \alpha 11} = 0$, then 
\eqref{f05:eq:my-det1} equals a $5 \times 5$ submatrix of $\mathbf{Q}_{(1)}$ with rows operations performed; therefore \eqref{f05:eq:my-det1} is rank-deficient.
It follows that the determinant of \eqref{f05:eq:my-det1} takes the form
\begin{equation*}
s \lambda_{\alpha 1 \beta 1}(\lambda_{\alpha 1 \beta 1} - \lambda_{\alpha \alpha 1 1}).
\end{equation*}
Here the scale $s = s(A^{(1)}, A^{(\alpha)}, A^{(\beta)})$ is a polynomial in the $A$-matrices. 
Due to polynomiality, $s$ is nonzero Zariski-generically if we can exhibit a \textit{single} instance of matrices $A^{(1)}, A^{(\alpha)}, A^{(\beta)}$ where the determinant of \eqref{f05:eq:my-det1} does not vanish identically for all $\lambda_{\alpha 1 \beta 1}, \lambda_{\alpha \alpha 1 1}$.
Furthermore, we just need an instance with $\alpha = \beta$, as this corresponds to a specialization of the case $\alpha \neq \beta$.  
Computational verification with a random numerical instance of $A^{(1)}, A^{(\alpha)}$ proves the non-vanishing (see attached code).  
Recalling the standing assumptions, we deduce 
$\lambda_{\alpha 1 \beta 1} = \lambda_{\alpha \alpha 1 1}$.

We apply the same argument to modewise permutations of $\lambda \odot_b \mathbf{Q}$ and $\mathbf{Q}$, and obtain
\begin{equation*}\label{f05:eq:nice-1}
\lambda_{\pi(\alpha 1 \beta 1)} = \lambda_{\pi(\alpha \alpha 1 1)} \quad \text{for all } \alpha, \beta \in \{2, \dots, n\} \text{ and permutations } \pi.
\end{equation*}
The argument goes through as $\pi \cdot \mathbf{Q}$ and $\pi \cdot (\lambda \odot_b \mathbf{Q})$ have multilinear ranks bounded by $(4,4,4,4)$ and $\pi \cdot \mathbf{Q} = \operatorname{sgn}(\pi) \mathbf{Q}$.  
So \eqref{f05:eq:my-det1} looks the same but with indices permuted and possibly a sign flip.

We now see that $\lambda$-entries with  two $1$-indices agree.
Indeed, taking $\alpha = \beta$ above gives $\lambda_{\pi_1(\alpha 1 \alpha 1)} = \lambda_{\pi_2(\alpha \alpha 1 1)}$ for all $\pi_1$ and $\pi_2$ that fix $(\alpha \alpha 1 1)$ and $(\alpha 1 \alpha 1)$ respectively.  
So, $\lambda_{\alpha \alpha 1 1} = \lambda_{\pi(\alpha \alpha 1 1)}$ for all $\pi$.  
Taking $\alpha \neq \beta$ gives $\lambda_{\alpha \alpha 1 1} = \lambda_{\pi(\alpha 1 \beta 1)} = \lambda_{\beta \beta 1 1}$ for all $\pi$.  
Together, there exists $c \in \mathbb{R}^*$ such that 
$c = \lambda_{\pi(\alpha \beta 1 1)}$ for all $\alpha, \beta \in \{2, \dots, n\}$ and permutations \nolinebreak $\pi$.

\medskip 

\underline{Step 2}:  
Next we consider the submatrix of $(\lambda \odot_b \mathbf{Q})_{(1)}$ with column indices  
$((\beta,1),(\gamma,3),(1,2))$, 
$((1,2),(\beta,2),(1,1))$,
$((1,2),(\beta,3),(1,1))$,
$((1,3),(\beta,3),$ $(1,2))$,
$((1,1),(\beta,1),(1,3))$
and row indices 
$(1,1)$,
$(1,2)$,
$(1,3)$,
$(\alpha,1)$,
$(\alpha,2)$, 
where $\alpha, \beta, \gamma \in \{2, \ldots, n\}$.
It looks like 
\begin{equation} \label{f05:eq:my-det2}
\begin{bmatrix}
c\ast & \ast & \ast & \ast & \ast \\
c\ast & \ast & \ast & \ast & \ast \\
c\ast & \ast & \ast & \ast & \ast \\
\lambda_{\alpha \beta \gamma 1}\ast & c\ast & c\ast & c\ast & c\ast \\
\lambda_{\alpha \beta \gamma 1}\ast & c\ast & c\ast & c\ast & c\ast \\
\end{bmatrix},
\end{equation}
where asterisks denote corresponding entries in $\mathbf{Q}_{(1)}$.
As a polynomial in $c$ and $\lambda_{\alpha \beta \gamma 1}$, the determinant of \eqref{f05:eq:my-det2} is a scalar multiple of $c(c^2 - \lambda_{\alpha \beta \gamma 1})$.
This is because the polynomial has degree $\leq 3$, if $c=0$ then the bottom two rows of \eqref{f05:eq:my-det2} are linearly dependent, and if $c^2 = \lambda_{\alpha \beta \gamma 1}$ then \eqref{f05:eq:my-det2} is a $5 \times 5$ submatrix of $\mathbf{Q}_{(1)}$ with row and column operations performed.  
The scale is a polynomial in $A^{(1)}, A^{(\alpha)}, A^{(\beta)}, A^{(\gamma)}$.
It is Zariski-generically nonzero if we exhibit one instance of $A$-matrices such that the determinant of \eqref{f05:eq:my-det1} does not vanish for all $c, \lambda_{\alpha \beta \gamma 1}$.
Further, it suffices to find an instance where $\alpha = \beta = \gamma$, as all other cases specialize to this.  Computational verification with a random numerical instance of $A^{(1)}, A^{(\alpha)}$ proves the non-vanishing.  
It follows that $c^2 = \lambda_{\alpha \beta \gamma 1}$. 
Appealing to symmetry like before, 
$
c^2 = \lambda_{\pi(\alpha \beta \gamma 1)}
$
for all $\alpha, \beta, \gamma \in \{2, \dots, n\}$ and permutations $\pi$.  
Summarizing, all $\lambda$-entries with a single $1$-index equal $c^2$.

\medskip

\underline{Step 3}:
Consider  the submatrix of $(\lambda \odot \mathbf{Q})_{(1)}$ with columns 
$((\beta,1),(\gamma,3),$ $(\delta, 2))$, 
$((1,2),(\alpha,2),$ $(1,1))$,
$((1,2),(\alpha,3),(1,1))$,
$((1,3),(\alpha,3),(1,2))$,
$((1,1),$ $(\alpha,1),(1,3))$
and rows 
$(1,1)$,
$(1,2)$,
$(1,3)$,
$(\alpha,1)$,
$(\alpha,2)$, 
where $\alpha, \beta, \gamma, \delta \in \{2, \ldots, n\}$ and $\alpha, \delta$ are distinct.
The submatrix looks like 
\begin{equation} \label{f05:eq:my-det3}
\begin{bmatrix}
c^2\ast & \ast & \ast & \ast & \ast \\
c^2\ast & \ast & \ast & \ast & \ast \\
c^2\ast & \ast & \ast & \ast & \ast \\
\lambda_{\alpha \beta \gamma \delta}\ast & c\ast & c\ast & c\ast & c\ast \\
\lambda_{\alpha \beta \gamma \delta}\ast & c\ast & c\ast & c\ast & c\ast \\
\end{bmatrix}.
\end{equation}
The determinant of \eqref{f05:eq:my-det3} is $c(c^3 - \lambda_{\alpha \beta \gamma \delta})$ multiplied by a polynomial in $A^{(1)}, A^{(\alpha)}, A^{(\beta)}, A^{(\gamma)}, A^{(\delta)}$.
The most specialized case is $\alpha=\beta=\gamma$.  Computer verification with a random numerical instance proves the polynomial is not identically zero.
We deduce that $c^3 = \lambda_{\alpha \beta \gamma \delta}$.  
By symmetry, $c^3 = \lambda_{\pi(\alpha \beta \gamma \delta)}$ for all $\alpha, \beta, \gamma, \delta \in \{2, \dots, n\}$ with $\alpha, \delta$  distinct and all permutations $\pi$.
In other words, $\lambda$-entries with no $1$-indices and non-identical indices equal $c^3$.  

\medskip

Steps 1, 2 and 3 show that  $\lambda$ takes the announced form.  So, $\lambda$ is rank-$1$ off the diagonal.  
This finishes the ``only if" direction.
Overall, we have proven that the $5 \times 5$ minors of the $3n \times 27n^3$ flattenings of $\mathbf{Q}$ give algebraic relations on $\{Q^{(\alpha \beta \gamma \delta)} : \alpha, \beta, \gamma, \delta \in [n] \}$ with the desired \nolinebreak properties.
\endgroup

\subsection{Question 10: Tammy Kolda}
\textit{Authors:} Johannes Brust and Tamara G. Kolda\par\medskip
\textit{Title:} Fast and Accurate CP-HIFI Solution\par\medskip
\begingroup
\providecommand{\vecop}{\operatorname{vec}}
\providecommand{\diag}{\operatorname{diag}}
\theoremstyle{plain}
\newtheorem{fileglemma}{Lemma}
\setcounter{equation}{0}

The system to be solved is
\begin{equation}\label{f07:eq:unaligned-infinite}
\left[
(Z \otimes K)^T S
S^T (Z \otimes K)
+ \lambda (I_r \otimes K)
\right] \vecop(W)
= (I_r \otimes K) \vecop(B).
\end{equation}

We consider several approaches for solving \cref{f07:eq:unaligned-infinite} in the remainder of this subsection.
We present a direct method for the symmetric linear system in \cref{f07:sec:direct-ui-sym}, using an additional regularization term.
In \cref{f07:sec:transformed-ui}, we present a transformation of the symmetric system based on the eigendecomposition of $K$.
In \cref{f07:sec:pcg-ui}, we present an iterative method based on the transformed symmetric system, adding some regularization akin to the symmetric direct method.
In \cref{f07:tab:unaligned-infinite-cost-comparison} and \cref{f07:sec:cost-compare-ui}, we provide an accounting of the costs and comparison of direct and iterative methods.

\subsubsection{Direct Solution of UI Subproblem (Symmetric Form)}
\label{f07:sec:direct-ui-sym}

\Cref{f07:eq:unaligned-infinite} is an indefinite symmetric linear system of size $rn \times rn$.
Since it is indefinite, we add a regularization term parameterized by $\rho > 0$ to ensure positive definiteness.
The modified system is
\begin{equation}\label{f07:eq:unaligned-infinite-1}
[ F^T F + \lambda (I_r \otimes K) + \rho I_{rn} ] \vecop(W) = \vecop(KB),
\end{equation}
where $F = S^T(Z \otimes K)$.
Observe that we have pulled $K$ inside the vectorization on the right-hand side.

To compute $F$, we want to avoid forming the $N \times nr$ Kronecker product $Z \otimes K$ explicitly.
Instead, we create two special matrices: $\hat{K} \in \mathbb{R}^{q \times n}$ and $\hat{Z} \in \mathbb{R}^{q \times r}$.
Each index $\ell \in [q]$ corresponds to a known entry index that we denote as $(i_1^{(\ell)}, i_2^{(\ell)}, \dots, i_d^{(\ell)}) \in \Omega$.
Then, for each $\ell \in [q]$, we let
\begin{align}
\label{f07:eq:zhat}
\hat{Z}(\ell,:) &= \left(
A_{d}(i_d^{(\ell)},:) \ast \cdots \ast A_{k+1}(i_{k+1}^{(\ell)},:) \ast
A_{k-1}(i_{k-1}^{(\ell)},:) \ast \cdots \ast A_{1}(i_1^{(\ell)},:)
\right)^T,
\text{ and}
\\
\label{f07:eq:khat}
\hat{K}(\ell,:) &= K(i_k^{(\ell)},:).
\end{align}
Here, $\ast$ represents elementwise multiplication.
In other words, $\hat{Z}$ and $\hat{K}$ represent the subset of rows of $Z$ and $K$, respectively, that corresponds to the known entries of $\mathcal{T}$.
Then, row $\ell$ of $F$ is given by
\begin{equation}
F(\ell,:) = \hat{Z}(\ell,:) \otimes \hat{K}(\ell,:).
\end{equation}

\subsubsection{Transforming the UI Subproblem}
\label{f07:sec:transformed-ui}

We can exploit a factorization of $K$ to transform \cref{f07:eq:unaligned-infinite} into an equivalent but potentially better conditioned system.
Assuming we have the eigendecomposition $K = U D U^T$, we can rewrite \cref{f07:eq:unaligned-infinite} by factoring out $(I_r \otimes U)$ to obtain
\begin{equation}
\left[
\underbrace{(Z \otimes UD)^T S}_{\bar{F}^T}
\underbrace{S^T (Z \otimes UD)}_{\bar{F}}
+ \lambda (I_r \otimes D)
\right] \vecop(\underbrace{U^TW}_{\bar{W}})
= \vecop(\underbrace{DU^TB}_{\bar{B}}).
\end{equation}
Now we have a transformed system in the variable $\bar{W} = U^TW$, and we can solve for $W$ via $W = U\bar{W}$ after solving the system.
Note that we cannot pull $D$ into the definition of $\bar{W}$ because it is indefinite.
We define $\bar{F} := S^T(Z \otimes UD) \in \mathbb{R}^{q \times rn}$, which is analogous to $F$ with $K$ replaced by $UD$.
We define $\bar{B} := DU^TB \in \mathbb{R}^{n \times r}$.
Adding a regularization term as before, we obtain the modified system
\begin{equation}
\label{f07:eq:ui-transformed}
\left[
\bar{F}^T \bar{F} + \lambda (I_r \otimes D) + \rho I_{rn}
\right] \vecop(\bar{W}) = \vecop(\bar{B}).
\end{equation}

\subsubsection{Key Lemmas for PCG Solution of UI Subproblem}
\label{f07:app:unaligned-lemmas}

Before we continue to the details of solving \cref{f07:eq:ui-transformed} via PCG, we present some key lemmas about working with matrices where each row is a Kronecker product of rows of two other matrices.
These lemmas are important for efficiently computing the matrix-vector products and a preconditioner needed for PCG. We state these generically here so they can be reused in other contexts.

Let $A \in \mathbb{R}^{q \times r}$ and $B \in \mathbb{R}^{q \times n}$. Define the $q \times rn$ matrix $C$ row-wise as
\begin{equation}\label{f07:eq:FAB}
C(\ell,:) = A(\ell,:) \otimes B(\ell,:), \quad \text{for} \quad \ell=1,\ldots,q.
\end{equation}
Recall that for the Kronecker product of an $n$-vector and an $r$-vector, or for the vectorization of an $n \times r$ matrix, there is a correspondence between $k \in [rn]$ and the pair $(i,j)$ with $i \in [n]$ and $j \in [r]$ such that $k = i + (j-1)n$.
For the Kronecker product, this means $C_{\ell k} = B_{\ell i} A_{\ell j}$.
For a vectorized matrix, we have $(\vecop(X))_k = X_{ij}$.

\Cref{f07:lem:F} shows how to compute the matrix-vector product $Cx$ efficiently.
This would normally cost $\mathcal{O}(qrn)$ if we formed $C$ explicitly.
However, using the structure of $C$, we can compute it using only $\mathcal{O}(q(r+n))$ operations. Moreover, we avoid forming $C$ explicitly, which reduces the memory from $\mathcal{O}(qrn)$ to $\mathcal{O}(q(r+n))$.

\begin{fileglemma}\label{f07:lem:F}
Given the setup in \cref{f07:eq:FAB}, let $X \in \mathbb{R}^{n \times r}$ be a matrix and define $x = \vecop(X)$.
Then we have
\[
Cx = (A \ast BX)1_r.
\]
Here $1_r$ denotes the $r$-vector of all ones.
\end{fileglemma}
\begin{proof}
For all $\ell = 1,\ldots,q$ we have
\[
(Cx)_{\ell} = \sum_{k=1}^{rn} C_{\ell k}x_k = \sum_{j=1}^r \sum_{i=1}^n B_{\ell i}X_{ij}A_{\ell j} = \sum_{j=1}^r (BX)_{\ell j}A_{\ell j}.
\]
\end{proof}

\Cref{f07:lem:Ft} shows how to compute the matrix-vector product $C^Tv$ without forming $C$ explicitly. The cost is unchanged at $\mathcal{O}(qrn)$, but the memory is reduced from $\mathcal{O}(qrn)$ to $\mathcal{O}(q(r+n))$.

\begin{fileglemma}\label{f07:lem:Ft}
Given the setup in \cref{f07:eq:FAB}, let $v \in \mathbb{R}^{q}$.
Then we have
\[
C^Tv = \vecop(B^T\diag(v)A).
\]
\end{fileglemma}
\begin{proof}
Define $k = i + (j-1)n$ for $i=1,\ldots,n$ and $j=1,\ldots,r$. Then we have
\[
(C^Tv)_k = \sum_{\ell=1}^q C_{\ell k}v_{\ell} = \sum_{\ell=1}^q B_{\ell i}A_{\ell j}v_{\ell} = (B^T\diag(v)A)_{ij} = \left(\vecop(B^T\diag(v)A)\right)_k.
\]
\end{proof}

\Cref{f07:lem:FtF} shows how to compute the diagonal of $C^TC$ efficiently.
We reduce the computation from $\mathcal{O}(qr^2n^2)$ to $\mathcal{O}(q(r^2+n^2))$ operations.
Again, we avoid forming $C$ explicitly, which reduces the memory from $\mathcal{O}(qrn)$ to $\mathcal{O}(q(r+n))$.

\begin{fileglemma}\label{f07:lem:FtF}
Given the setup in \cref{f07:eq:FAB}. Then
\[
\diag(C^TC) = \vecop\left((B \ast B)^T(A \ast A)\right).
\]
\end{fileglemma}
\begin{proof}
Define $k = i + (j-1)n$ for $i=1,\ldots,n$ and $j=1,\ldots,r$. Then we have
\begin{align*}
(C^TC)_{kk} &= \sum_{\ell=1}^q C_{\ell k}^2 = \sum_{\ell=1}^q B_{\ell i}^2A_{\ell j}^2 \\
&= \left[(B \ast B)^T(A \ast A)\right]_{ij} = \left[\vecop\left((B \ast B)^T(A \ast A)\right)\right]_k.
\end{align*}
\end{proof}

We apply these results in the next section.

\subsubsection{PCG Solution of Transformed UI Subproblem}
\label{f07:sec:pcg-ui}

We can form $\bar{F}$ similarly to how we formed $F$.
We define $H = UD \in \mathbb{R}^{n \times n}$ and $\hat{H} \in \mathbb{R}^{q \times n}$ such that $\hat{H}(\ell,:) = H(i_k^{(\ell)},:)$ for each $\ell \in [q]$.
Then, for each $\ell \in [q]$, we let
\begin{equation}
\bar{F}(\ell,:) = \hat{Z}(\ell,:) \otimes \hat{H}(\ell,:).
\end{equation}
Let $x \in \mathbb{R}^{rn}$ be an arbitrary vector, and let $X \in \mathbb{R}^{n \times r}$ be its matrix representation so that $\vecop(X) = x$.
From \cref{f07:lem:F,f07:lem:Ft} in \cref{f07:app:unaligned-lemmas}, we can compute $\bar{F}^T\bar{F}x$ as
\[
\vecop\left(\hat{H}^T\diag\left((\hat{Z} \ast \hat{H}X)1_r\right)\hat{Z}\right).
\]
Then, we can compute the matrix-vector products for the conjugate gradient iterations without forming any Kronecker products using
\begin{equation}\label{f07:eq:pcg-matvec-ui}
\left(\bar{F}^T\bar{F} + \lambda (I_r \otimes D) + \rho I_{rn}\right)x =
\vecop\left(
\hat{H}^T\diag\left((\hat{Z} \ast \hat{H}X)1_r\right)\hat{Z} + \lambda DX + \rho X
\right).
\end{equation}

We propose a diagonal preconditioner of the form
\[
\bar{D} = \diag(\diag(\bar{F}^T\bar{F})) + \lambda (I_r \otimes D) + \rho I_{rn}.
\]
Observe that $\bar{d} := \diag(\bar{D})$ is easy to compute since
\begin{equation}\label{f07:eq:preconditioner-ui}
\begin{aligned}
\bar{d}
&= \diag\bigl(\diag(\bar{F}^T\bar{F})) + \lambda (I_r \otimes D) + \rho I_{rn}\bigr) \\
&= \diag(\bar{F}^T\bar{F}) + \lambda (1_r \otimes \diag(D)) + \rho 1_{rn} \\
&= \vecop\bigl((\hat{H} \ast \hat{H})^T(\hat{Z} \ast \hat{Z})\bigr) + \lambda (1_r \otimes \diag(D)) + \rho 1_{rn}.
\end{aligned}
\end{equation}
The last step comes from \cref{f07:lem:FtF} in \cref{f07:app:unaligned-lemmas}.

\subsubsection{Comparison of Costs}
\label{f07:sec:cost-compare-ui}

A comparison of the direct solution of the original symmetric problem \cref{f07:eq:unaligned-infinite-1} and PCG iterative solutions of the transformed problem \cref{f07:eq:ui-transformed} are shown in \cref{f07:tab:unaligned-infinite-cost-comparison}.
For PCG, we let $p$ denote the number of iterations needed for convergence.
Recall that $d$ is the order of the tensor, $n$ is the size of mode $k$, $r$ is the target rank, and $q$ is the number of known entries.
In general, we do not make assumptions about the relative sizes of $n$ and $r$. We do assume, however, that $d < n,r \ll q$.
Because we are working with an incomplete tensor, the MTTKRP is relatively cheap and never dominates the cost.

\begin{table}[ht]
\centering
\caption{Comparison of costs to solve the mode-$k$ unaligned infinite-dimensional subproblem \cref{f07:eq:unaligned-infinite} of size $nr \times nr$ where $n$ is the size of mode $k$ and $r$ is the target tensor decomposition rank.
The variable $q$ is the number of known entries in the observed tensor $\mathcal{T}$.
For the PCG iterative method, $p$ is the number of iterations.}
\label{f07:tab:unaligned-infinite-cost-comparison}
\renewcommand{\arraystretch}{1.5}
\pgfplotstabletypeset[
col sep=&,row sep=\\,
string type,
columns/Description/.style={column type=c},
columns/DS/.style={column type=c, column name={Direct Symmetric}},
columns/PCG/.style={column type=c, column name={PCG Iterative}},
every head row/.style={before row={\toprule}, after row=\midrule},
every last row/.style={after row=\bottomrule},
every even row/.style={before row={\rowcolor[gray]{0.9}}},
font=\footnotesize,
assign column name/.style={/pgfplots/table/column name={\sffamily\textbf{#1}}},
]{
Description & DS & PCG \\
Factorize $K = U D U^T$ \emph{one-time cost!} & --- & $\mathcal{O}(n^3)$ \\
Compute $\hat{Z}$ and MTTKRP $B := \mathcal{T} Z$ & $\mathcal{O}(qrd)$ & $\mathcal{O}(qrd)$ \\
Form $F$ (and $G$) or $H$ & $\mathcal{O}(qrn)$ & $\mathcal{O}(n^2)$ \\
Form matrix for linear solve & $\mathcal{O}(qr^2n^2)$ & --- \\
Form right-hand side & $\mathcal{O}(n^2r)$ & $\mathcal{O}(n^2r)$ \\
Form preconditioner ($\bar{d}$) & --- & $\mathcal{O}(qn^2 + qr^2)$ \\
Solve system & $\mathcal{O}(r^3n^3)$ & $\mathcal{O}(pnqr)$ \\
Recover $W$ & --- & $\mathcal{O}(n^2r)$ \\
Total cost & $\mathcal{O}(qn^2r^2 + n^3r^3)$ & $\mathcal{O}(qn^2 + qr^2 + qnrp)$ \\
Storage & $\mathcal{O}(qnr + r^2n^2)$ & $\mathcal{O}(qn + qr)$ \\
}
\end{table}

\paragraph{Factorizing the kernel matrix $K$ for the transformed system}
The eigendecomposition of $K$ costs $\mathcal{O}(n^3)$ flops.
We stress once again that this is only done \emph{one time} before the outermost alternating optimization iterations begin.
In the methods we compare here, this is needed only for the PCG iterative method.

\paragraph{Shared costs of all methods}
The $q \times r$ matrix $\hat{Z}$ defined in \cref{f07:eq:zhat} is used by both methods.
Likewise, the $n \times r$ MTTKRP $B = \mathcal{T}Z$ is used by all methods.
The cost to compute $\hat{Z}$ is $\mathcal{O}(qrd)$.
Computing $B$ is an MTTKRP with an incomplete tensor (Ballard and Kolda, \emph{Tensor Decompositions for Data Science}, Cambridge University Press, 2025, with PDF available freely online). This would normally cost $\mathcal{O}(qrd)$ operations, but we can use $\hat{Z}$ to reduce the cost to $\mathcal{O}(qr)$ operations.

\paragraph{Direct solve of symmetric regularized system}
We first analyze the cost to form and solve the system as discussed in \cref{f07:sec:direct-ui-sym}.
We have to explicitly form $F$ to form the system in \cref{f07:eq:unaligned-infinite-1}.
The cost to compute the $q \times rn$ matrix $F$ is $\mathcal{O}(qrn)$ and requires $\mathcal{O}(qrn)$ storage.
Forming the $rn \times rn$ matrix $F^TF + \lambda (I_r \otimes K) + \rho I_{rn}$ is dominated by the cost to compute $F^TF$, which costs $\mathcal{O}(qr^2n^2)$ operations.
We also have to compute the right-hand side $\vecop(KB)$, which costs $\mathcal{O}(n^2r)$ operations.
Finally, using a direct method such as Cholesky to solve the system costs $\mathcal{O}((rn)^3)$ operations.
The storage is either dominated by storing $F$ or the system matrix, which is $\mathcal{O}(rnq + r^2n^2)$.

\paragraph{PCG iterative solve of transformed system}
We now analyze the cost of using PCG to solve the transformed system \cref{f07:eq:ui-transformed} as discussed in \cref{f07:sec:pcg-ui}.
The right-hand side $\vecop(\bar{B}) = \vecop(DU^TB)$ can be computed at a cost of $\mathcal{O}(n^2r)$ operations.
We first have to compute the $n \times n$ matrix $H := UD$, which costs $\mathcal{O}(n^2)$ operations.
Forming the diagonal preconditioner, the $rn$-vector $\bar{d}$ in \cref{f07:eq:preconditioner-ui}, costs $\mathcal{O}(qn^2 + qr^2)$ operations.
We never form $\bar{F}$ explicitly, which saves both computation and storage.
Each matrix-vector product is computed as in \cref{f07:eq:pcg-matvec-ui} at a cost of $\mathcal{O}(qnr)$ operations.
Each preconditioner application costs $\mathcal{O}(rn)$ operations.
Assuming that PCG converges in $p$ iterations, the total cost for the PCG iterations is $\mathcal{O}(pqnr)$ operations.
Finally, after solving for $\bar{W}$, we have to recover $W = U\bar{W}$, which costs $\mathcal{O}(n^2r)$ operations.
The storage needed for PCG is dominated by storing $\hat{Z}$ and $\hat{H}$, which is $\mathcal{O}(qn + qr)$.

\paragraph{Summary and Comparison}
The direct method is cubic in the size of the unknown matrix $W$.
In contrast, the PCG iterative method has a cost that is orders of magnitude lower, depending on the number of iterations $p$ needed for convergence and the relative sizes of $n$, $r$, and $p$. In general, we expect the problem to be well conditioned so that $p$ is not too large.
The PCG method also has significantly lower storage requirements. Assuming $r < n < rn < q$, we have $qrn$ storage for the direct methods versus $qn$ storage for PCG.
\endgroup

\end{document}